\documentclass[11pt]{article}

\usepackage[top=0.95in,bottom=0.95in,left=0.95in,right=0.95in]{geometry}

\usepackage[usenames,dvipsnames]{xcolor}

\usepackage{amsmath,amsfonts,bm}

\def\eqref#1{equation~\ref{#1}}

\def\1{\bm{1}}

\DeclareMathAlphabet{\mathsfit}{\encodingdefault}{\sfdefault}{m}{sl}
\SetMathAlphabet{\mathsfit}{bold}{\encodingdefault}{\sfdefault}{bx}{n}

\newcommand{\E}{\mathbb{E}}

\DeclareMathOperator*{\argmax}{arg\,max}

\newcommand{\Top}{\textup{Top}}
\newcommand{\haricolor}[1]{\color{black}}

\usepackage{url}
\usepackage[colorlinks=true,citecolor=ForestGreen]{hyperref}       

\usepackage[utf8]{inputenc} 
\usepackage[T1]{fontenc}    
\usepackage{url}            
\usepackage{booktabs}       
\usepackage{amsfonts}       
\usepackage{nicefrac}       
\usepackage{microtype}      
\usepackage{graphicx, subcaption}
\usepackage{empheq}

\usepackage{booktabs}       
\usepackage{amsfonts}       
\usepackage{nicefrac}       

\usepackage{natbib}

\usepackage{float}
\usepackage{subcaption}
\usepackage{graphicx}
\usepackage[export]{adjustbox}
\usepackage[inline]{enumitem}

\usepackage{multirow}
\usepackage{hyperref}

\usepackage{amsmath,amssymb,amsthm}
\usepackage[mathscr]{eucal}
\usepackage{stmaryrd}
\usepackage{accents}
\usepackage{upgreek}
\usepackage[scaled=0.8]{FiraMono}

\usepackage{algorithm}
\usepackage{algpseudocode}
\usepackage[english]{babel}

\usepackage{booktabs,colortbl,multirow}

\usepackage{pifont}

\usepackage{algorithm}
\usepackage{wrapfig}

\newtheorem{lemma}{Lemma}
\newtheorem{remark}{Remark}

\newcommand*\widefbox[1]{\fbox{\hspace{0.8em}#1\hspace{0.8em}}}

\graphicspath{{figs//}}

\allowdisplaybreaks

\newcommand{\todowj}[1]{\todo[color={green!30!white}]{#1}}

\newcommand{\todoakm}[1]{\todo[color={RoyalBlue!35!white}]{#1}}

\newcommand{\todoasr}[1]{\todo[color={cyan}]{#1}}

\newcommand{\removefornow}[1]{}

\renewcommand{\Pr}{\mathbb{P}}

\newcommand{\defEq}{\stackrel{.}{=}}
\newcommand{\defeq}{\defEq}

\newcommand{\bp}{p_t}
\newcommand{\bq}{q_t}
\newcommand{\bT}{\mathbb{T}}

\newcommand{\cV}{\mathscr{V}}

\newcommand{\zo}{\text{0-1}}
\newcommand{\entropy}{\textrm{entropy}}

\newcommand{\tv}{\textrm{TV}}

\newcommand{\normm}{\operatorname{norm}}

\theoremstyle{definition}

\usepackage[textsize=tiny,textwidth=2cm,disable]{todonotes}

\title{Faster Cascades via Speculative Decoding}

\author{
Harikrishna Narasimhan$^\dagger$, ~~Wittawat Jitkrittum$^*$, ~~Ankit Singh Rawat$^*$\\  Seungyeon Kim$^*$, ~~Neha Gupta$^\dagger$, ~~Aditya Krishna Menon$^*$, ~~Sanjiv Kumar$^*$\\[12pt]
$^\dagger$Google Research, Mountain View\hspace{1cm}
$^*$Google Research, New York
\\[12pt]
Corresponding author: \tt{hnarasimhan@google.com}\\[5pt]
}

\begin{document}

\maketitle

\begin{abstract}
Cascades and speculative decoding are two common approaches to improving language models' inference efficiency. 
Both approaches 
involve
interleaving models of different sizes, 
but via fundamentally distinct mechanisms:
cascades employ a \emph{deferral rule} that invokes the larger model only for ``hard'' inputs,
while speculative decoding uses \emph{speculative execution} to primarily invoke the larger model in parallel verification mode.
These mechanisms offer different benefits: empirically, cascades offer better cost-quality trade-offs, often even outperforming the large model, while theoretically, speculative decoding offers a guarantee of quality-neutrality. 
In this paper, we leverage the best of both these approaches by designing new 
\emph{speculative cascading} techniques that implement their deferral rule through speculative execution. We characterize the optimal deferral rule for our speculative cascades, and employ a plug-in approximation to the optimal rule. 
Experiments with Gemma and T5 models on a range of language benchmarks show that our approach yields better cost-quality trade-offs than cascading and speculative decoding baselines.

\end{abstract}

\section{Introduction}

Large language models (LLMs) have yielded significant advances in quality on a range of natural language processing tasks~\citep{Radford:2018,Raffel:2020,Brown:2020,Black:2022,Chowdhery:2022,Wei:2022,Chung:2022,Tay:2023,Anil:2023,Touvron:2023,Gemini:2023}, at the cost of an 
increase in inference latency.
This has sparked a growing body of literature on reducing LMs' inference costs without (overly) compromising 
on quality~\citep{Elbayad:2020,Pope:2022,Schuster:2022,Leviathan:2023,chen2023accelerating,Sheng:2023,Sun:2023}.
One such line of work involves constructing a family of models of various sizes (e.g., a small and large model), and suitably orchestrating amongst them to make a prediction.
Two canonical instantiations of this strategy are \emph{model cascading}~\citep{wang2020wisdom,Mamou:2022,Varshney:2022,Khalili:2022,Dohan:2022,Chen:2023a,gupta2024language,Ding:2024} and \emph{speculative decoding}~\citep{Stern:2018,chen2023accelerating,Leviathan:2023,Sun:2023,Li:2024,Xia:2024}.

While similar in spirit, 
cascades and speculative decoding are fundamentally different in details.
Cascades
employ a \emph{deferral rule} to identify ``hard'' inputs,
and only invoke larger models on such inputs.
For example,
in a two-model cascade, one first invokes the smaller model, and uses its associated probability of the generated output to decide whether to defer to the larger model.
By contrast, 
speculative decoding 
uses a small model to \emph{draft} a block of tokens via standard auto-regressive decoding, which are then \emph{verified} in parallel by a large model. 
One 
then accepts all drafted tokens until the first ``implausible'' one,
which is rolled back 
based on
the larger LM's prediction. 

Owing to their different mechanisms, both methods have complementary strengths.
Cascades seek to output distributions that have the best quality for a given cost budget, 
and potentially provide \textit{\haricolor{red}better cost-quality trade-offs}, sometimes even
yielding better accuracies than the individual models they are constructed with \citep{jitkrittum2024does, kim2023speculative} (\S\ref{sec:cascade-meets-speed}).  
By contrast,
speculative decoding is theoretically guaranteed to match the output distribution (or a close approximation thereof~\citep{Tran-Thien_2023}), and are practically observed to provide impressive speed-ups~\citep{Stern:2018,chen2023accelerating,Leviathan:2023,Sun:2023}.
Given the complementary nature of these two approaches, a natural question arises: \emph{can we leverage the best of both techniques?} 

In this paper, we do so by designing new 
techniques for two-model cascades
that implement their deferral rule in a speculative manner: 
we have the smaller model generate drafts auto-regressively, and the larger model execute in parallel on the drafts to decide \emph{whether or not to defer on them}. We show that this \emph{speculative cascading} approach 
yields better
cost-quality trade-offs than both standard cascades and speculative decoding. 
In detail, we make the following contributions:
%
\begin{enumerate}[itemsep=2pt,topsep=0pt,leftmargin=16pt,label=(\roman*)]
    \item  We introduce a general recipe for speculative execution, where we seek to mimic a general \emph{target} distribution that interleaves the drafter's and verifier's 
    distributions.
    Lossy speculative sampling \citep{Tran-Thien_2023} is a special case of 
    this recipe 
    for a particular target distribution (\S\ref{sec:gen-speed}).
    
    \item We  show how common cascading 
    rules, 
    such as Chow's rule \citep{chow1970optimum} and confidence-difference thresholding \citep{jitkrittum2024does}, 
    can be implemented speculatively by plugging in their target distribution into our framework. We refer to these as \emph{speculative cascades} (\S\ref{sec:seq-to-spec}).

    \item 
    We characterize the \emph{theoretically optimal} deferral rule for a speculative cascade, and design a speculative cascading technique that implements a plug-in estimate to the optimal rule (\S\ref{sec:opt-spec}, Lemma \ref{lem:spec-def-opt}, Table \ref{tab:targets}). {\haricolor{blue} We also present token-specific variants of our deferral rules (\S\ref{sec:sample-dep}).}
    \item Through experiments with {
    \haricolor{blue}
    Gemma  \citep{team2024gemma} and T5 models \citep{raffel2020exploring} on a range of benchmark language tasks including summarization, translation, reasoning, coding and QA}, we show that speculative cascades 
    are able to provide better cost-quality trade-offs than their sequential cascade and speculative decoding counterparts (\S\ref{sec:expts}).
\end{enumerate}

\section{A Tale of Two Efficient LM Inference Strategies}
\label{sec:prelims}

Let $\mathscr{V}$ denote a finite vocabulary of \emph{tokens},
with $\mathscr{V}^*$ denoting the set of all \emph{sequences} generated by this vocabulary. 
Let $\Delta_\cV$ denote the set of all probability distributions over tokens in $\cV$. 
Given an arbitrary length sequence $x = x_1 x_2 \ldots x_L \in \mathscr{V}^*$ and index $i \leq L$, denote by $x_{<i} = x_1 x_2 \ldots x_{i - 1}$. 
A \emph{language model} (LM)
is a
probability distribution over $\mathscr{V}^*$. 
Let $\Pr$ denote the ground-truth probability distribution over $\mathscr{V}^*$. This could be, for example, a distribution over prompt-response pairs
that the LM may encounter during deployment, or a distribution of sequences used to pre-train the LM. 
We will measure the quality of an LM based on how closely it mimics $\Pr$. 

Suppose we are provided two LMs $q$ and $p$, 
where $p$ is the larger (more expensive)
model. Our goal is to design an inference strategy that selectively invokes $q$ and $p$ to  trade-off between quality and latency (
which may be approximated by the fraction of times that $p$ is invoked).
We will denote by $q(x_t|x_{<t})$ the probability $q$ associates to token $x_t \in \mathscr{V}$ given prefix 
$x_{<t} \in \mathscr{V}^{t-1}$, and by $p(x_t|x_{<t})$ 
the same distribution from model $p$. Whenever it is clear from context, we will hide the conditioning on prefix $x_{<t}$, and use the shorthand $p_t(\cdot)$ for $p(\cdot|x_{<t})$ and $q_t(\cdot)$ for $q(\cdot|x_{<t})$.  


\textbf{Cascades} 
are an effective strategy to trade-off cost and quality by having the smaller model $q$ handle the ``easy'' samples, 
and the larger model $p$ handle the ``hard'' ones \citep{gupta2024language,Yue:2024}. A common cascading approach  is confidence thresholding or Chow's rule \citep{chow1970optimum, jitkrittum2024does},
where we first run $q$ on the  input, and defer to $p$ when $q$'s {confidence} for its generated response is sufficiently low. This strategy is typically
implemented at the \emph{sequence-level}, where for a given prefix $x_{<m}$ 
we invoke $q$ to generate a complete response $x_{m}\ldots x_{m+n}$. We evaluate $q$'s predicted probability for the response, and check whether it falls below a threshold $\alpha \in [0,1]$:
\begin{align}
\textstyle
 q(x_{m}\ldots x_{m+n}\,|\,x_{<m}) < 1 - \alpha.
 \label{eq:sequential-chow}
\end{align}
If the above holds, 
we defer to $p$ to generate a new response;
otherwise,
we retain $q$'s response. One may  tune the threshold to achieve a desired cost-quality trade-off. The literature also offers variants of Chow's rule that use a more nuanced aggregation of \emph{per-token} uncertainties \citep{gupta2024language}.  

\textbf{Speculative decoding} 
is an alternate strategy that applies \emph{token-level} interleaving between $q$ and $p$,  
resulting in \emph{provably} matching the larger model quality at a reduced inference cost~\citep{Stern:2018,Leviathan:2023}.
Given a prefix $x_{<t}$, we 
\emph{draft} $\gamma$ draft tokens $x_t, \ldots, x_{t+\gamma-1}$ via auto-regressive sampling from $q$, 
and \emph{verify} if these tokens can be accepted
by
running $p$ in parallel on the $\gamma$ prefixes $x_{<t}, \ldots, x_{<t+\gamma-1}$. 
We then rollback to the first rejected token $t+j^*$ (where $j^* \in \{ 0, 1, \ldots, \gamma - 1 \}$), replace $x_{t+j^*}$
with a new token, and repeat the process {
with prefix $x_{< t + j^* + 1}$}.

During the verification stage, a draft token $x_{t+j}$ generated by $q$ is accepted 
with probability
$\min\left( 1, \frac{p_{t+j}(x_{t+j})}{q_{t+j}(x_{t+j})} \right)$
and rejected otherwise,
recalling the shorthand
$q_{t+j}(\cdot) = q(\cdot|x_{<t+j})$ and $p_{t+j}(\cdot) = p(\cdot|x_{<t+j})$. 
A rejected token is then replaced by a new token sampled from a modified distribution
$\operatorname{norm}\left(\max\left\{0,\, p_{t+j}(\cdot) - q_{t+j}(\cdot)\right\}\right),$ 
where $\text{norm}( \cdot )$
denotes normalization to sum to 1. 
This sampling process is 
provably
equivalent to sampling $\gamma$ tokens auto-regressively from $p$ for prefix $x_{<t}$~\citep{Leviathan:2023}. 
We summarize this speculative sampling procedure in Algorithm \ref{alg:spec-decode}. 
Each invocation of this algorithm generates 
at most $\gamma + 1$ next tokens (and at least one) 
\todoakm{and at least $1$? Hari: done!}
for a given prefix $x_{<t}$. One may run this algorithm multiple times to generate a complete output sequence.

In practice, one may employ a lossy variant \citep{Tran-Thien_2023} of the above sampling that allows some \emph{deviation} from verifier's distribution $p$. In this case, a draft token $x_{t+j}$ 
is accepted 
with probability
$\min\left( 1, \frac{p_{t+j}(x_{t+j})}{(1 - \alpha) \cdot q_{t+j}(x_{t+j})} \right)$,
where $\alpha \in [0, 1)$ is a strictness parameter, with higher values indicating greater deviation from $p$.
A rejected token may then be replaced by a token sampled from the  residual distribution
$\operatorname{norm}\left(\max\left\{0,\, \frac{1}{\beta}\cdot p_{t+j}(\cdot) - q_{t+j}(\cdot)\right\}\right),
$
where $\beta \geq 1 - \alpha$ is a parameter that depends on $\alpha$, $q$ and $p$. 
A common heuristic is to simply set $\beta = 1$ \citep{ zhou2024distillspec}.

\begin{table}[]
    \centering
    \caption{Target distributions associated with different inference algorithms, where $\alpha$ is a free parameter and $\beta \geq 1-\alpha$ depends on $\alpha$, $q$ and $p$. The last column indicates whether the  execution is sequential (Algorithm \ref{alg:seq-sampling}), via an oracle (Algorithm \ref{alg:oracle}), or speculative (Algorithm \ref{alg:spec-cascade}) \citep{Leviathan:2023}. 
    See~(\ref{eqn:t_delta}) for details on $\delta$.
    The third row presents a variant of the BiLD algorithm of \cite{kim2023speculative}, where $D(q, p)$ is a measure of discrepancy between $q$ and $p$; the original algorithm differs  in the use of a deterministic speculative decoding procedure with a dynamic draft window 
    (see \S\ref{sec:related}). 
    }
    \vspace{-6pt}
    \resizebox{\linewidth}{!}{%
    \begin{tabular}{@{}llll@{}}
    \toprule
    Inference strategy & 
    Deferral decision $\delta(q,p)$ &
    Target distribution $\pi(x)$ & Execution\\
    \midrule
         SpecDecoding  \citep{Leviathan:2023} &
        - &
         $p(x)$
          & Speculative
         \\[2pt]
         Lossy SpecDecoding  \citep{Tran-Thien_2023} & 
         - &
         $\max\{\min\{ q(x), \frac{p(x)}{1 - \alpha} \}, 
         \frac{p(x)}{\beta}\}
         $
         & Speculative
         \\[2pt]
         BiLD* \citep{kim2023speculative} & 
        $\1\big(\, D(q, p) > \alpha \big)$
         &
         $ (1 - \delta) \cdot q(x) + \delta \cdot p(x)$
         & Speculative
         \\
         \midrule
         TokenCascade [{\tt Chow}] \citep{chow1970optimum} &
         $\1\big(\max_v q(v) < 1 - \alpha\big)$ 
         & $ (1 - \delta) \cdot q(x) + \delta \cdot p(x)$
         & Sequential
         \\[2pt]
         Oracle [{\tt Diff}] \citep{jitkrittum2024does}
         & $\1\big(\max_v q(v) < \max_v p(v) - \alpha\big)$ 
         & $ (1 - \delta) \cdot q(x) + \delta \cdot p(x)$
         & Oracle
         \\
         \midrule
         SpecCascade [{\tt Chow}]
         & $\1\big(\max_v q(v) < 1 - \alpha\big)$ 
         &$ (1 - \delta) \cdot q(x) + \delta \cdot p(x)$
         & Speculative
         \\[2pt]
         SpecCascade [{\tt Diff}]
         & $\1\big(\max_v q(v) < \max_v p(v)  - \alpha\big)$ 
         &$ (1 - \delta) \cdot q(x) + \delta \cdot p(x)$
         & Speculative
         \\[2pt]
         SpecCascade [{\tt OPT}]
         & $\1\big(\max_v q(v) < \max_v p(v)  - \alpha \cdot D_{\tv}(p, q)\big)$ 
         &$ (1 - \delta) \cdot q(x) + \delta \cdot p(x)$
         & Speculative\\
          \bottomrule
    \end{tabular}%
    }
    \vspace{-10pt}
    \label{tab:targets}
\end{table}

\section{Cascades Meet Speculative Decoding}
\label{sec:cascade-meets-speed}

Both cascades and speculative decoding interleave models of different sizes to reduce inference cost, but fundamentally differ in the mechanisms they use. 
As a step towards comparing the strengths and weaknesses of these approaches, we first
describe how one may design a \emph{token-level cascade}.

\subsection{Warm-up: Token-level cascades}
\label{sec:token-cascades}
It is straightforward to extend 
the sequence-level
Chow's rule from \S\ref{sec:prelims} to form a \emph{token-level cascade} between  $q$ and $p$. For a prefix $x_{<t}$, we first compute the smaller model's distribution $q(\cdot|x_{<t})$, and check whether $\max_{v \in \cV}\,q(v|x_{<t})$ is below a pre-chosen threshold.
if so,
we evaluate $p(\cdot|x_{<t})$, and  sample $x_t \sim p(\cdot|x_{<t})$;
otherwise, 
we  sample
$x_t \sim q(\cdot|x_{<t})$.

More generally, we may design a token-level \emph{deferral rule} ${r}: \cV^{t-1}  \rightarrow \{0,1\}$ that takes the prefix $x_{<t}$ as input and outputs a binary decision, with ${r}(x_{<t}) =  1$ indicating that we 
defer to $p$ (i.e., draw a sample from $p$ rather than $q$). 
For example, token-level Chow's rule can be written as:
\begin{align}
\textstyle
    {r}_{\rm \tt Chow}(x_{<t}) = 1 ~\iff~\max_{v \in \cV}\,q(v|x_{<t}) < 1 - \alpha,
    \label{eq:chow01}
\end{align}
where $\alpha$ is a threshold parameter; 
the higher the value, the lower is the frequency of deferral to $p$. 
One may also use other confidence measures than the maximum probability, such as the entropy of the small model's probability distribution.
We elaborate in \S\ref{app:chow}
that the  choice of confidence measure would depend on the evaluation metric of interest;  Equation \ref{eq:chow01} 
is typically prescribed when the cascade's quality 
is evaluated in terms of its accuracy against the ground-truth distribution on individual tokens, whereas entropy 
is prescribed when the metric of interest is the cross-entropy loss.

\subsection{Optimal token-level cascade deferral}
\label{sec:opt-sec-cascade}
While Chow's rule \eqref{eq:chow01} is easy to implement, 
it can be sub-optimal if the smaller model's max-token probability is not reflective of which of the two models are better equipped to predict the next token for a given prefix \citep{jitkrittum2024does}.
Given this, it is natural to ask what the \emph{optimal} deferral rule $r$ for a token-cascade looks like,
and whether we can reasonably approximate this rule. 

For this, we must first specify an objective to minimize at each step $t$. 
Following the prior cascade literature \citep{jitkrittum2024does,  gupta2024language}, a reasonable objective to  minimize is the expected loss from the deferral rule against the ground-truth distribution
$\Pr$,
with an added cost for  deferring to the larger model. We state this below for a fixed prefix $x_{<t}$,
using as before
\todoakm{``using once more''? Hari: done!}
the short-hand $q_t(\cdot)$ for $q(\cdot|x_{<t})$ and $p_t(\cdot)$ for $p(\cdot|x_{<t})$:
%
\begin{align}
L_{\rm def}(r; x_{<t}) &=\E_{v \sim \Pr(\cdot|x_{<t})}\Big[ \big(1 - r(x_{<t})\big) \cdot \ell(v, q_t) + r(x_{<t}) \cdot \big(\ell(v, p_t) + \alpha\big) \Big],
\label{eq:seq-def-risk}
\end{align}
for a cost penalty $\alpha \geq 0$ 
and loss function $\ell: \cV \times \Delta_{\cV} \rightarrow \mathbb{R}_+$.
Common choices for $\ell$ include the 0-1 loss $\ell_{\zo}(v, q_t) = \1\left(v \ne \argmax_{v'} q_t(v')\right)$  and the log loss $\ell_{\log}(v, q_t) = -\log\left(q_t(v)\right).$
\begin{lemma}[{Optimal deferral for token-level cascades} \citep{jitkrittum2024does}]
\label{lem:seq-def-opt}
The minimizer of \eqref{eq:seq-def-risk} is of the form:
\begin{align}
    r^*(x_{<t}) = 1 ~~\iff~~ \E_{v \sim \Pr(\cdot|x_{<t})}\left[\ell(v, q_t)\right] \,>\, \E_{v \sim \Pr(\cdot|x_{<t})}\left[\ell(v, p_t)\right] + \alpha.
    \label{eq:seq-def-opt}
\end{align}
\vspace{-15pt}
\end{lemma}
Intuitively, we compare the expected loss  from $q$ with the expected cost of invoking $p$, and decide to defer when the latter is smaller.
We note here that this optimization problem is set up for a \emph{fixed} prefix $x_{<t}$. One may also consider the coupled optimization problem across all positions.
\todo{Explain that this optimization problem is set up for a fixed prefix $x_{<t}$. One may also consider the coupled optimization problem across all prefixes from 1 to $T$. Hari: done!}

\emph{Plug-in estimator for \eqref{eq:seq-def-opt}.} 
The optimal rule in \eqref{eq:seq-def-opt} requires computing expectations over the ground-truth   distribution $\Pr(\cdot|x_{>t})$, which is not available during inference time.
 A common approach in the cascades literature is to replace the expected losses with the models' confidence estimates  \citep{jitkrittum2024does}. For example, when $\ell = \ell_\zo$, it may be reasonable to use $1 - \max_v q_t(v)$ as an estimate of the expected 0-1 loss $\E_{x_t \sim \Pr(\cdot|x_{<t})}\left[\ell_\zo(x_t, q_t)\right]$ and $1 - \max_v p_t(v)$ as an estimate of $\E_{x_t \sim \Pr(\cdot|x_{<t})}\left[\ell_\zo(x_t, q_t)\right]$.  The extent to which these estimates are accurate depend on how well  $q$ and $p$ are calibrated \citep{Guo:2017}. The resulting plug-in estimator 
 for (\ref{eq:seq-def-opt}) thresholds the \textit{difference} of confidence estimates from both distributions:
\begin{empheq}[box=\widefbox]{align}
\textstyle
    {\hat{r}_{\rm\tt Diff}(x_{<t}) = 1 ~~\iff~~ \max_{v} q_t(v)  \,<\, \max_{v} p_t(v) - \alpha.}
    \label{eq:seq-opt-def-01-plugin}
\end{empheq}
Similarly, when $\ell = \ell_{\log}$, we may use the entropy $-\sum_v q_t(v)\cdot \log(q_t(v))$ from $q_t$ as an estimate of its expected log-loss, and similarly for $p_t$ (see Appendix~\ref{app:log}). 
\begin{remark}[\textbf{Oracle deferral rules}]
\label{rem:oracle}
For efficiency reasons, 
  $\hat{r}_{\rm\tt Diff}$ 
  cannot be directly used in a \emph{token-level} cascade,\todowj{``to improve efficiency''? We can if we want quality, right? Hari: yep, added a clarification}
  as it needs the large model to be invoked at every step $t$. 
  However, it serves as an  \emph{oracle} that allows to analyze the head-room available to improve upon Chow's rule.
  See also Remark~\ref{rem:exact-oracle}.
\end{remark}

\subsection{\haricolor{red}Contrasting token-level cascade and speculative decoding trade-offs}
\label{sec:cascade-ensemble-effect}
Token-level cascades and 
speculative decoding differ in the distribution over tokens they seek to mimic. Speculative decoding seeks to mimic the large model's output distribution, and is usually used when one wants to match the quality of the large model. 
On the other hand, token-level cascades seek to output distributions that closely approximate the ground-truth label distribution and potentially offer \textit{\haricolor{red} good cost-quality trade-offs}, sometimes yielding better quality than even the large model.

Cascades are useful when the draft model fares better than the verifier on some inputs, and one may want to retain the drafter's predictions even when it disagrees with the verifier. Even in cases where both the drafter and verifier fare poorly on some inputs (e.g., due to label noise), one may want to ignore the disagreement between the drafter and verifier to avoid triggering unnecessary  roll-backs. 

As a concrete example, we consider token-level cascades of T5 models \citep{raffel2020exploring} of two different sizes finetuned on a WMT EN $\rightarrow$ DE translation \cite{bojar2014findings} and an extreme summarization (XSum) task \citep{narayan2018don}. We construct these cascades using both 
(token-level)
Chow's rule  in \eqref{eq:chow01} 
and the oracle {\tt Diff} rule in \eqref{eq:seq-opt-def-01-plugin}, 
and also apply speculative decoding with the smaller (larger) model as the drafter (verifier). 
In Figure \ref{fig:oracle_cascade}, we plot quality as a function of fraction of samples deferred to the large model (number of deferrals divided by number of generated tokens), as we vary the cost parameter $\alpha$. 
Note that with speculative decoding, each verification step\todowj{Minor: worth explicitly mentioning $\gamma$ tokens are verified in parallel.} 
verifies $\gamma$ tokens in parallel, but 
is counted as a single deferral to the large model.
While speculative decoding matches the quality of the large model (right-most point), the oracle yields significantly better 
cost-qualty trade-offs.
Even Chow's rule, which is sub-optimal for cascading \citep{jitkrittum2024does}, offers better trade-offs, and
outperforms speculative decoding in a small region.\todoasr{But does this small improvement comes at some added cost? E.g., higher latency? -- Okay, we seem to be discussing this in the following program. But should we breifly allude to this in the caption of Figure 1 as well? Hari: done!} As noted by \cite{kim2023speculative}, this may be attributed to the ensembling effect in a cascade.
\todoakm{the caption refers to ``deferral rate'', but the x-axis says ``Fraction of calls''. Hari: changed this as per next comment!}
\todoakm{maybe the x-axis could be more explicit, like ``\# large-model calls / \# generated tokens''. Hari: done!}

However, as also evident from the plots, token-level cascades require a significantly larger number of deferrals to the large model to achieve the same quality.
This is because token-level cascades are executed \emph{sequentially}: whenever $q$ defers, we execute $p$  once to generate one next token for the prefix accumulated so far, and the control transfers back to $q$. In contrast, speculative decoding runs $p$ in \emph{scoring} mode to verify $\gamma$ draft tokens from $q$ in parallel.  Moreover,  the stochastic verification algorithm in speculative decoding often results in fewer tokens from $q$ getting rejected compared to the deterministic deferral rules used in a cascade. These observations motivate a natural question: \emph{given their complementary strengths, how can we leverage the best of both these techniques?} 

\begin{figure}[t]
    \centering
    \includegraphics[scale=0.3]{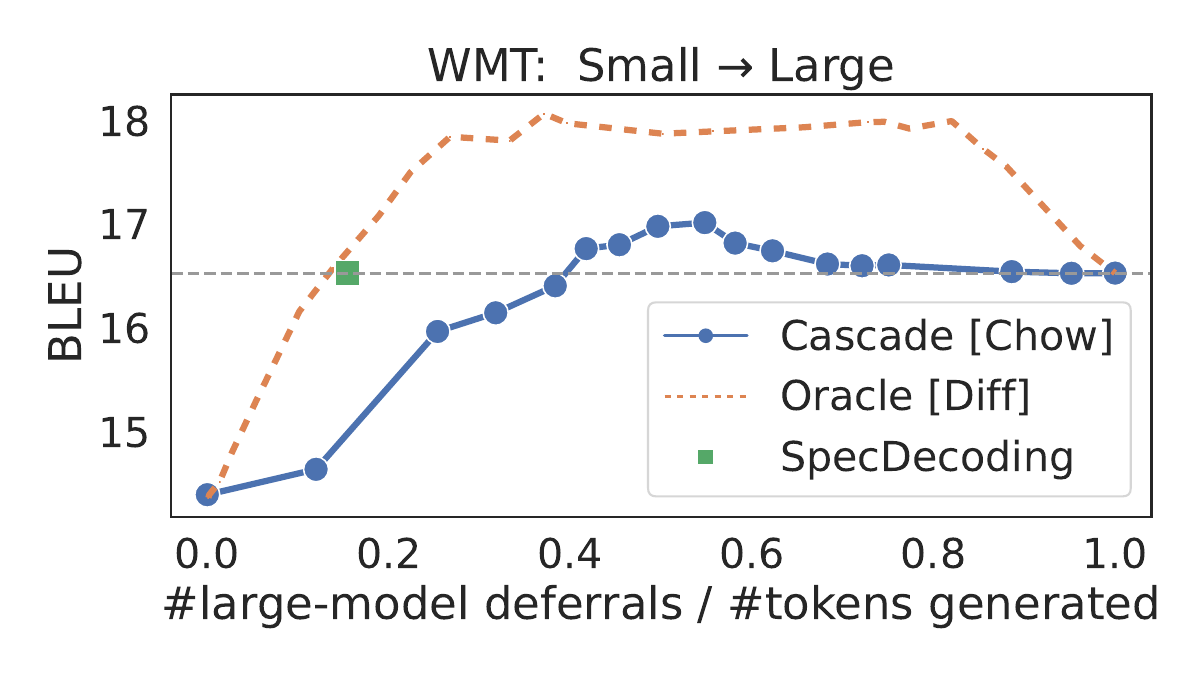}
     \includegraphics[scale=0.3]{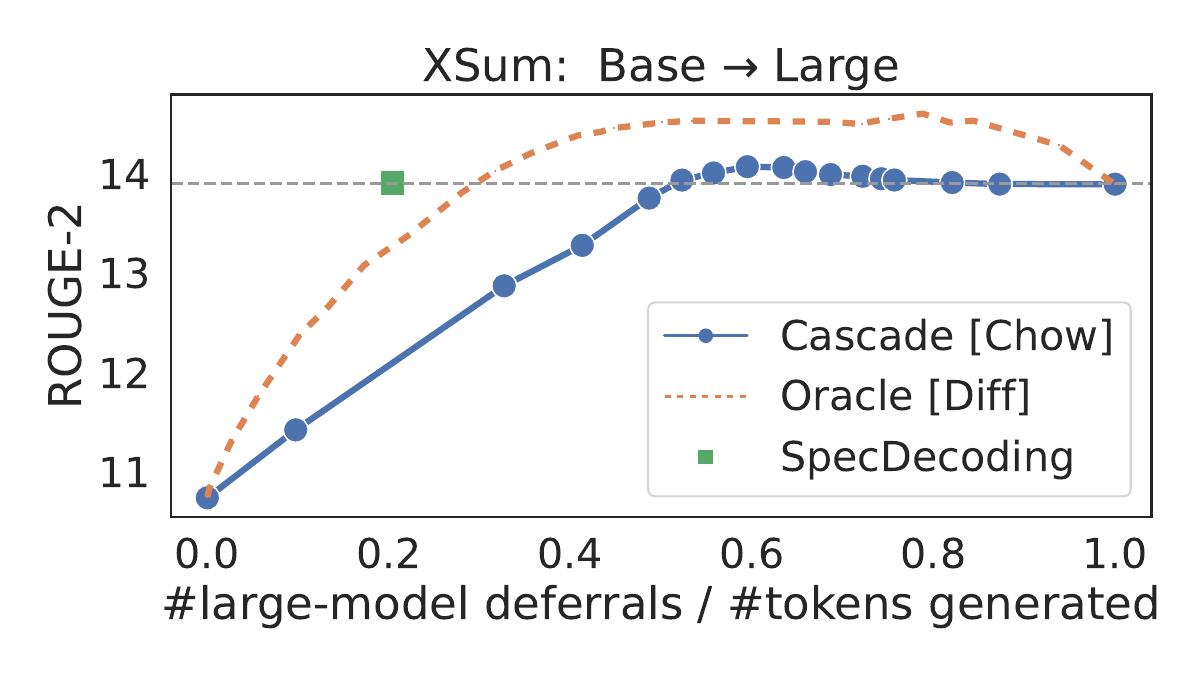}
    \caption{Plots of quality as a function of the \textit{\haricolor{red}number of deferrals to the larger model divided by the total number of generated tokens} for cascades constructed from T5 models (under temperature sampling with $T = 1$). The left-most point represents the small model and the right-most represents the large model.  We compare token-level cascades constructed with Chow's rule ({\tt Chow}) and an oracle deferral rule ({\tt Diff}), and speculative decoding with block size $\gamma = 5$. \textit{With a cascade, each call to the large model yields exactly one token, whereas with speculative decoding, a single call scores $\gamma$ draft tokens in parallel.}
    While speculative decoding matches the quality of the large model (see dashed horizontal line), the oracle deferral rule yields significantly better quality
    on a range of deferral rates; this however comes at the cost of higher number of deferrals to the large model.
    }
    \label{fig:oracle_cascade}
\end{figure}

\section{Speculative Cascades: Leveraging the Best of Both Worlds}
\label{sec:spec-cascades}
In addressing the above question, we present our main contribution: a principled approach to combining  the better trade-offs cascades offer with the faster execution of speculative decoding.

\subsection{Speculative decoding with general target distributions}
\label{sec:gen-speed}
We begin by considering a generic version of speculative sampling that seeks to mimic a \emph{general} target distribution derived from the drafter's and verifier's distributions.
In the proposed sampling procedure outlined in Algorithm \ref{alg:gen-speed}, we sample tokens auto-regressively as before from the drafter's distribution. During the verification step, however, we do not compare the drafter's token probabilities against the verifier's  distribution. Instead, we use a user-specified target distribution $\pi = \bT(q, p) \in \Delta_{\cV}$ derived from the drafter's and verifier's distributions at position $t$, for some function $\bT(\cdot,\cdot)$ that is inexpensive to compute.
We accept a draft token  $x_t$ when $q(x_t) \leq \pi(x_t)$ and reject it otherwise with probability $1 - \frac{\pi(x_t)}{q(x_t)}$. Upon rejection, we re-sample from the residual distribution $\operatorname{norm}\left(\max\{0, \pi(\cdot) - q(\cdot)\}\right)$.

This general procedure not only encompasses standard speculative decoding \citep{Leviathan:2023} for $\bT(q, p) = p$, but also includes lossy speculative decoding \citep{Tran-Thien_2023} as a special case:
\begin{lemma}
\label{lem:lossy-sd-target-distr}
Algorithm \ref{alg:gen-speed} reduces to the lossy speculative sampling procedure in \citep{Tran-Thien_2023} with  parameters $\alpha$ and $\beta$ when $\bT(q, p)(v) = \max\{\min\{ q(v), \frac{p(v)}{1 - \alpha} \},  \frac{p(v)}{\beta}\}$. 
\end{lemma}

\subsection{From sequential to speculative cascades}
\label{sec:seq-to-spec}
Equipped with Algorithm \ref{alg:gen-speed}, we now propose new cascading techniques that implement their deferral rule in a speculative manner. Recall 
from \S\ref{sec:token-cascades}
that a 
token-level
cascade of two models $q$ and $p$ is defined by a deferral rule $r: \cV^{t-1}\rightarrow \{0,1\}$. 
For a prefix $x_{<t}$, the next-token distribution at position $t$ modeled by this cascade can be written as:
\[
\pi(v) = (1 - r(x_{<t})) \cdot q_t(v) + r(x_{<t}) \cdot p_t(v).
\]
In fact, for  all the deferral rules described in \S\ref{sec:prelims}, the resulting distribution can be described by a target distribution function $\bT_\delta$ of the form:
\begin{equation}
    \label{eqn:t_delta}
    \bT_\delta(q, p)(v)
     = (1 - \delta(q, p)) \cdot q(v) + \delta(q, p) \cdot p(v),
\end{equation}
for some function $\delta: \Delta_\cV \times \Delta_\cV \rightarrow \{0,1\}$ that maps  distributions $(q, p)$ to a binary decision. 
For example, for {\tt Chow}, $\delta(q, p) = \1\big(\max_v q(v) < 1 - \alpha\big)$, and for {\tt Diff}, $\delta(q, p) = \1\big(\max_v q(v) < \max_v p(v) - \alpha\big).$ See Table \ref{tab:targets} for a summary of target distributions for different deferral rules. \todowj{Table 1 is referred to for the first time here. But it is placed a few pages back. Hari: added a ref in the contributions list}

Our proposal is to then invoke the speculative sampling procedure in Algorithm \ref{alg:gen-speed} with $\bT_\delta$ as the target distribution function.   
We outline this generic \emph{speculative cascading} approach in Algorithm \ref{alg:spec-cascade}, and contrast it with the sequential execution of a deferral rule in Algorithm \ref{alg:seq-sampling}. 



\begin{remark}[\textbf{Exact implementation of  oracle deferral rule {\tt Diff}}]
\label{rem:exact-oracle}
In a sequential cascade, the large model's distribution $p$ cannot be used \todowj{Saying that $p$ is not available may confuse some readers. ``using $p$ defeats the purpose of a cascade''? Hari: done!} at the time the deferral decision is made (see Remark \ref{rem:oracle}), as this would defeat the purpose of the cascade. With a speculative cascade, however, we can employ rules like {\tt Diff}  that depend on both $q$ and $p$. This is because we run the large model $p$ in parallel on  drafts generated by the small model $q$, allowing us to compute both $p(\cdot)$ and $q(\cdot)$ on every prefix.
\end{remark}




So far we have considered deferral rules designed for sequential cascades. In what follows, we derive the optimal deferral rule $r$ for a speculative cascade, where we sample speculatively from a target distribution $\pi = (1 - r(x_{<t})) \cdot q_t + r(x_{<t}) \cdot p_t$ using $q_t$ as the drafter.

\begin{figure}
\begin{minipage}[c]{0.43\linewidth}
\vspace{-10pt}
\begin{algorithm}[H]\small
\caption{\tt{SpecDecode}}
\label{alg:spec-decode}
\begin{algorithmic}
\Require Models $q$, $p$, Prefix $x_{<t}$, Block size $\gamma$
\State $\bT(q,p) \defeq p$
\Ensure $\text{\tt{GenSpecSample}}(q, p, \bT, x_{<t}, \gamma)$
\end{algorithmic}
\end{algorithm}
\vspace{-0.45cm}
\begin{algorithm}[H]\small
\caption{\tt{TokenCascade}}
\label{alg:seq-sampling}
\begin{algorithmic}
\Require Models $q$, $p$, Deferral logic ${\delta}$, Prefix $x_{<t}$
\State $\bq(\cdot) \defeq q(\cdot|x_{<t})$
\If{${\delta}(\bq, \emptyset) = 0$ }
    \State Sample  $x_t \sim \bq(\cdot)$
\Else
    \State $\bp(\cdot) \defeq p(\cdot|x_{<t})$; ~~Sample  $x_t \sim \bp(\cdot)$
\EndIf
\Ensure $x_t$
\end{algorithmic}
\end{algorithm}
\vspace{-0.45cm}
\begin{algorithm}[H]\small
\caption{\tt{OracleCascade}}
\label{alg:oracle}
\begin{algorithmic}
\Require Models $q$, $p$, Deferral logic ${\delta}$, Prefix $x_{<t}$
\State $\bq(\cdot) \defeq q(\cdot|x_{<t})$;~~~ $\bp(\cdot) \defeq p(\cdot|x_{<t})$
\If{${\delta}(\bq, \bp) = 0$ }
    \State Sample  $x_t \sim \bq(\cdot)$
\Else
    \State Sample  $x_t \sim \bp(\cdot)$
\EndIf
\Ensure $x_t$
\end{algorithmic}
\end{algorithm}
\end{minipage}
\vspace{-0.27cm}
~~
\begin{minipage}[c]{0.57\linewidth}
\vspace{-5pt}
\begin{algorithm}[H]
\caption{\tt{GenSpecSample}}
\label{alg:gen-speed}
\begin{algorithmic}\small
\Require Models $q$, $p$, Target distr.\ $\bT$, Prefix $x_{<t}$, Block size $\gamma$ 
\vspace{1pt}
\State $[\gamma] \equiv \{0, \ldots, \gamma\}$
\vspace{2pt}
\State {\color{olive}Sample $\gamma$ tokens auto-regressively from $q$}
\vspace{1pt}
\For{$j = 0$ to $\gamma-1$}
\State $q_{t+j}(\cdot) \defeq q(\cdot|x_{<t+j});~~~~$ $x_{t+j} \sim q_{t+j}(\cdot)$
\EndFor
\vspace{2pt}
\State {\color{olive}Run $p$ in parallel to score $\gamma$ draft tokens}
\vspace{1pt}
\State $p_{t+j}(\cdot) \defeq p(\cdot|x_{<t+j}), ~\forall j \in [\gamma]$
\State $\pi_{t+j} = \bT(q_{t+j}, p_{t+j})$
\vspace{3pt}
\State {\color{olive}Find the earliest draft token that gets rejected}
\vspace{1pt}
\State $a_{j} \sim \text{Ber}\left(\min\left\{1, \frac{ \pi_{t+j}({x_{t+j}}) }{ q_{t+j}({x_{t+j}}) }\right\} \right), ~\forall j \in [\gamma -1];$
~~ $a_{\gamma} = 0$
\State $j^{*} = \min\{j \in [\gamma] \,:\, a_{j} = 0\}$
\vspace{3pt}
\State {\color{olive}Sample a new token from residual distribution}
\vspace{1pt}
\State $p_{\rm res}(\cdot) =
\begin{cases}
\operatorname{norm}\hspace{-1pt}\left(\max\left\{0,\, \pi_{t+j^{*}}(\cdot) - q_{t+j^{*}}(\cdot)\right\}\right)
& \text{if}~j^* < \gamma
\\
\pi_{t+\gamma}(\cdot) & \text{else}
\end{cases}
$
\State Sample  $x_{t+j^{*}} \sim p_{\rm res}(\cdot)$
\vspace{1pt}
\Ensure $x_t, \ldots, x_{t+j^{*}}$
\end{algorithmic}
\end{algorithm}
\vspace{-0.55cm}
\begin{algorithm}[H]\small
\caption{{\tt{SpecCascade}}}
\label{alg:spec-cascade}
\begin{algorithmic}
\Require Models $q$, $p$, Deferral logic $\delta$, Prefix $x_{<t}$, Block size $\gamma$
\State $\bT_\delta(q,p) \defeq (1 - \delta(q,p)) \cdot q +  \delta(q,p)  \cdot p$
\Ensure $\text{\tt{GenSpecSample}}(q, p, \bT_\delta, x_{<t}, \gamma)$
\end{algorithmic}
\end{algorithm}
\end{minipage}
\end{figure}

\subsection{Optimal speculative cascade deferral} 
\label{sec:opt-spec}
%
As with sequential cascades (\S\ref{sec:prelims}), we begin by defining an  objective to minimize. We seek a deferral rule  $r: \cV^{t-1} \rightarrow \{0,1\}$ that minimizes a loss against the ground-truth distribution, while limiting the inference cost to be within a budget. 
(Per above, this deferral rule implicitly defines a target distribution $\pi$.)
The inference cost crucially depends on how frequently a draft token is rejected in the verification phase, triggering a rollback. 
To this end, we derive the probability that a token sampled from $q$  is rejected during verification, for a target distribution resulting from a deferral rule $r$.\todoakm{nit, in Section 3 we did not keep a separate section on ``deferral risk for token-level cascades''. Hari: combined with next section!}
\begin{lemma}
\label{lem:token-acceptance-rate}
For a given prefix $x_{<t}$,  and target distribution $\pi = (1 - r(x_{<t})) \cdot q_t + r(x_{<t}) \cdot p_t$,\todowj{$\pi_t$?} the probability of a token  drawn from draft distribution $q_t$ being rejected is equal to:
$r(x_{<t}) \cdot D_{\textup{\tv}}(\bp, \bq),$
where $D_\textup{\tv}(p, q) = \sum_{v \in \cV}\max\{0, p(v) - q(v)\}$ is the \emph{TV} distance between $p$ and $q$.
\end{lemma}
Intuitively, whenever $r(x_{<t}) = 0$, $\pi(v) = q_t(v)$, and therefore there is no rejection or roll-back; when $r(x_{<t}) = 1$, the rejection rate equals $D_{\textup{\tv}}(\bp, \bq)$.


For a fixed prefix $x_{<t}$, we formulate the goal of finding a solution to:
\begin{align}
\hspace{-0.25cm}\min_{r}&~~\E_{v \sim \Pr(\cdot|x_{<t})}\Big[ \big(1 - r(x_{<t})\big) \cdot \ell(v, q_t) + r(x_{<t}) \cdot \ell(v, p_t)\big) \Big]
\hspace{0.1cm}\text{s.t.}~~
 r(x_{<t}) \cdot D_{\tv}(\bp, \bq)  \,\leq\, B,
\label{eq:spec-def-constrained}
\end{align}
for some budget $B > 0$. 
Equivalently, one may minimize an unconstrained objective similar to \eqref{eq:seq-def-risk},
for suitable cost parameter $\alpha > 0$ (see \S\ref{app:equivalence}):
\begin{align}
\hspace{-5pt}L_{\rm spec}(r; x_{<t}) &= \E_{v \sim \Pr(\cdot|x_{<t})}\big[ \big(1 - r(x_{<t})\big) \cdot \ell(v, q_t) + r(x_{<t}) \cdot \big(\ell(v, p_t) + \alpha \cdot D_{\tv}(\bp, \bq)\big) \big],
\label{eq:spec-def-risk}
\end{align}

Contrasting \eqref{eq:spec-def-risk}\todowj{nit: sometimes we use "\eqref{eq:seq-def-risk}". Sometimes we use (\ref{eq:seq-def-risk}).} with the deferral risk in \eqref{eq:seq-def-risk} for a sequential cascade, a  key difference is that the cost of deferring to  the larger model is \emph{no longer a constant}, but depends on the similarity between $\bq$ and $\bp$, as measured by the TV distance between them.

\todoakm{nit, have we used the phrase ``speculative deferral'' in the text? what about ``optimal speculative cascade deferral''?. Hari: done!}
We next derive the optimal deferral rule for \eqref{eq:spec-def-risk}, and construct a feasible estimator for it.

\begin{lemma}[{Optimal deferral for speculative cascades}]
The minimizer of \eqref{eq:spec-def-risk} is of the form:
\begin{align}
    r^*(x_{<t}) = 1 ~~\iff~~ \E_{v \sim \Pr(\cdot|x_{<t})}\left[\ell(v, q_t)\right] \,>\, \E_{v \sim \Pr(\cdot|x_{<t})}\left[\ell(v, p_t)\right] + \alpha \cdot D_{\textup{\tv}}(\bp, \bq).
    \label{eq:spec-def-opt}
\end{align}
\vspace{-15pt}
\label{lem:spec-def-opt}
\end{lemma}
When $p_t$ and $q_t$ are similar, the rejection rate for $q_t$ is low, and hence the deferral decision will  depend largely on which of the two models yields a lower expected loss. When $p_t$ and $q_t$ are very different, the optimal decision is to defer to  $p_t$ only when it yields a substantially lower loss than $q_t$. 

\emph{Plug-in estimator for \eqref{eq:spec-def-opt}.} 
The optimal rule requires estimating expectations with respect the ground-truth distribution $\Pr(\cdot|x_{<t}).$ We employ similar plug-in estimators as the ones used with sequential cascades (\S\ref{sec:cascade-meets-speed}). When $\ell=\ell_\zo$, we replace the expected 0-1 loss with (one minus) the  maximum probability from the model, giving us:
\begin{empheq}[box=\widefbox]{align}
\textstyle
    {\hat{r}_{\rm\tt OPT}(x_{<t}) = 1 ~~\iff~~ \max_{v} q_t(v)  \,<\, \max_{v} p_t(v) - \alpha \cdot D_{\textup{\tv}}(\bp, \bq).}
    \label{eq:spec-opt-def-01-plugin}
\end{empheq}
%
The efficacy of the  plug-in estimator depends on how closely the individual models approximate the ground-truth distribution $\Pr(\cdot|x_{<t})$; this is formalized by the following regret bound:
\begin{lemma}[{Regret bound for $\hat{r}_{\rm\tt OPT}$}] 
\label{lem:regret-01}
Suppose $\ell = \ell_{\emph{\zo}}$. Then for a fixed prefix $x_{<t}$:
\begin{align*}
L_{\rm spec}(\hat{r}_{\rm\tt OPT}; x_{<t}) - \min_r\, L_{\rm spec}(r; x_{<t}) ~\leq~ \max_{v \in \cV} \big|\Pr(v|x_{<t}) - q_t(v)\big| \,+\, \max_{v \in \cV} \big|\Pr(v|x_{<t}) - p_t(v)\big|.
\\[-20pt]
\end{align*}
\end{lemma}
One can now run the speculative cascading procedure in Algorithm \ref{alg:spec-cascade} using  \eqref{eq:spec-opt-def-01-plugin} as the deferral rule; the corresponding $\delta(\cdot)$ is listed in Table \ref{tab:targets}. See  \S\ref{app:opt-log} for a similar derivation for $\ell=\ell_{\log}$.

{\haricolor{blue}
\subsection{Token-specific speculative cascades}
\label{sec:sample-dep}
The plug-in deferral rules 
in  (\ref{eq:seq-opt-def-01-plugin}) and (\ref{eq:spec-opt-def-01-plugin}) 
decide between the drafter's distribution $q_t(\cdot)$ and the verifier's distribution $p_t(\cdot)$ by comparing their maximum token probabilities. 
A downside to this approach is that the draft token $x_t \sim q_t(\cdot)$ may not 
maximize
$q_t(\cdot)$.
Thus, even when  $x_{t}$ is of poor quality, we may end up accepting it because $q_t$ happens to be \emph{more peaked} than $p_t$. 

To alleviate this problem, we propose the use of \emph{token-specific deferral rules} $r: \cV^{t-1} \times \cV \rightarrow \{0,1\}$
that use both the prefix $x_{<t}$ and a candidate token $v$ to provide a binary decision $r(x_{<t}, v) \in \{0, 1\}$, with 0 indicating that the token is of acceptable quality. We may then construct a target distribution of the following form:
\begin{align}
    \pi_{\rm\tt Token}(v) = q_t(v) \cdot (1 - r(x_{<t}, v))  + p_t(v) \cdot \eta,
    \label{eq:sample-dep-target-distribution}
\end{align}
where $\eta = \sum_{v' \in \cV} r(x_{<t}, v') \cdot q_t(v')$ is a normalizing term chosen to ensure that $\sum_v \pi_{\tt Token}(v) = 1$. This target distribution closely mimics  $q_t(\cdot)$  on tokens that the deferral rule $r$ deems to be of acceptable quality, and defers to $p_t(\cdot)$ otherwise. 
One can modify the generic speculative sampling algorithm in Algorithm \ref{alg:gen-speed} to use $\pi_{\tt Token}$ as the target distribution, as shown in Algorithm \ref{alg:token-spec} in \S\ref{app:sample-dep}.

To design 
$r$, we propose a  heuristic variant of the {\tt Diff} rule in equation \ref{eq:seq-def-opt} that compares the expected 0-1 loss from the candidate token $v$ with the expected 0-1 loss from distribution $p_t$ (in \S\ref{app:sample-dep}, we discuss deriving a similar variant of the {\tt OPT} rule in equation \ref{eq:spec-def-opt}):
\todoakm{the reviewer may wonder how these relate to Equations 4 and 9. Hari: added refs to eq 4 and 9.}
\begin{align}
    r(x_{<t}, v) = 1 ~~\iff~~ 1 - \Pr(v|x_{<t}) \,>\, \E_{v \sim \Pr(\cdot|x_{<t})}\left[\ell_{\zo}(v, p_t)\right] + \alpha,
    \label{eq:token-specific-r}
\end{align}
for a cost parameter $\alpha$. 
The following are some simple plug-in approximations to \eqref{eq:token-specific-r}:
\begin{empheq}[box=\widefbox]{align}
    \hat{r}_{\rm\tt TokenV1}(x_{<t}, v) = 1 &~~\iff~~\textstyle q_t(v)  \,<\, \max_{v'} p_t(v') - \alpha
    \label{eq:sample-dep-01-plugin-v1}\\
    \hat{r}_{\rm\tt TokenV2}(x_{<t}, v) = 1 &~~\iff~~\textstyle p_t(v)  \,<\, \max_{v'} p_t(v') - \alpha
    \label{eq:sample-dep-01-plugin-v2}\\
    \hat{r}_{\rm\tt TokenV3}(x_{<t}, v) = 1 &~~\iff~~\textstyle p_t(v)  \,<\, \max_{v'} p_t(v') \cdot(1 - \alpha)
    \label{eq:sample-dep-01-plugin-v3},
\end{empheq}
where we approximate $\Pr(v|x_{<t})$ with either $q_t(v)$ or $p_t(v)$. Equation \ref{eq:sample-dep-01-plugin-v3}  is a multiplicative plug-in approximation that has  similarities to the rejection criterion used 
by \citet{Leviathan:2023} for lossy speculative greedy decoding, and  results in an intuitive target distribution:
\begin{align*}
\textstyle
    \pi_{\tt TokenV3}(v) =
    q_t(v) \cdot \1\big(v \in \Top_\alpha\big)
    \,+\,
    p_t(v) \cdot \sum_{v' \notin \Top_\alpha} q_t(v'),
\end{align*}
where $\Top_\alpha = \{v \in \cV:\, p_t(v)  \,\geq\, \max_{v'} p_t(v') \cdot(1 - \alpha)\}$ is the set of top ranked tokens by $p_t(\cdot)$. For these top-ranked tokens, $\pi_{\tt TokenV3}$ approximates $q_t(\cdot)$; for the rest, it is a re-scaled version of $p_t(\cdot)$. 
}
\section{Further related work and conclusions}
\label{sec:related}
There has been a stream of work on improving
the  draft generation process in speculative decoding; these include having the drafter and verifier share the same backbone \citep{Stern:2018, kim2024towards, cai2024medusa, monea2023pass, hooper2023speed, zhang2023draft, elhoushi2024layer, liu2024kangaroo}, using multiple small draft models \cite{chen2023cascade, wang2024minions}, using  tree-structured draft batches \citep{spector2023accelerating, miao2024specinfer},  distilling the drafter with the verifier \citep{zhou2024distillspec}, and leveraging multiple sampled draft candidates \cite{Sun:2023}.\todoakm{maybe this can come after Experiments? Hari: the problem with doing this is that we reference this section in the expts when pointing to the target distribution for BiLD!}

The work that is most closely related to our specific proposal is the Big Little Decoder (BiLD)~\citep{kim2023speculative}, which can be seen as another lossy variant of speculative decoding \citep{Leviathan:2023, Tran-Thien_2023, zhou2024distillspec}. BiLD has two phases: a \emph{fallback} phase, during which the drafter $q$ is run auto-regressively until its maximum predicted probability is sufficiently low; and a \emph{rollback} phase, during which the verifier $p$ is run in parallel on the prefixes generated by $q$ and rolls back to the point where $D(q,p) > \alpha$, for a metric $D$ that measures discrepancy 
and threshold $\alpha$. The fallback phase implements Chow's deferral rule in (\ref{eq:chow01}), and allows for the draft window size to vary dynamically based on an estimate of how likely the draft tokens will be accepted; the rollback phase can be seen as a deterministic variant of the rejection sampling algorithm of~\citet{Leviathan:2023}.

 An advantage of {BiLD} over the rejection sampling algorithm in \citep{Leviathan:2023} is the use of Chow's rule to vary the draft window size. However, the final target distribution it seeks to mimic, $\bT_{\textup{BiLD}}(q, p)(v) = \1(D(q, p) \leq \alpha)\cdot q(v) + \1(D(q, p) > \alpha)\cdot p(v)$, is an approximation to $p$; specifically, the target distribution $\pi = \bT_{\textup{BiLD}}(q, p)$ is chosen to satisfy $D(\pi, p) \leq \alpha$. Hence, in cases where $q$ deviates substantially from $p$,  BiLD would choose $p$  as the target distribution, even when $q$ offers better quality on a prefix (where quality can be measured using a suitable loss function). In contrast, our proposed approach in \S\ref{sec:spec-cascades} uses speculative decoding to approximate  target distributions that seek to \emph{optimally} cascade between $q$ and $p$.
In our experiments, we compare the efficacy of using  $\bT_{\textup{BiLD}}$ as the target distribution with the target distributions we propose in this paper (see Table \ref{tab:targets}).


\begin{figure}[t]
    \centering
\begin{table}[H]
    \centering
    \caption{Reduction in latency from different methods ($T=1, \gamma=5$) when matching the quality of the large model (cols 2--7), and the best quality metric when matching each method yields without exceeding the latency of the large model (cols 8--13). Quality is measured in terms of the BLEU for WMT and ROUGE-2 for XSum and CNNDM. See Figure \ref{fig:tradeoffs} for $T=0.5$ and \S\ref{app:T5-greedy} for $T=0$.}
    \vspace{-5pt}
    \resizebox{\columnwidth}{!}{%
    \begin{tabular}{lcccccccccccccccccc}
    \toprule
    & \multicolumn{7}{c}{Latency$\downarrow$ when matching large model's quality}
    & &
    \multicolumn{7}{c}{Best quality  \emph{without} exceeding large model's latency}
    \\
    \cmidrule{2-8}
    \cmidrule{10-16}
    & \multicolumn{3}{c}{Small $\rightarrow$ Large}
    &
    & \multicolumn{3}{c}{Small $\rightarrow$ XL}
    &
    & \multicolumn{3}{c}{Small $\rightarrow$ Large}
    &
    & \multicolumn{3}{c}{Small $\rightarrow$ XL}
    \\
     \cmidrule{2-4}
     \cmidrule{6-8}
     \cmidrule{10-12}
     \cmidrule{14-16}
    Method & 
    WMT & XSum & CNNDM & & 
    WMT & XSum & CNNDM & &
    WMT & XSum & CNNDM & & 
    WMT & XSum & CNNDM 
    \\
    \midrule
         \texttt{SeqCascade} [{\tt Chow}] & 
         1.55$\times$ &
         0.84$\times$ &
         0.98$\times$ & &
         2.46$\times$
         & 0.93$\times$
         & 0.94$\times$
         & 
         & 16.56
         & 12.97
         & 9.91 
         &
         & 16.29
         & 16.40
         & 11.18
         \\
    \midrule
         \texttt{TokenCascade} [{\tt Chow}] & 
         1.03$\times$ & 
         0.93$\times$ & 
         1.40$\times$ & & 
         1.46$\times$ &
         0.82$\times$ &
         1.51$\times$ & 
         &16.52 
         &13.30
         &10.36 &
         & 16.65
         & 17.09
         &11.44
         \\
         \texttt{SpecDecode} [{\tt Lossy}] & 
         1.61$\times$ & 
         1.10$\times$
         & 1.57$\times$ &
         & 2.17$\times$
         & 1.28$\times$ 
         & 2.07$\times$ &&
         17.26& 
         13.90& 
         10.43& 
         & 16.94
         &17.36
         &11.53
         \\
         \texttt{BiLD}$^*$ 
         & 1.34$\times$ 
         & 1.04$\times$ 
         & 1.38$\times$ &
         & 1.85$\times$
         & 1.28$\times$
         & 1.84$\times$ &&
         16.49& 
         13.81&
         10.14& 
         & 15.90
         & 17.35
         & 11.35
         \\
    \midrule
         \texttt{SpecCascade} [{\tt Chow}]  
         & 1.43$\times$ 
         & 1.04$\times$ 
         & 1.41$\times$ &
         & 2.01$\times$
         & 1.28$\times$
         & 1.97$\times$ &&
         17.76&
         13.82&
         10.28& 
         & 16.35
         &17.36
         & 11.39
         \\
         \texttt{SpecCascade} [{\tt Diff}] 
         & 1.79$\times$ 
         & {1.17$\times$}
         & 1.75$\times$ &
         & 2.44$\times$ 
         & 1.30$\times$
         & {2.15$\times$} & &
         18.04&
         14.00&
         10.64& 
         & 18.07
         & 17.37
         & 11.67
         \\
         \texttt{SpecCascade} [{\tt OPT}] 
         & \textbf{1.95$\times$}
         & {1.17$\times$}
         & {1.80$\times$} &
         & \textbf{2.61$\times$} 
         & {1.34$\times$} 
         & \textbf{2.21$\times$} & &
         {18.33}& 
         {14.10}& 
         {10.86}& 
         & {18.09}
         &{17.48}
         &{11.85}
         \\
         \texttt{SpecCascade} [{\tt Token}]
         & 1.85$\times$
         & \textbf{1.18$\times$}
         & \textbf{1.89$\times$} &
         & {2.50$\times$}
         & \textbf{1.40$\times$}
         & {1.89$\times$}
         & &
         \textbf{22.50} &
         \textbf{15.85} &
         \textbf{12.63} &
         & 
         \textbf{22.70} &
         \textbf{18.79} &
         \textbf{12.63}
         \\
    \bottomrule
    \end{tabular}
    }
    \label{tab:speedup}
    \vspace{-10pt}
\end{table}
    \hspace{-9pt}
    \includegraphics[scale=0.35]{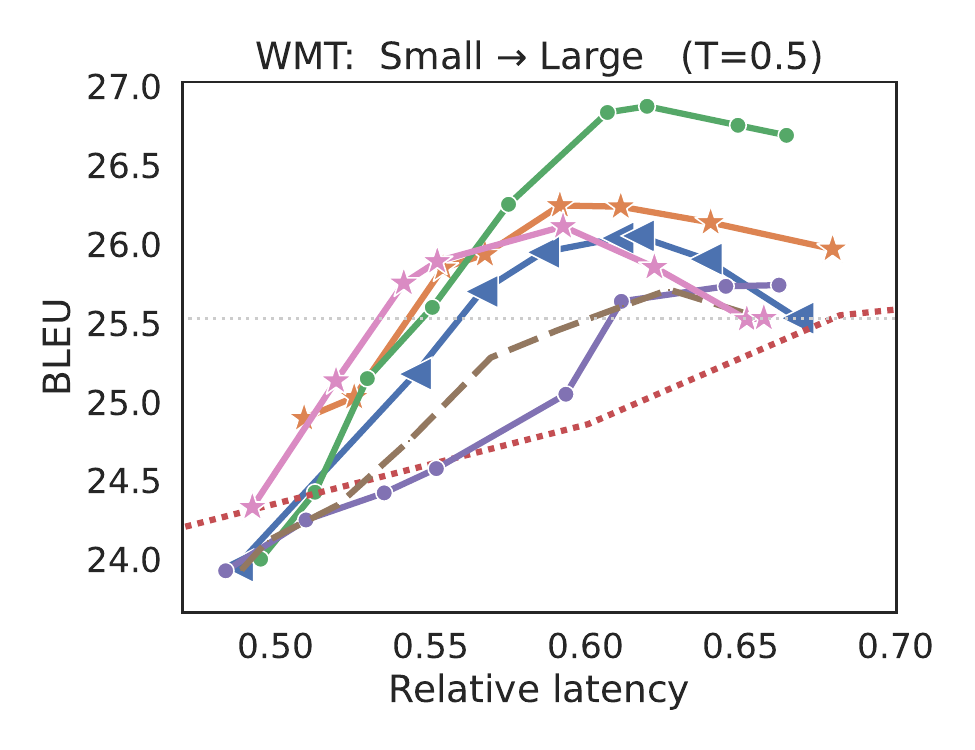}
    \hspace{-16pt}
    \includegraphics[scale=0.35]{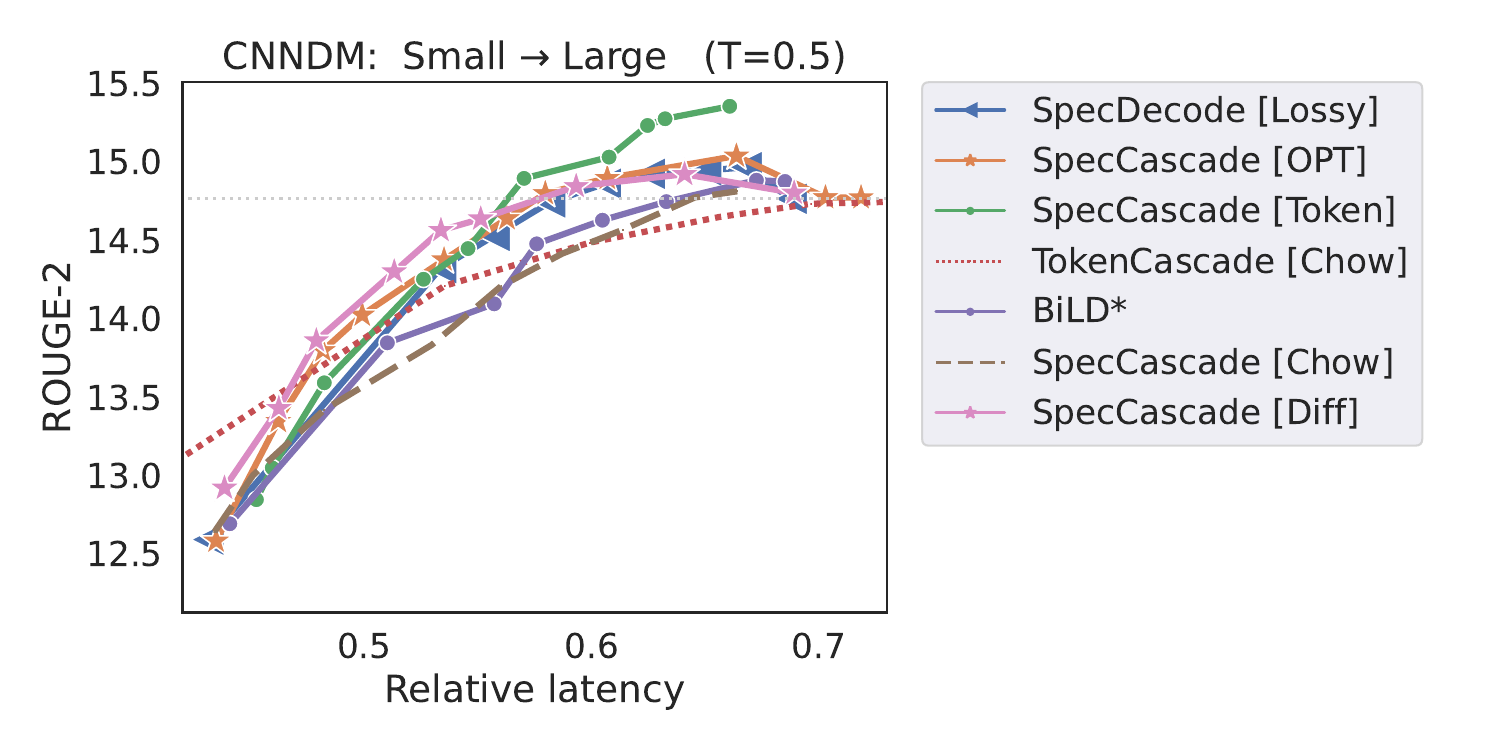}
    \caption{Latency-quality trade-off plots for methods that interleave T5-small with T5-large ($\gamma=5$). Latency is measured \emph{relative} to that of  calling  T5-large on all inputs.
    The horizontal dotted line denotes the quality of T5-large. \S\ref{app:expts-temperature}--\ref{app:expts-gamma} contain more plots with varying temperatures and $\gamma$.
    }
    \label{fig:tradeoffs}
\end{figure}

\section{Experimental results}
\label{sec:expts}
We compare our speculative cascading techniques with both sequential cascades and standard speculative decoding on a range of language benchmarks, including translation, reasoning, coding, QA, etc. 
%
We evaluate speculative cascades constructed from both the \textbf{T5 v1.1 family} of encoder-decoder models \citep{Raffel:2020}, and \textbf{Gemma v2} decoder-only models \citep{team2024gemma}. 
We construct the cascades with four deferral rules: (i) {\tt Chow} in (\ref{eq:chow01}), (ii) {\tt Diff} in (\ref{eq:seq-opt-def-01-plugin}), (iii) {\tt OPT} in (\ref{eq:spec-opt-def-01-plugin}), and (iv) 
the {\tt Token}-specific rule  in (\ref{eq:sample-dep-01-plugin-v3}) (we present results for the V1 and V2 variants in \S\ref{app:expts-token-specific}). 

{\haricolor{blue}
\textbf{Cascades versus SpecDecode evaluation}.\ Our evaluation protocol is markedly different from the standard evaluation of speculative decoding algorithms, where the goal is to speed up inference with a large model while preserving its output distribution.
In contrast, our focus is on \textbf{trading-off quality for lower inference costs by cascading} two models of different sizes. 
\todoakm{the ``and'' seems like it could be a new sentence. it is not related to the point about deviating from the standard SPEED evaluation protocol? Hari: done!}
We also 
 \textbf{do \textit{not} claim to develop a new state-of-the-art method for fast LM inference}. Furthermore, the speculative cascades we design build on the  original speculative decoding algorithm \citet{Leviathan:2023}. While one could potentially also adapt our proposal to  other recent variants of speculative decoding \citep{cai2024medusa, li2024eagle}, these involve a wholly orthogonal suite of techniques to what we propose (such as architectural changes, allowing for multiple drafts, distillation, and so on; see \S\ref{sec:related}).
}


\textbf{Baselines.} The cascading and speculative decoding methods we compare to include:
\begin{enumerate}[itemsep=1pt,topsep=0pt,leftmargin=16pt,label=(\roman*)]
    \item  \emph{Sequence-level cascade} \citep{jitkrittum2024does,gupta2024language} based on sequence-level Chow's rule in \eqref{eq:sequential-chow} (\texttt{SeqCascade} [{\tt Chow}]).
    \item \emph{Token-level cascade}
    outlined in Algorithm \ref{alg:seq-sampling}, with token-level Chow's rule in \eqref{eq:chow01} used for deferral \citep{chow1970optimum, Gupta:2022} (\texttt{TokenCascade} [{\tt Chow}]).
    \item \emph{Lossy speculative decoding} described in \S\ref{sec:prelims}, with both $\beta = 1$ \citep{Leviathan:2023, zhou2024distillspec} (\texttt{SpecDecode} [{\tt Lossy}]) and $\beta$ tuned using the procedure in \cite{Tran-Thien_2023} ({\tt Lossy}$^\star$).
    \item \emph{Big-Little Decoder approach} \citep{kim2023speculative}, with both the original deterministic version (\texttt{BiLD}), and the variant where we apply Algorithm \ref{alg:gen-speed} to the target distribution $\mathbb{T}_{\textup{BiLD}}$ in \S\ref{sec:related} (\texttt{BiLD}$^*$).
\end{enumerate}

\textbf{Fine-tuned T5 cascades.} 
\todoakm{maybe the para heading can explicate fine-tuning here, vs few-shot for Gemma. Hari: done!}
Our experiments on T5 models are based on the setup in \cite{zhou2024distillspec}; see \S\ref{app:expt-setup} for details. We use T5-small (77M) as the small model, and either T5-large (800M) or T5-XL (3B) as the large model. In each case, we  supervised
 fine-tune these models on three  tasks: WMT EN$\rightarrow$DE translation \citep{bojar-EtAl:2014:W14-33}, CNN/DM summarization \citep{hermann2015teaching}, and XSum abstractive summarization \citep{narayan2018don}. 
\todoakm{do the latency numbers match what is reported in DistillSpec? Hari: our quality metrics match what they report; I'll be able to compare with their speed ups once I better understand how they compute relative latency}
We use temperatures $T=0, 0.1, 0.5, 1.0$, and block sizes $\gamma = 3, 5, 7$ (full results in \S\ref{app:expts}). 
Following the protocol in \cite{Leviathan:2023, zhou2024distillspec}, to measure latency, we evaluate the wall-clock decoding time with batch size 1.

\begin{figure}[t]
\centering
    \includegraphics[scale=0.27]{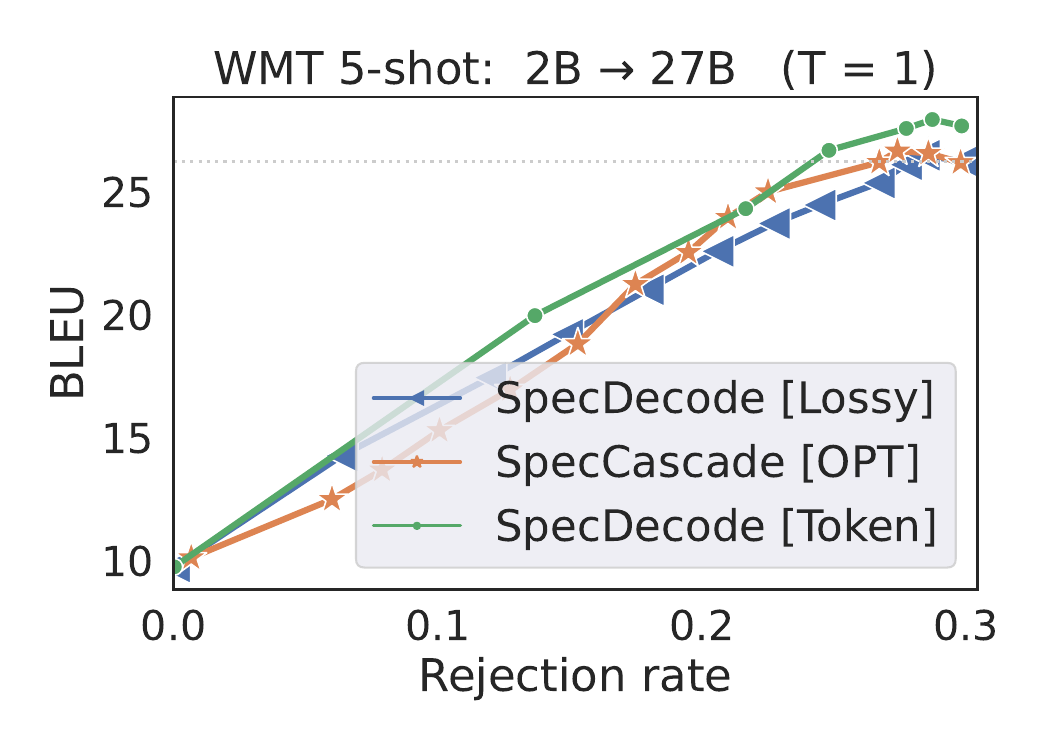}
    \includegraphics[scale=0.27]{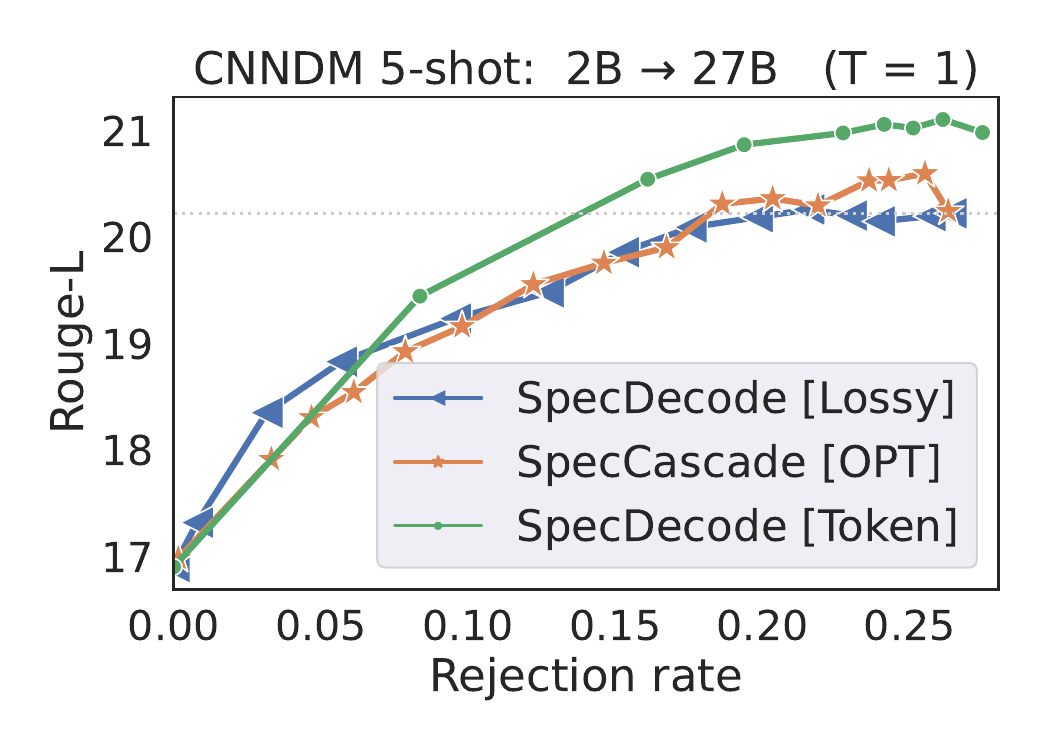}
    \includegraphics[scale=0.27]{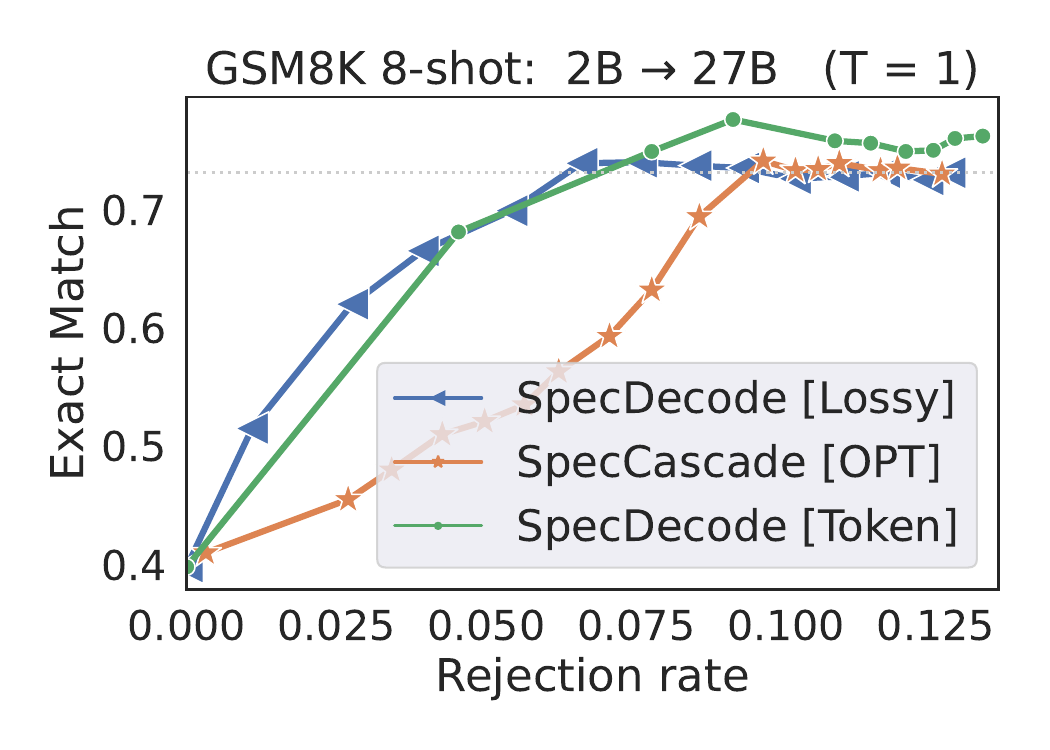}
    \includegraphics[scale=0.27]{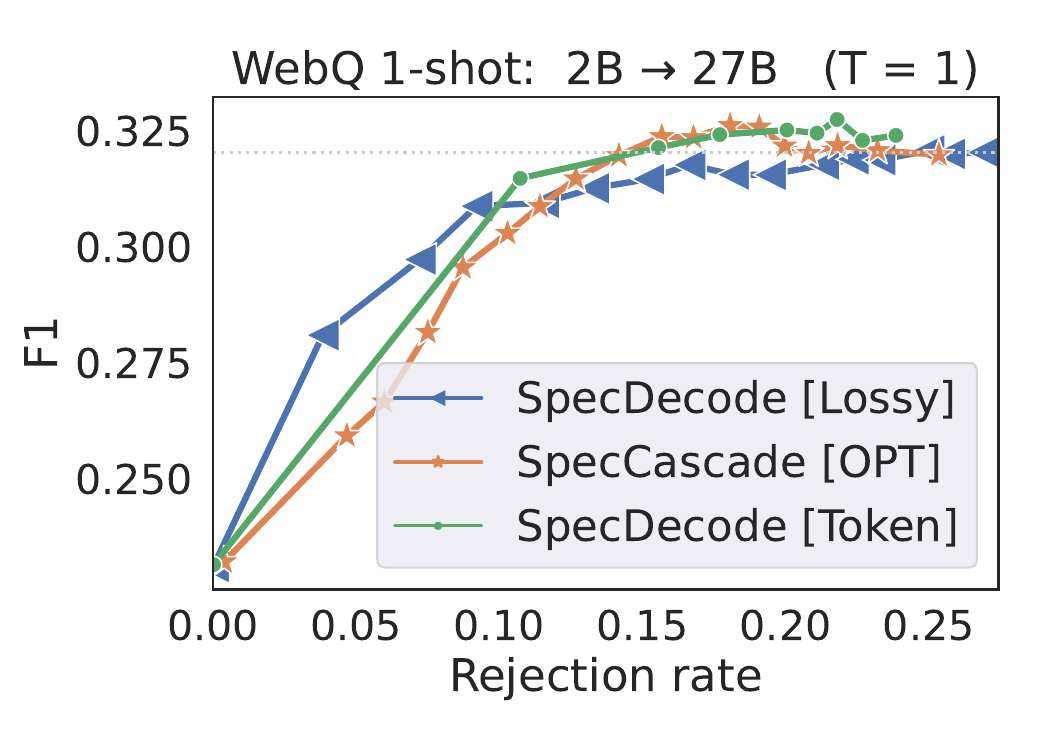}
    \includegraphics[scale=0.27]{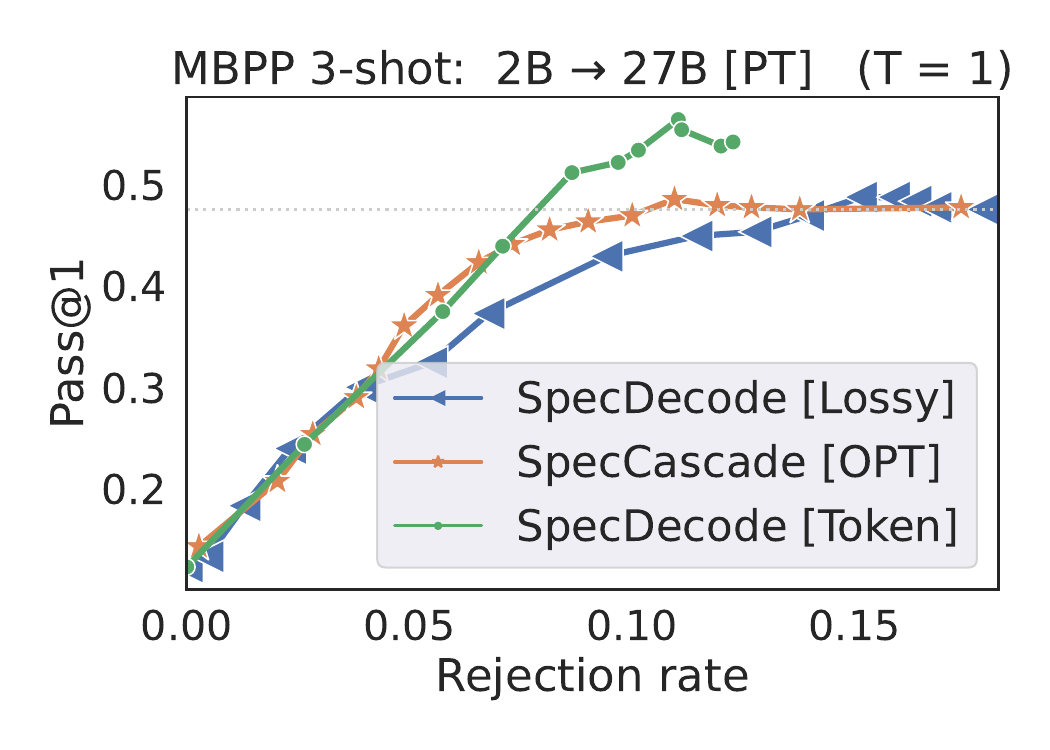}
    \includegraphics[scale=0.27]{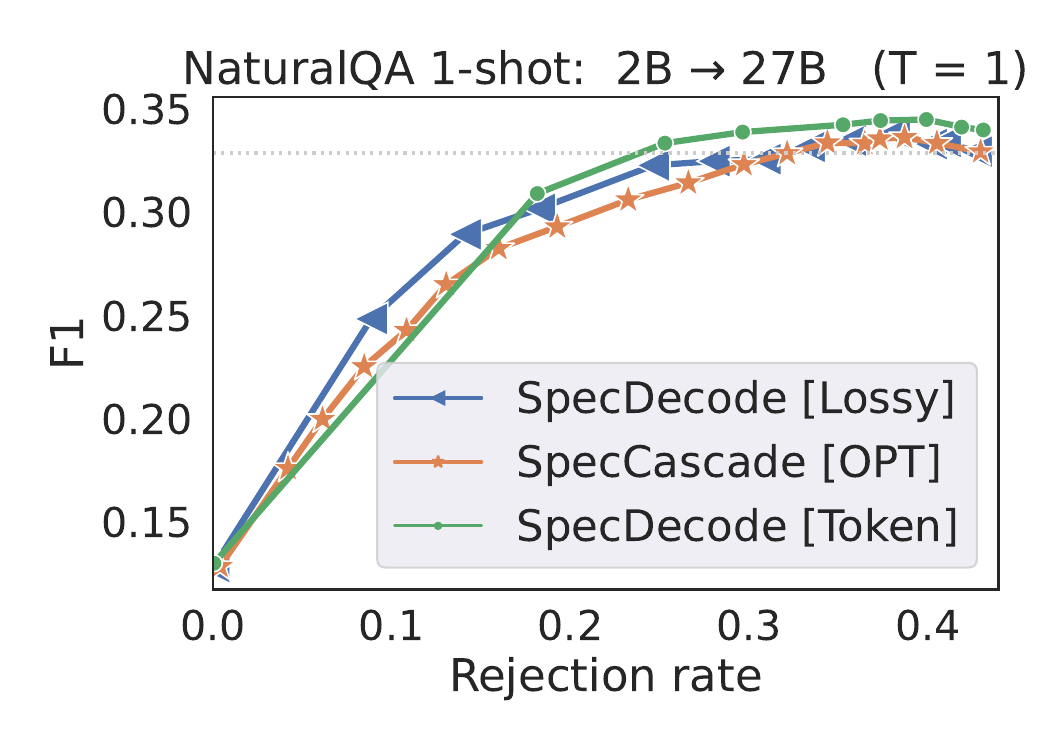}
    \caption{Plots of quality vs.\ rejection rate for methods that interleave Gemma 2B with Gemma 27B ($\gamma=1$).  We use instruction-tuned models; for MBPP we report additional results with pre-trained models. See \S\ref{app:expts-token-specific} for remaining plots, comparison to  (\ref{eq:sample-dep-01-plugin-v1}--\ref{eq:sample-dep-01-plugin-v2}) and 
    results on 2B $\rightarrow$ 9B cascades.
    }
    \label{fig:tradeoffs-gemma}
\end{figure}



In Figure \ref{fig:tradeoffs}, we present plots of quality vs.\ latency for the different methods. 
In each case, we vary the lenience parameter $\alpha$, and  plot either the BLEU or ROUGE-2 metric as a function of the relative latency to the larger model. For brevity, we include the three main baselines; in \S\ref{app:expts-bild}--\ref{app:expts-lossy},
we compare to \texttt{SpecDecode} [{\tt Lossy}$^\star$] \citep{Tran-Thien_2023} and the original {\tt BiLD} algorithm \cite{kim2023speculative}. 
Methods that use speculative execution are considerably faster than sequential token-level cascades ({\tt TokenCascade [Chow]}), although sequential cascades do have an advantage in the low-latency regimes. This is because unlike speculative approaches, which always call the large model after every $\gamma$ steps, a sequential cascade only invokes the large model when the small model defers. 
\todoakm{the legend is a bit hard to make out. we may  consider only including two plots per row. Hari: Did this for Fig 2. For Fig 3, I've made the font size bigger. Would you like two figures per row for Fig 3 too?}
\todoakm{further to the above, we may consider the value of including full deferral curves for both T5 and Gemma, versus having a summary like Table 2. Hari: hmm, I think reviewers usually like to see numbers in a tabulated form, especially when it comes to speed-ups}

In Table \ref{tab:speedup}, 
\todoakm{caption should clarify this is for T5. Hari: done!}
we report  (i) the reduction in latency from T5 cascades when matching the quality of the large model, and (ii) the best quality that each method can deliver without exceeding the latency of the large model. 
{\haricolor{blue}\texttt{SpecCascade} [{\tt Token}]  often yields  the highest speed-up and the best quality metrics, with {\tt OPT} coming in second.} The cascading approaches are often seen to fare poorly on both quality and latency metrics, with the exception of WMT, where \texttt{SeqCascade} yields non-trivial speed-ups.\todoasr{Do we need to briefly remind the reader that BiLD* is different from BiLD implementation?}
{\haricolor{blue}The reason the {\tt Token}-specific rule fares better than  {\tt OPT} and {\tt Diff} is because the latter compute their deferral decisions based on which of $q_t(\cdot)$ and $p_t(\cdot)$ is more \textit{peaked}; this can be a disadvantage when the sampled token is not  close to the distribution mode, which is likely to happen with higher temperatures. As shown in \S\ref{app:expts-temperature}, with lower temperatures, the gap between these rules diminishes.\todoakm{should we also call out the WMT results being sensitive to temperature? Hari: I think this is true of all methods. I'll be including plots with varying temperatures in the appendix. After I include that, I'll add a note in the main text.}
}

{\haricolor{blue}
\textbf{Few-shot Gemma cascades.} 
\todoakm{what about the MBPP results? Hari: TBD}
To evaluate the Gemma model cascades, we use few-shot prompting with 8 language benchmarks: WMT, CNN/DM, GSM8K, MBPP, SQuAD 2.0, WebQuestions, NaturalQA and TriviaQA; many of these feature in the SpecBench suite \citep{xia2024unlocking}.  Figure \ref{fig:tradeoffs-gemma}  presents plots of quality vs.\ rejection rate with a 2B drafter and 27B verifier for $\gamma=1$. For brevity, we only compare the methods that fare the best in the previous experiments. With the exception of TriviaQA, {\tt SpecCascade [Token]} is able to both match the 27B's quality at a lower  rejection rate and yield the best overall quality, often better than 27B. 
Since all three methods use the exact same implementation for speculative execution, a lower rejection rate directly translates to a lower latency. 
}

{\haricolor{red}
Interestingly, {\tt OPT} is not as effective as with T5. We attribute this to the differences in distributions between the two setups: with T5, the maximum token probability served as a good indicator of token accuracy for both $q$ and $p$; with Gemma, we expect the large model to have a closer alignment with the ground-truth distribution (due to it being several billion parameters apart from the smaller model), and hence using the large model probabilities to measure confidence for both the small and large model (\ref{eq:sample-dep-01-plugin-v3}) yields better trade-offs than comparing the modes from the two model distributions. 
}



\section{Conclusions}
We have proposed new speculative cascading techniques that use a combination of auto-regressive drafting and parallel verification to implement their deferral rule, and shown that they yield better cost-quality trade-offs than standard cascades and speculative decoding. {\haricolor{blue} A limitation of our approach is that while it offers a higher throughput, it also incurs a higher total compute cost compared to sequential cascades.} 
In the future, we wish to replace our plug-in estimators
with a router model \citep{gupta2024language} trained on ground-truth samples, 
 to improve the local deferral objective 
 at each position $t$ (\ref{eq:spec-def-risk}) 
 with a global 
objective, 
and to extend our proposal to {more than two models}.

\bibliographystyle{iclr2025_conference}
\bibliography{NeurIPS/references, NeurIPS/references_cascades}
\clearpage
\newpage
\appendix
\section{Proofs}
\label{app:proofs}
\subsection{Proof of Lemma \ref{lem:seq-def-opt}}
\begin{proof}
Expanding the loss in \eqref{eq:seq-def-risk}, we have:
\begin{align*}
    L_{\rm def}(r; x_{<t}) &=\big(1 - r(x_{<t})\big) \cdot \E_{x_t \sim \Pr(\cdot|x_{<t})}\left[  \ell(x_t, q_t)\right] + r(x_{<t}) \cdot \big(\E_{x_t \sim \Pr(\cdot|x_{<t})}\left[\ell(x_t, p_t)\right] + \alpha\big)\\
    &= r(x_{<t}) \cdot \left(\E_{x_t \sim \Pr(\cdot|x_{<t})}\left[\ell(x_t, p_t)\right] + \alpha
    - \E_{x_t \sim \Pr(\cdot|x_{<t})}\left[  \ell(x_t, q_t)\right]\right)
    \,+\, \E_{x_t \sim \Pr(\cdot|x_{<t})}\left[  \ell(x_t, q_t)\right]
\end{align*}
This objective is minimized by a deferral rule $r: \cV^{t-1} \rightarrow \{0,1\}$ that minimizes, for each prefix $x_{<t}$, the term within the parenthesis. Therefore the minimizer $r^*(x_{<t}) = 1$ whenever the term within the parenthesis is negative:
\[
\E_{x_t \sim \Pr(\cdot|x_{<t})}\left[\ell(x_t, p_t)\right] + \alpha
    - \E_{x_t \sim \Pr(\cdot|x_{<t})}\left[  \ell(x_t, q_t)\right] < 0,
\]
and $r^*(x_{<t}) = 0$ otherwise. Re-arranging the terms completes the proof.
\end{proof}

\subsection{Proof of Lemma \ref{lem:lossy-sd-target-distr}}
\begin{proof}
The proof follows straight-forwardly from the results in \citep{Tran-Thien_2023}. Recall from \S\ref{sec:prelims} that the lossy speculative decoding procedure of \citep{Tran-Thien_2023} accepts a draft token $x$ with probability:
\begin{equation}
    \kappa(x) = \min\left\{1, \frac{p(x)}{(1-\alpha)\cdot q(x)}\right\},
    \label{eq:lossy-sd-1}
\end{equation}
and replaces a rejected draft token with a token sampled from the residual distribution:
\begin{equation}
p_{\rm res}(x) = \normm \left(\max\left\{0,\, \frac{1}{\beta}\cdot p(x) - q(x)\right\}\right),
    \label{eq:lossy-sd-2}
\end{equation}
for parameters $\alpha \in [0, 1)$ and $\beta \geq 1 - \alpha$.

We need to show that running Algorithm \ref{alg:gen-speed} with the target distribution:
\[
\pi(x) = \max\left\{\min\left\{ q(x), \frac{p(x)}{1 - \alpha} \right\},  \frac{p(x)}{\beta}\right\}
\]
results in the same acceptance probability \eqref{eq:lossy-sd-1} and residual distribution \eqref{eq:lossy-sd-2}.

The acceptance probability for a draft token $x$ when running Algorithm \ref{alg:gen-speed} on $\pi$ is given by:
\[
\kappa^\pi(x) = \min\left\{1, \frac{\pi(x)}{ q(x)}\right\}.
\]
The corresponding residual distribution is given by:
\[
p^\pi_{\rm res}(x) = \normm\left(\max\left\{0,\,  \pi(x) - q(x)\right\}\right).
\]

We consider three possible cases: 

\textbf{Case (i): $q(x) > \frac{1}{1 - \alpha}\cdot p(x) \geq \frac{1}{\beta}\cdot p(x)$.} In this case, $\pi(x) = \frac{1}{1 - \alpha}\cdot p(x)$. As a result:
\[
\kappa^\pi(x) = \min\left\{1, \frac{p(x)}{(1 - \alpha)\cdot q(x)}\right\} =  \kappa(x);
\]
\begin{align*}
p^\pi_{\rm res}(x) &=  \normm\left( \max\left\{0,\,  \frac{1}{1 - \alpha} \cdot p(x) - q(x)\right\} \right) \\ 
&= 0 = \normm\left( \max\left\{0,\,  \frac{1}{\beta} \cdot p(x) - q(x)\right\} \right) = p_{\rm res}(x).
\end{align*}

\textbf{Case (ii): $\frac{1}{1 - \alpha}\cdot p(x) \geq \frac{1}{\beta}\cdot p(x) > q(x)$.} In this case, $\pi(x) = \frac{1}{\beta}\cdot p(x)$. As a result:
\[
\kappa^\pi(x) = \min\left\{1, \frac{p(x)}{\beta\cdot q(x)}\right\} = 1  =  \min\left\{1, \frac{p(x)}{(1 - \alpha)\cdot q(x)}\right\} = \kappa(x);
\]
\begin{align*}
p^\pi_{\rm res}(x) &=  \normm\left( \max\left\{0,\,  \frac{1}{\beta} \cdot p(x) - q(x)\right\} \right) = p_{\rm res}(x).
\end{align*}
\textbf{Case (iii): $\frac{1}{1 - \alpha}\cdot p(x) \geq q(x) \geq \frac{1}{\beta}\cdot p(x)$.} In this case, $\pi(x) = q(x)$. As a result:
\[
\kappa^\pi(x) = 1 =  
\min\left\{1, \frac{p(x)}{(1 - \alpha)\cdot q(x)}\right\} = \kappa(x);
\]
\[
p^\pi_{\rm res}(x) = 0 = \normm \left(\max\left\{0,\, \frac{1}{\beta}\cdot p(x) - q(x)\right\}\right) = p_{\rm res}(x).
\]
In all three cases, the acceptance probabilities and residual distributions are identical.
\end{proof}

\subsection{Proof of Lemma \ref{lem:token-acceptance-rate}}
\begin{proof}
Under a target distribution $\pi_t$, the probability of a draft token drawn from $q_t$ being is rejected is given by \citep{Leviathan:2023}:
\begin{align*}
\text{rejection probability} &= \sum_{v \in \cV} q_t(v) \cdot \left(1 - \min\left\{1, \frac{\pi_t(v)}{q_t(v)}\right\}\right)\\
&= 1- \sum_{v \in \cV}\min\left\{q_t(v), \pi_t(v)\right\}\\
&= 
\sum_{v \in \cV} \pi_t(v)- \sum_{v \in \cV}\min\left\{q_t(v), \pi_t(v)\right\}
\\
&=
\sum_{v \in \cV} \max\left\{0, \pi_t(v) - q_t(v)\right\}.
\end{align*}
Expanding $\pi$, the rejection probability becomes:
\[
\text{rejection probability} = \sum_{v \in \cV}
\max\left\{0,\, (1 - r(x_{<t})) \cdot q_t(v) + r(x_{<t}) \cdot p_t(v) - q_t(v) \right\}
\]
When $r(x_{<t}) = 1$, we have:
\begin{align*}
\text{rejection probability} &= 
\sum_{v \in \cV} \min\left\{0,\, p_t(v) - q_t(v) \right\} = D_{\tv}(p_t, q_t) = r(x_{<t}) \cdot D_{\tv}(p_t, q_t).
\end{align*}
When $r(x_{<t}) = 0$, we have:
\begin{align*}
\text{rejection probability} &= 
0 = r(x_{<t}) \cdot D_{\tv}(p_t, q_t),
\end{align*}
as desired.
\end{proof}

\subsection{Proof of Lemma \ref{lem:spec-def-opt}}
\begin{proof}
Expanding the deferral risk in \eqref{eq:spec-def-risk}, we have:
\begin{align*}
L_{\rm spec}(r; x_{<t}) &= r(x_{<t}) \cdot \left(
\E_{x_t \sim \Pr(\cdot|x_{<t})}\left[\ell(x_t, p_t)\right] + \alpha \cdot D_{\tv}(\bp, \bq) 
- \E_{x_t \sim \Pr(\cdot|x_{<t})}\left[ \ell(x_t, q_t)\right]
\right)\\ 
&\hspace{7cm}+ \E_{x_t \sim \Pr(\cdot|x_{<t})}\left[ \ell(x_t, q_t)\right].
\end{align*}
This objective is minimized by a deferral rule $r: \cV^{t-1} \rightarrow \{0,1\}$ that minimizes, for each prefix $x_{<t}$, the term within the parenthesis. Therefore the minimizer $r^*(x_{<t}) = 1$ whenever the term within the parenthesis is negative:
\[
\E_{x_t \sim \Pr(\cdot|x_{<t})}\left[\ell(x_t, p_t)\right] + \alpha \cdot D_{\tv}(\bp, \bq) 
- \E_{x_t \sim \Pr(\cdot|x_{<t})}\left[ \ell(x_t, q_t)\right] < 0,
\]
and $r^*(x_{<t}) = 0$ otherwise. Re-arranging the terms completes the proof.
\end{proof}

\subsection{Proof of Lemma \ref{lem:regret-01}}
For a fixed prefix $x_{<t}$, we can write the deferral risk in \eqref{eq:spec-def-risk} as:
\begin{align*}
    L_{\rm spec}(r; x_{<t}) &= r(x_{<t}) \cdot \left(
\E_{x_t \sim \Pr(\cdot|x_{<t})}\left[\ell(x_t, p_t)\right] + \alpha \cdot D_{\tv}(\bp, \bq) 
- \E_{x_t \sim \Pr(\cdot|x_{<t})}\left[ \ell(x_t, q_t)\right]
\right) + C,
\end{align*}
where $C$ is a term independent of the deferral rule $r$. Let $r^*: \cV^{t-1} \rightarrow\{0,1\}$ denote the optimal deferral rule that minimizes $L_{\rm spec}$ for any prefix $x_{<t}$. We then have:
\begin{align*}
    \lefteqn{L_{\rm spec}\left(\hat{r}_{\rm\tt OPT}; x_{<t}\right) - L_{\rm spec}\left(r^*; x_{<t}\right)}\\
    &=
    \left(\hat{r}_{\rm\tt OPT}(x_{<t}) - r^*(x_{<t})\right) \cdot \left(
\E_{x_t \sim \Pr(\cdot|x_{<t})}\left[\ell(x_t, p_t)\right] + \alpha \cdot D_{\tv}(\bp, \bq) 
- \E_{x_t \sim \Pr(\cdot|x_{<t})}\left[ \ell(x_t, q_t)\right]
\right).
\end{align*}
Adding and subtracting $\max_{v} q_t(v) - \max_{v} p_t(v)$ to the term within the second parenthesis, we get:
\begin{align}
\lefteqn{L_{\rm spec}\left(\hat{r}_{\rm\tt OPT}; x_{<t}\right) - L_{\rm spec}\left(r^*; x_{<t}\right)}
\nonumber\\\nonumber
    &=
    \left(\hat{r}_{\rm\tt OPT}(x_{<t}) - r^*(x_{<t})\right) \cdot \left(
\max_{v} q_t(v) + \alpha \cdot D_{\tv}(\bp, \bq) 
- \max_{v} p_t(v)
\right)\\\nonumber
&~~~~~+ 
\left(\hat{r}_{\rm\tt OPT}(x_{<t}) - r^*(x_{<t})\right) \cdot \left(
\E_{x_t \sim \Pr(\cdot|x_{<t})}\left[\ell(x_t, p_t)\right] 
-\E_{x_t \sim \Pr(\cdot|x_{<t})}\left[\ell(x_t, q_t)\right] - \max_{v} q_t(v) + \max_{v} p_t(v)
\right)\\\nonumber
&=
    \left(\hat{r}_{\rm\tt OPT}(x_{<t}) - r^*(x_{<t})\right) \cdot \left(
\max_{v} p_t(v) + \alpha \cdot D_{\tv}(\bp, \bq) 
- \max_{v} q_t(v)
\right)\\\nonumber
&~~~~~+ 
\left(\hat{r}_{\rm\tt OPT}(x_{<t}) - r^*(x_{<t})\right) \cdot \left(
\E_{x_t \sim \Pr(\cdot|x_{<t})}\left[\ell(x_t, p_t)\right] - 1 + \max_{v} p_t(v)\right)\\\nonumber
&~~~~~+ 
\left(\hat{r}_{\rm\tt OPT}(x_{<t}) - r^*(x_{<t})\right) \cdot\left(1- \max_{v} q_t(v) - \E_{x_t \sim \Pr(\cdot|x_{<t})}\left[\ell(x_t, q_t)\right] \right)\\\nonumber
&=
    \left(\hat{r}_{\rm\tt OPT}(x_{<t}) - r^*(x_{<t})\right) \cdot \left(
\max_{v} p_t(v) + \alpha \cdot D_{\tv}(\bp, \bq) 
- \max_{v} q_t(v)
\right)\\\nonumber
&~~~~~+ 
\left|\hat{r}_{\rm\tt OPT}(x_{<t}) - r^*(x_{<t})\right| \cdot \left|
\E_{x_t \sim \Pr(\cdot|x_{<t})}\left[\ell(x_t, p_t)\right] - 1 + \max_{v} p_t(v)\right|\\\nonumber
&~~~~~+ 
\left|\hat{r}_{\rm\tt OPT}(x_{<t}) - r^*(x_{<t})\right| \cdot\left|1- \max_{v} q_t(v) - \E_{x_t \sim \Pr(\cdot|x_{<t})}\left[\ell(x_t, q_t)\right] \right|\\\nonumber
&=
\underbrace{
\left(\hat{r}_{\rm\tt OPT}(x_{<t}) - r^*(x_{<t})\right) \cdot \left(
\max_{v} p_t(v) + \alpha \cdot D_{\tv}(\bp, \bq) 
- \max_{v} q_t(v)
\right)}_{\text{term}_1}\nonumber\\
&~~~~~+ 
\underbrace{
\left|
\E_{x_t \sim \Pr(\cdot|x_{<t})}\left[\ell(x_t, p_t)\right] - 1 + \max_{v} p_t(v)\right|}_{\text{term}_2}
+ 
\underbrace{
\left|1- \max_{v} q_t(v) - \E_{x_t \sim \Pr(\cdot|x_{<t})}\left[\ell(x_t, q_t)\right] \right|}_{\text{term}_3}
\label{eq:regret-helper}
\end{align}
where we have used the fact that 
$\left|\hat{r}_{\rm\tt OPT}(x_{<t}) - r^*(x_{<t})\right| \leq 1. 
$ 

We bound each term separately. For the first term,  consider two cases: (i) $\max_{v} p_t(v) + \alpha \cdot D_{\tv}(\bp, \bq) 
- \max_{v} q_t(v) \leq 0$ and (ii) $\max_{v} p_t(v) + \alpha \cdot D_{\tv}(\bp, \bq) 
- \max_{v} q_t(v) > 0$. When (i) holds, $\hat{r}_{\rm\tt OPT}(x_{<t}) = 1$; so irrespective of whether $r^*(x_{<t})$ is 0 or 1,
\begin{align*}
\text{term}_1
&\leq \max_{v} p_t(v) + \alpha \cdot D_{\tv}(\bp, \bq) 
- \max_{v} q_t(v) \leq 0
\end{align*}
When (ii) holds, $\hat{r}_{\rm\tt OPT}(x_{<t}) = 0$; so irrespective of whether $r^*(x_{<t})$ is 0 or 1,
\begin{align*}
\text{term}_1
&\leq -\left(
\max_{v} p_t(v) + \alpha \cdot D_{\tv}(\bp, \bq) 
- \max_{v} q_t(v)
\right) < 0.
\end{align*}
Thus we have:
\begin{align}
    \text{term}_1 \leq 0.
    \label{eq:term1}
\end{align}

We next move to the second term. Since $\ell=\ell_\zo$, we have:
\begin{align*}
\text{term}_2 
&=
\left|
\E_{x_t \sim \Pr(\cdot|x_{<t})}\left[\ell(x_t, p_t)\right] - 1 + \max_{v} p_t(v)\right|\\
&= 
\left|\E_{x_t \sim \Pr(\cdot|x_{<t})}\left[1\left(x_t \ne \argmax_v p_t(v) \right)\right] -1 + \max_{v} p_t(v)\right|\\
&= \left|\max_{v} p_t(v) - \sum_{x_t}\Pr(x_t|x_{<t}) \cdot 1\left(x_t = \argmax_v p_t(v) \right)\right|
\end{align*}
Suppose $v^* \in \argmax_v p_t(v)$, then:
\begin{align}
\text{term}_2 &= \left| p_t(v^*) - \Pr(v^*|x_{<t})\right| \leq \max_v \left| p_t(v) - \Pr(v|x_{<t})\right|.
\label{eq:term2}
\end{align}

Similarly, we can show that:
\begin{align}
\text{term}_3 &\leq \max_v \left| q_t(v) - \Pr(v|x_{<t})\right|.
\label{eq:term3}
\end{align}

Substituting \eqref{eq:term1}--\eqref{eq:term3} in \eqref{eq:regret-helper} completes the proof.

\begin{table}[!t]
    \centering
    \resizebox{\columnwidth}{!}{%
    \begin{tabular}{llll}
    \toprule
    Inference strategy & 
    Deferral decision $\delta(q,p)$ &
    Target distribution $\pi(x)$ & Execution\\
    \midrule
         SpecDecoding  \cite{Leviathan:2023} &
        - &
         $p(x)$
          & Speculative
         \\[2pt]
         Lossy SpecDecoding  \citep{Tran-Thien_2023} & 
         - &
         $\max\{\min\{ p(x), \frac{q(x)}{1 - \alpha} \}, 
         \frac{q(x)}{\beta}\}
         $
         & Speculative
         \\[2pt]
         BiLD* \citep{kim2023speculative} & 
        $\1\big(\, D(q, p) > \alpha \big)$
         &
         $ (1 - \delta) \cdot q(x) + \delta \cdot p(x)$
         & Speculative
         \\
         \midrule
         Cascade [{\tt Chow}] \citep{chow1970optimum} &
         $\1\big(\max_v q(v) < 1 - \alpha\big)$ 
         & $ (1 - \delta) \cdot q(x) + \delta \cdot p(x)$
         & Sequential
         \\[2pt]
         Cascade [{\tt ChowLog}]  &
         $\1\big(\,\entropy(q)> \alpha\big)$ 
         & $ (1 - \delta) \cdot q(x) + \delta \cdot p(x)$
         & Sequential
         \\
        \midrule
         Oracle [{\tt Diff}] \citep{jitkrittum2024does}
         & $\1\big(\max_v q(v) < \max_v p(v) - \alpha\big)$ 
         & $ (1 - \delta) \cdot q(x) + \delta \cdot p(x)$
         & Oracle
         \\[2pt]
        Oracle [{\tt DiffLog}]
         & $\1\big(\,\entropy(p) < \entropy(q) - \alpha\big)$ 
         & $ (1 - \delta) \cdot q(x) + \delta \cdot p(x)$
         & Oracle
         \\
         \midrule
         SpecCascade [{\tt Chow}]
         & $\1\big(\max_v q(v) < 1 - \alpha\big)$ 
         &$ (1 - \delta) \cdot q(x) + \delta \cdot p(x)$
         & Speculative
         \\[2pt]
         SpecCascade [{\tt ChowLog}]
         & $\1\big(\,\entropy(q) > \alpha\big)$ 
         &$ (1 - \delta) \cdot q(x) + \delta \cdot p(x)$
         & Speculative
         \\[2pt]
         SpecCascade [{\tt Diff01}]
         & $\1\big(\max_v q(v) < \max_v p(v)  - \alpha\big)$ 
         &$ (1 - \delta) \cdot q(x) + \delta \cdot p(x)$
         & Speculative
         \\[2pt]
         SpecCascade [{\tt DiffLog}]
         & $\1\big(\,\entropy(p) < \entropy(q) - \alpha\big)$ 
         &$ (1 - \delta) \cdot q(x) + \delta \cdot p(x)$
         & Speculative
         \\[2pt]
         SpecCascade [{\tt OPT01}]
         & $\1\big(\max_v q(v) < \max_v p(v)  - \alpha \cdot D_{\tv}(p, q)\big)$ 
         &$ (1 - \delta) \cdot q(x) + \delta \cdot p(x)$
         & Speculative\\[2pt]
         SpecCascade [{\tt OPTLog}]
         & $\1\big(\,\entropy(p) < \entropy(q) - \alpha\cdot D_{\tv}(p, q)\big)$ 
         & $ (1 - \delta) \cdot q(x) + \delta \cdot p(x)$
         & Speculative
         \\
          \bottomrule
    \end{tabular}
    }
    \vspace{4pt}
    \caption{Target distributions associated with different inference algorithms, where $\alpha$ is a free parameter and $\beta \geq 1-\alpha$ is a parameter dependent on $q, p$ and $\alpha$. The last column indicates whether the  execution is sequential (Algorithm \ref{alg:seq-sampling}), via an oracle (Algorithm \ref{alg:oracle}), or speculative (Algorithm \ref{alg:spec-cascade}) \citep{Leviathan:2023}. 
    The third row presents a variant of the BiLD algorithm of \cite{kim2023speculative}, where $D(q, p)$ is a measure of discrepancy between $q$ and $p$; the original algorithm differs from \citep{Leviathan:2023} in the use of a deterministic speculative decoding procedure with a dynamic draft window
    (see \S\ref{sec:related}). 
    }
    \vspace{-10pt}
    \label{tab:targets-full}
\end{table}

\section{Derivation of Chow's rule}
\label{app:chow}
We show below that Chow's rule is a plug-in estimator to the optimal solution to the following objective
\begin{align}
L_{\rm rej}(r; x_{<t}) &=\E_{x_t \sim \Pr(\cdot|x_{<t})}\Big[ \big(1 - r(x_{<t})\big) \cdot \ell(x_t, q_t) + r(x_{<t}) \cdot \alpha \Big],
\label{eq:seq-rej-risk}
\end{align}
where the deferral rule is penalized with a constant penalty $\alpha \in [0, 1]$ for choosing to defer to the large model. 

Following the same steps as Lemma \ref{lem:seq-def-opt}, it is easy to show:
\begin{lemma}
The minimizer of \eqref{eq:seq-rej-risk} is of the form:
\begin{align}
    r^*(x_{<t}) = 1 ~~\iff~~ \E_{x_t \sim \Pr(\cdot|x_{<t})}\left[\ell(x_t, q_t)\right] \,>\, \alpha.
    \label{eq:seq-rej-opt}
\end{align}
\end{lemma}
If $\ell=\ell_\zo$, one may employ a plug-in estimator to \eqref{eq:seq-rej-opt} by replacing the expected 0-1 loss over $q_t$ with $1 - \max_v q_t(v)$, giving us $\hat{r}_{\rm \tt Chow}(x_{< t})$ in \eqref{eq:chow01}. If $\ell=\ell_{\log}$, one may replace the expected log loss over $q_t$ with the entropy of $q_t$, giving us:
\begin{align}
    \hat{r}_{\rm \tt ChowLog}(x_{<t}) = 1 ~\iff~\entropy\big(q(\cdot|x_{<t})\big) > \alpha,
    \label{eq:chowlog}
\end{align}
where $\entropy(q) = -\sum_{v \in \cV} q(v) \cdot \log(q(v)).$

\section{Optimal Deferral: Additional Discussion}
\label{app:log}
We provide additional discussion for the optimal deferral rules derived in \S\ref{sec:cascade-meets-speed} and \S\ref{sec:spec-cascades}.

\subsection{Optimal sequential deferral when $\ell=\ell_{\log}$}
Recall that the optimal deferral rule for a sequential cascade in Lemma \ref{lem:seq-def-opt} takes the form:
\begin{align*}
    r^*(x_{<t}) = 1 ~~\iff~~ \E_{x_t \sim \Pr(\cdot|x_{<t})}\left[\ell(x_t, q_t)\right] \,>\, \E_{x_t \sim \Pr(\cdot|x_{<t})}\left[\ell(x_t, p_t)\right] + \alpha \cdot D_{\textup{\tv}}(\bp, \bq).
\end{align*}

When $\ell = \ell_{\log}$, we may use the entropy $-\sum_v q_t(v)\cdot \log(q_t(v))$ from $q_t$ as an estimate of its expected log-loss, and similarly for $p_t$, giving us the plug-in estimator: 
\begin{align}
\textstyle
    \hat{r}_{\rm\tt DiffLog}(x_{<t}) = 1 ~~\iff~~ \sum_v q_t(v)\cdot \log(q_t(v))  \,<\, \sum_v p_t(v)\cdot \log(p_t(v)) - \alpha.
    \label{eq:seq-opt-def-log-plugin}
\end{align}

\subsection{Optimal speculative deferral when $\ell=\ell_{\log}$}
\label{app:opt-log}
Recall that the optimal deferral rule for a speculative cascade in Lemma \ref{lem:spec-def-opt} takes the form:
\begin{align*}
    r^*(x_{<t}) = 1 ~~\iff~~ \E_{x_t \sim \Pr(\cdot|x_{<t})}\left[\ell(x_t, q_t)\right] \,>\, \E_{x_t \sim \Pr(\cdot|x_{<t})}\left[\ell(x_t, p_t)\right] + \alpha \cdot D_{\textup{\tv}}(\bp, \bq).
\end{align*}
When $\ell=\ell_{\log}$, one may construct a \textit{plug-in} estimator for the above rule by replacing the expected log loss with the entropy from the distribution:
\begin{align}
\textstyle
    \hat{r}_{\rm\tt OPTLog}(x_{<t}) = 1 ~\iff~ \sum_v q_t(v)\cdot \log(q_t(v))  \,<\, \sum_v p_t(v)\cdot \log(p_t(v)) - \alpha\cdot D_{\textup{\tv}}(\bp, \bq).
    \label{eq:spec-opt-def-log-plugin}
\end{align}

\begin{lemma}[{Regret bound for $\hat{r}_{\rm\tt OPTLog}$}] 
Suppose $\ell = \ell_{\log}$. Suppose for a fixed $x_{<t}$,  $|\log(q(v))| \leq B_q$ and $|\log(p(v))| \leq B_p,\, \forall v \in \cV$, for some $B_q, B_p > 0$. Then:
\[
L_{\rm spec}(r_{\rm\tt OPT}; x_{<t}) - \min_r L_{\rm spec}(r; x_{<t}) ~\leq~ \textstyle B_q \cdot \sum_{v \in \cV} \big|\Pr(v|x_{<t}) - q_t(v)\big| \,+\, B_p \cdot \sum_{v \in \cV} \big|\Pr(v|x_{<t}) - p_t(v)\big|.
\]
\vspace{-10pt}
\label{lem:regret-log}
\end{lemma}
\begin{proof}
The proof follows similar steps to that for Lemma \ref{lem:regret-01}, except in bounding the resulting $\text{term}_2$ and $\text{term}_3$ for the log loss. In this case, 
\begin{align*}
\text{term}_2 
&=
\left|
\E_{x_t \sim \Pr(\cdot|x_{<t})}\left[\log(p_t(x_t))\right] - \sum_v p_t(v)\cdot \log(p_t(v))\right|\\
&= 
\left|\sum_v \Pr(v|x_{<t})\cdot \log(p_t(v)) -\sum_v p_t(v)\cdot \log(p_t(v))\right|\\
&\leq
\sum_v\left| \Pr(v|x_{<t}) -\sum_v p_t(v)\right|\cdot \log(p_t(v))\\
&\leq B_p \cdot\sum_v\left| \Pr(v|x_{<t}) -\sum_v p_t(v)\right|.
\end{align*}
Similarly,
\begin{align*}
\text{term}_3 
&\leq
\sum_v\left| \Pr(v|x_{<t}) -\sum_v p_t(v)\right|\cdot \log(p_t(v))\\
&\leq B_q \cdot\sum_v\left| \Pr(v|x_{<t}) -\sum_v q_t(v)\right|.
\end{align*}
Plugging these bounds into the  equivalent of \eqref{eq:regret-helper} in Lemma \ref{lem:regret-01} for the log-loss completes the proof.
\end{proof}

\subsection{Optimal speculative deferral for greedy decoding}
\label{app:greedy}
\begin{lemma}
\label{lem:greedy-special-case}
When $T \rightarrow 0$,  running Algorithm \ref{alg:spec-cascade} with  $\tilde{r}_{\rm\tt OPT}$ as the deferral rule and $\tilde{q}_t$ as the drafter is equivalent to running it with  $\hat{r}_{\rm\tt Diff}$ in \eqref{eq:seq-opt-def-01-plugin} as the deferral rule and $\tilde{q}_t$ as the drafter.
\end{lemma}
\begin{proof}
Note that under greedy inference, $\tilde{q}_t$ $\tilde{p}_t$ are one-hot encodings of $\argmax_v q_t(v)$ and $\argmax_v p_t(v)$ respectively. As a result,
\[
D_\tv(\tilde{q}_t, \tilde{p}_t) = \1\left(\argmax_v q_t(v) \ne \argmax_v p_t(v)\right). 
\]
When running Algorithm \ref{alg:spec-cascade} with  $\tilde{r}_{\rm\tt OPT}$ as the deferral rule, we will accept a draft token $v$ with probability:
\[
\kappa(v) = \min\left\{1, \frac{(1-\delta_{\rm OPT}(q,p)) \cdot \tilde{q}(v) + \delta_{\rm OPT}(q,p) \cdot \tilde{p}(v)}{\tilde{q}(v)}\right\} 
\]
where $\delta_{\rm OPT}(q,p) ~=~ \1\left(\max_{v} q(v)  \,<\, \max_{v} p(v) - \alpha \cdot \1\left(\argmax_v q(v) \ne \argmax_v p(v)\right)\right)$.
When $\argmax_v q(v) = \argmax_v p(v)$, then $\tilde{q} = \tilde{p}$, and irrespective of the outcome of $\delta(q,p),$ we have that $\pi(v) = 1$. When $\argmax_v q(v) \ne \argmax_v p(v)$, then 
$$\pi(v) = 1-\delta_{\rm OPT}(q,p) = \1\left(\max_{v} q(v)  \,\geq\, \max_{v} p(v) - \alpha\right) = 1 - 
\delta_{\rm \tt Diff}(q,p).$$

When a token gets rejected, we sample a new token from the residual distribution:
\[
p_{\rm res}(v) \propto \min\{0, (1-\delta_{\rm OPT}(q,p)) \cdot \tilde{q}(v) + \delta_{\rm OPT}(q,p) \cdot \tilde{p}(v) - \tilde{q}(v)\}
= \delta_{\rm OPT}(q,p) \cdot \min\{0, \tilde{p}(v) - \tilde{q}(v)\}
\]
When $\argmax_v q(v) = \argmax_v p(v)$, $p_{\rm res}(v) = 0$. When $\argmax_v q(v) \ne \argmax_v p(v)$, 
$$
p_{\rm res}(v) \propto \delta_{\rm OPT}(q,p) \cdot \min\{0, \tilde{p}(v) - \tilde{q}(v)\} = \delta_{\rm Diff}(q,p) \cdot \min\{0, \tilde{p}(v) - \tilde{q}(v)\}.
$$
Thus both the acceptance probability and the residual distribution are the same as the one we would have used had we run Algorithm \ref{alg:spec-cascade} with  $\hat{r}_{\rm\tt Diff}$ as the deferral rule.
\end{proof}
~~\\[-1cm]

\subsection{Equivalence between \eqref{eq:spec-def-constrained} and \eqref{eq:spec-def-risk}}
\label{app:equivalence}
Since the prefix $x_{<t}$ is fixed in \eqref{eq:spec-def-constrained}, the constrained optimization we seek to solve is of essentially of the following form:
\[
\min_{r \in \{0,1\}} (1-r) \cdot c_0 + r \cdot c_1 ~~~\text{s.t.}~~~ r \cdot c_2 \leq B,
\]
for some coefficients $c_0, c_1, c_2 > 0$. Since $r$ is a binary variable, we may formulate an equivalent unconstrained problem with the same minimizer:
\[
\min_{r \in \{0,1\}} (1-r) \cdot c_0 + r \cdot c_1 + \alpha \cdot r \cdot c_2,
\]
where we choose $\alpha = 0$ when $c_2 \leq B$ and choose an $\alpha > \frac{1}{c_2} \cdot (c_0 - c_1)$ otherwise. This unconstrained optimization problem is of the form in \eqref{eq:spec-def-risk}.

\section{Token-specific Speculative Cascade}
\label{app:sample-dep}

We provide a modification of Algorithm \ref{alg:spec-cascade} to accommodate the token-specific deferral rules in \S\ref{sec:sample-dep}.
\begin{algorithm}[H]\small
\caption{{\tt{TokenSpecCascade}}}
\label{alg:token-spec}
\begin{algorithmic}
\Require Models $q$, $p$, Token-specific deferral rule $r$, Prefix $x_{<t}$, Block size $\gamma$
\State $\bT_{\tt Token}(q,p)(v) \defeq q(v) \cdot (1 - r(x_{<t}, v))  +  p(v) \cdot \sum_{v' \in \cV} r(x_{<t}, v') \cdot q(v')$
\Ensure $\text{\tt{GenSpecSample}}(q, p, \bT_{\tt Token}, x_{<t}, \gamma)$
\end{algorithmic}
\end{algorithm}

\textbf{Optimal token-specific deferral.} Similar to \S\ref{sec:opt-spec}, we may consider deriving the optimal token-specific deferral rule. We start by formulating a similar optimization objective. For a fixed prefix $x_{<t}$, this would look like:
\begin{align}
\min_{r}&~~\E_{v \sim \Pr(\cdot|x_{<t})}\Big[ \ell(v, \pi_{\tt Token})\big) \Big]
\label{eq:token-spec-def-constrained}\\
&
\text{s.t.}~~
 D_{\tv}(\pi_{\tt Token}, \bq)  \,\leq\, B,
\nonumber
\end{align}
where 
$\pi_{\tt Token}(v) \defeq (1 - r(x_{<t}, v)) \cdot q_t(v) +  \eta  \cdot p_t(v)$ is the target distribution resulting from the choice of $r$, $\eta = \sum_{v' \in \cV} r(x_{<t}, v') \cdot q_t(v')$ is a normalization term, and $B > 0$ is a budget parameter.

However, unlike \S\ref{sec:opt-spec}, the above constrained optimization problem does not lend itself to a simple closed-form solution. In some highly simplistic special cases, we may be able to derive a solution. For example, suppose $\ell=\ell_\zo$,  and the mode of $q_t$ coincides with that of $\Pr(\cdot|x_{<t})$, i.e., $\argmax_v q_t(v) = \argmax_v \Pr(v|x_{<t})$; then the optimal token-specific rule is given by $r(x_{<t}, v) = 0,$ for all $v \in \cV$. 

Under more realistic cases, we may not be able to derive a solution as simple as the {\tt OPT} rule in \eqref{eq:spec-opt-def-01-plugin}. 
Therefore, in our experiments, we employ the three heuristic rules in equations \ref{eq:sample-dep-01-plugin-v1}--\ref{eq:sample-dep-01-plugin-v3}, which are motivated by the form of the simpler {\tt Diff} rule in \eqref{eq:seq-opt-def-01-plugin}.

\section{Additional Experimental Details}
\label{app:expts}
We provide additional details about our experimental setup and additional experimental results. \textbf{We will release code and an illustrative tutorial notebook along with the final manuscript.}
\begin{figure}[t]
    \centering
    \includegraphics[scale=0.33]{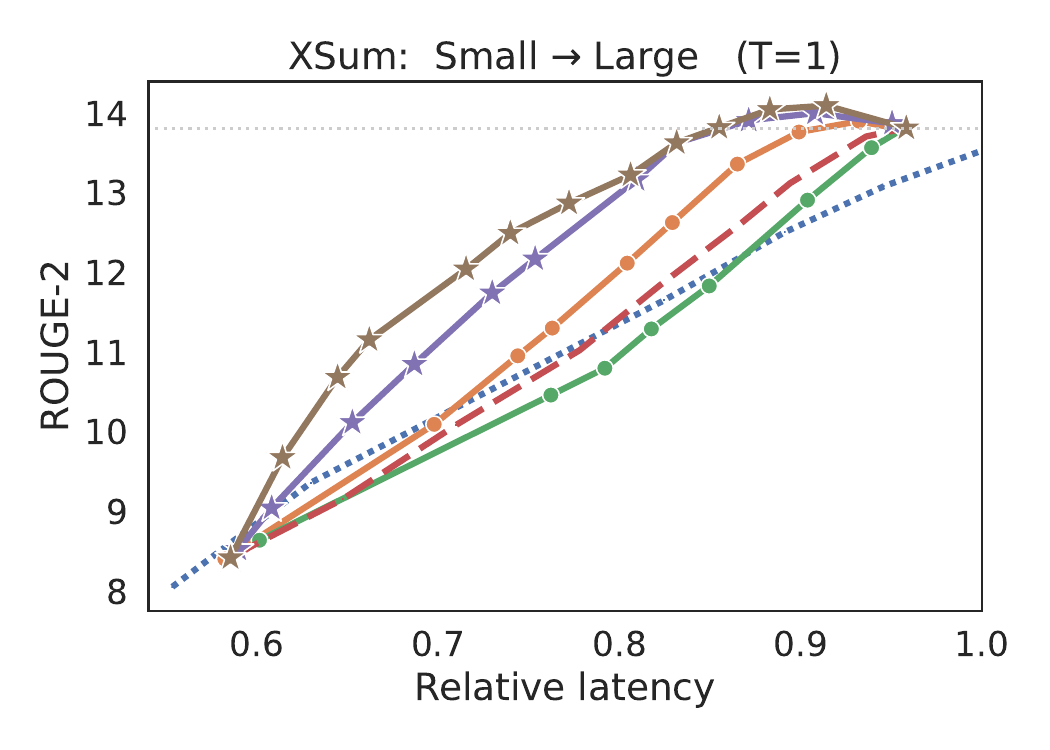}
    \includegraphics[scale=0.33]{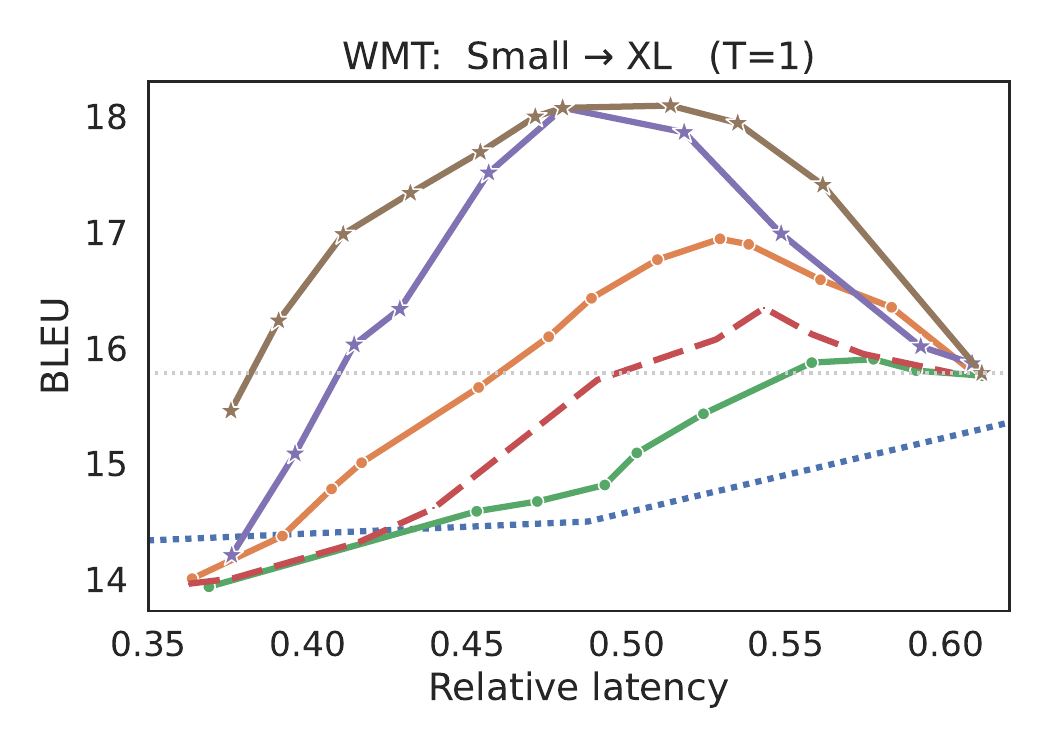}
    \includegraphics[scale=0.33]{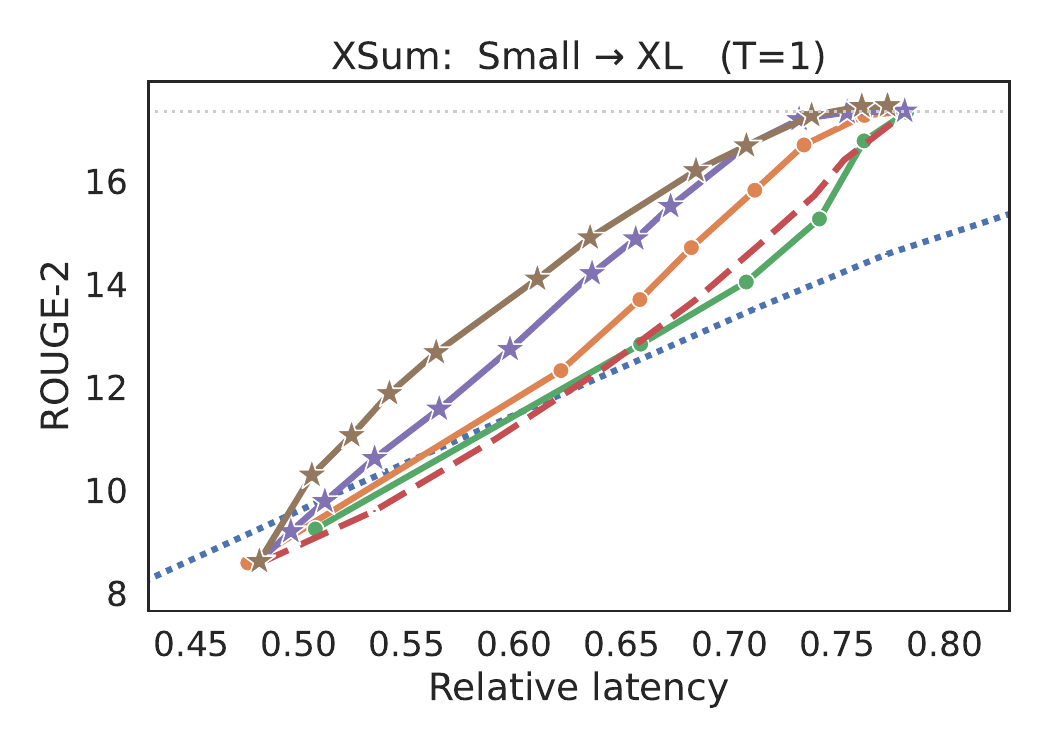}
    \includegraphics[scale=0.33]{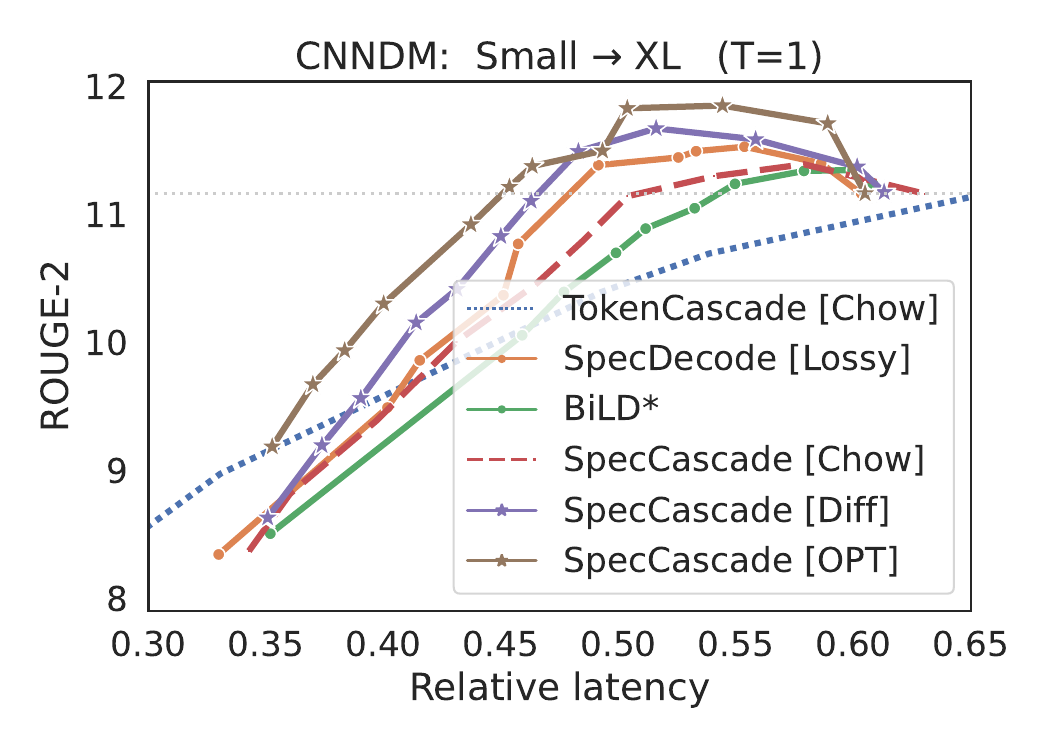}
    \vspace{-5pt}
    \caption{Plots of quality vs.\ latency for T5 models with temperature $T=1$ and block size $\gamma=5$. We include \textbf{T5 plots not included in Figure \ref{fig:tradeoffs}} in the main text. Each method interleaves T5-small with T-large (or T5-XL). The $x$-axis tracks the latency \emph{relative} to that of  calling  the large model on all inputs.
    The horizontal dotted line denotes the quality of the large model. 
    }
    \label{fig:tradeoffs-additional-temp-sample}
\end{figure}

\begin{figure}[t]
    \centering
    \includegraphics[scale=0.33]{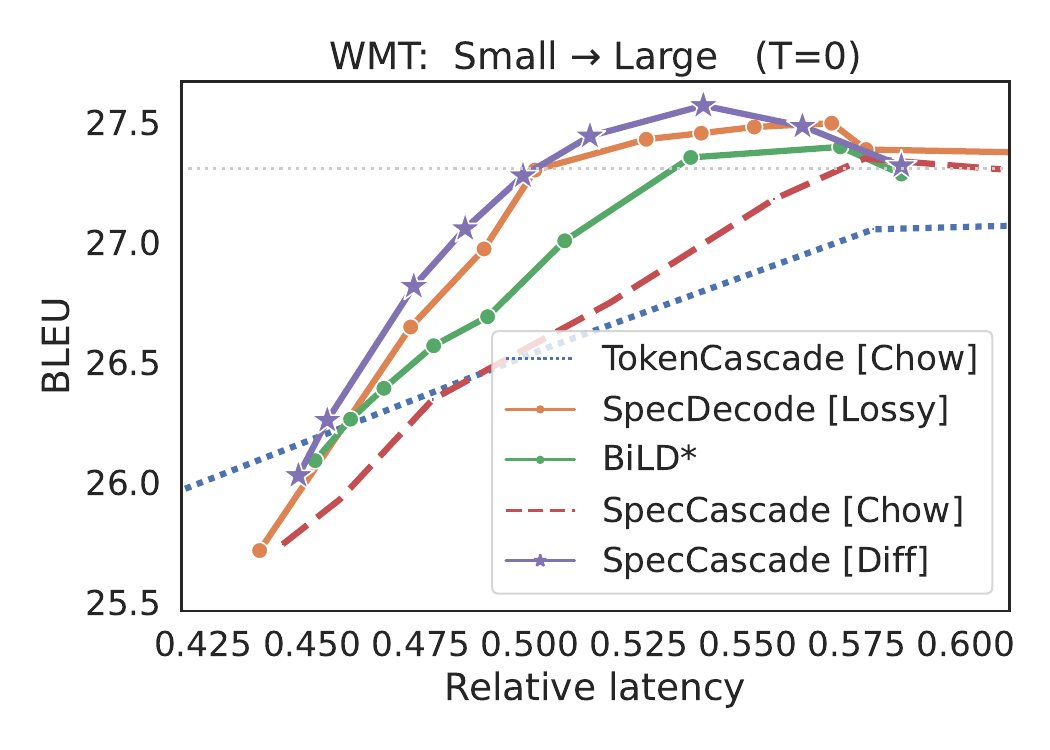}
    \includegraphics[scale=0.33]{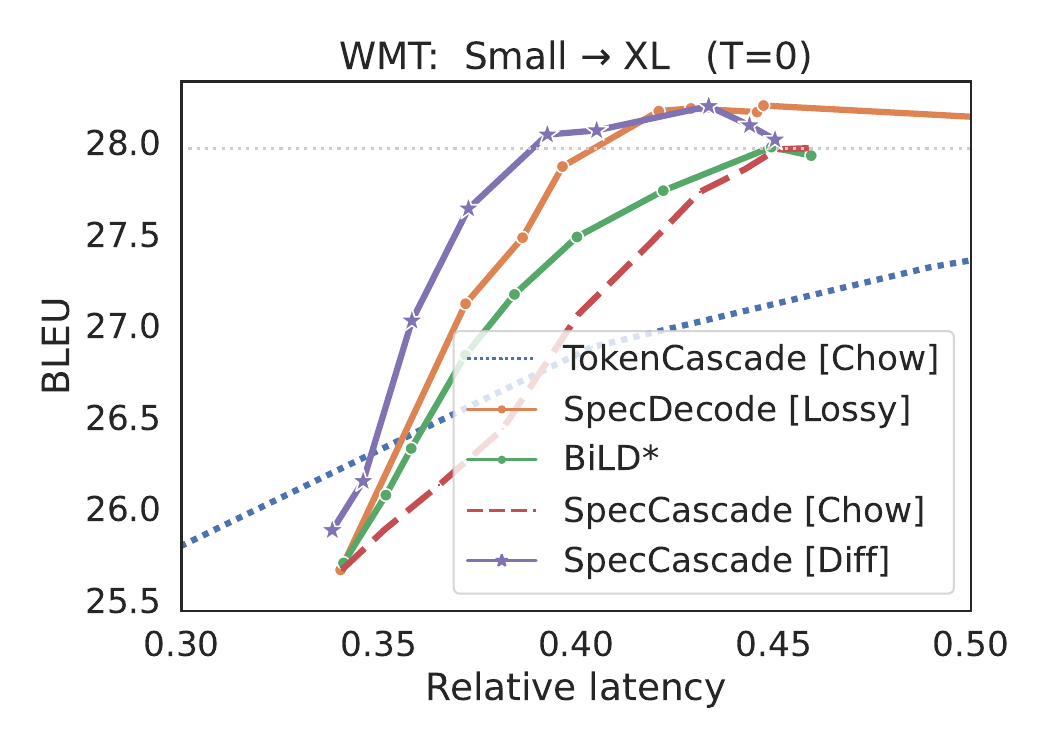}
    \includegraphics[scale=0.33]{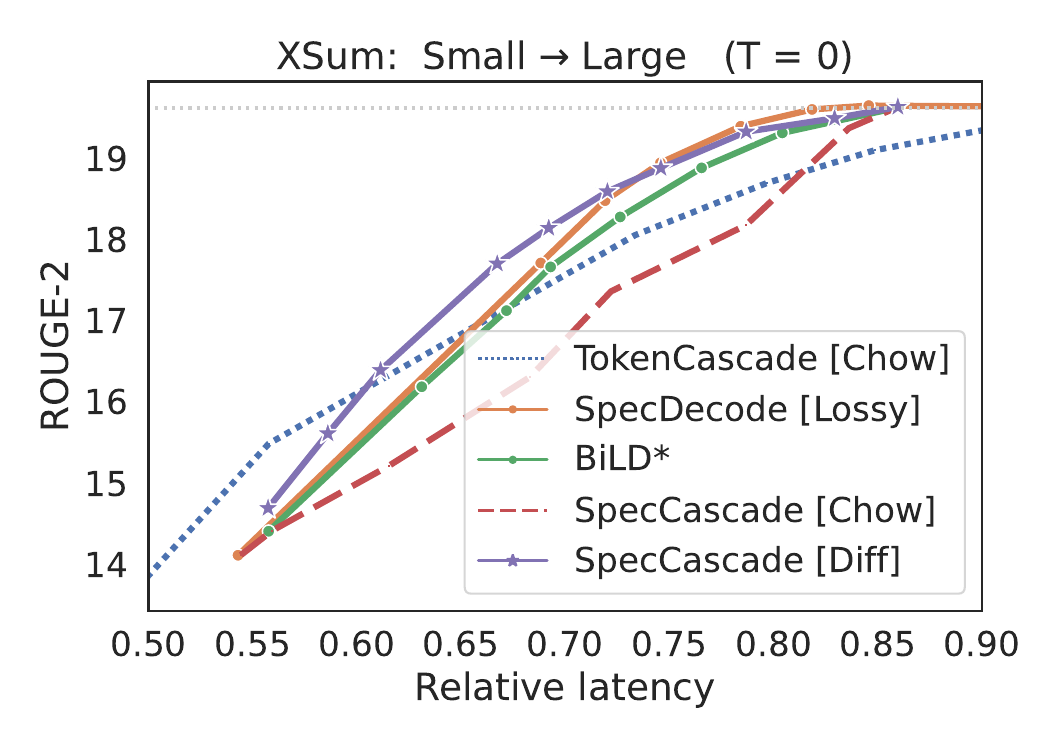}
    \includegraphics[scale=0.33]{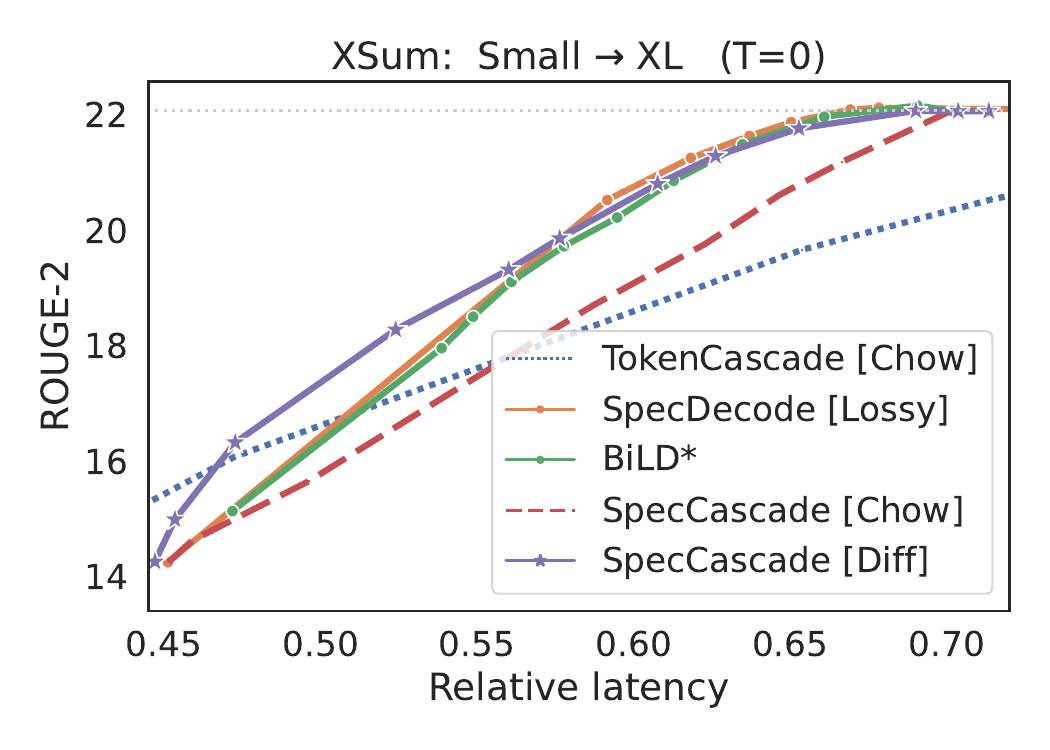}
    \includegraphics[scale=0.33]{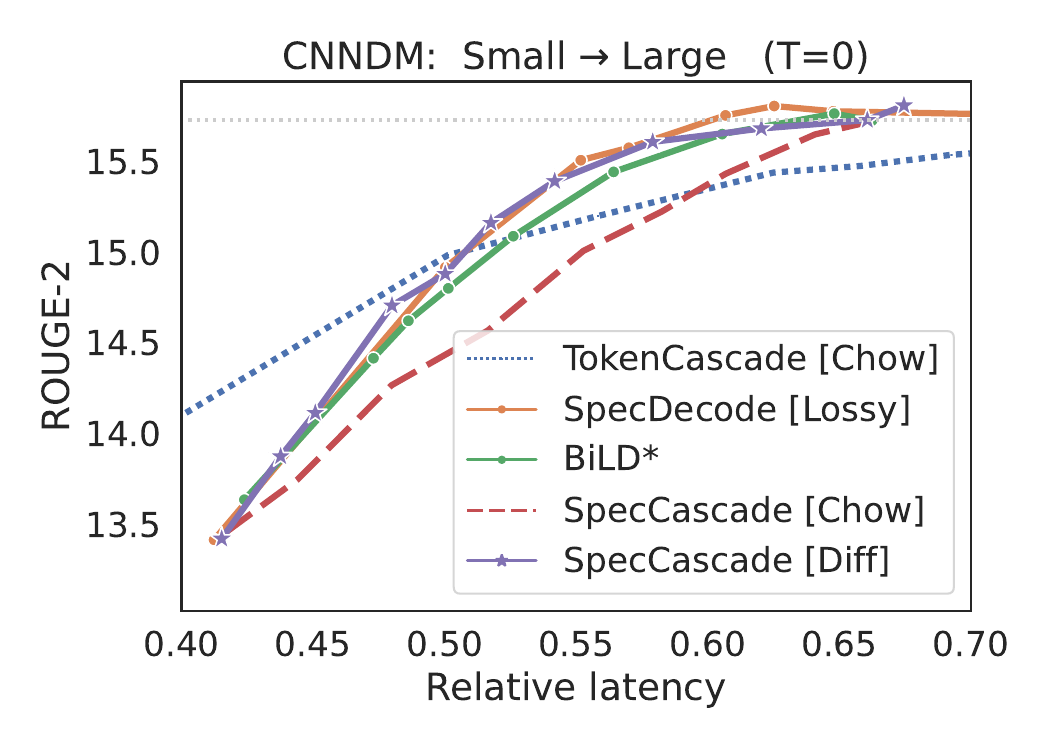}
    \vspace{-5pt}
    \caption{Plots of quality vs.\ latency for T5 models with \textbf{greedy decoding} with temperature $T=0$ and block size $\gamma=5$. Each method interleaves T5-small with T-large (or T5-XL). The $x$-axis tracks the latency \emph{relative} to that of  calling  the large model on all inputs.
    The horizontal dotted line denotes the quality of the large model. 
    }
    \label{fig:tradeoffs-additional-greedy}
\end{figure}

\subsection{Experimental setup and hyper-parameters}
\label{app:expt-setup}

We elaborate on our experimental setup and the hyper-parameters used.

\textbf{T5 datasets.} For the WMT English to German translation task \citep{bojar-EtAl:2014:W14-33}, we use a validation sample of size 3000 provided with the dataset. We set the maximum input length to 80 and the maximum output length to 80. For the Extreme Summarization (XSum) task \citep{narayan2018don}, we use a validation sample of size 11305, and set the  maximum input length to 1024 and the maximum output length to 64. For the CNN/Daily Mail summarization task \citep{hermann2015teaching}, we use a validation sample of size 13368, and set the  maximum input length to 2048 and the maximum output length to 128. Following \citep{zhou2024distillspec}, we use ROUGE-2 as the  evaluation metric for the summarization tasks.

{We note that \citet{kim2023speculative} report ROUGE-L metrics for CNN/DM, which generally tend to evaluate to higher values than ROUGE-2. Furthermore, most of their experimental results are with greedy decoding ($T=0$), and hence, the ROUGE-L evaluation metrics they report in their paper tend to be higher for the same T5 models when compared to our numbers for ROUGE-2 with temperature sampling.}

\textbf{Gemma datasets.} In addition to the WMT EN$\rightarrow$DE translation and the CNN/DM summarization datasets, we use the GSM8K \citep{cobbe2021training}  math reasoning dataset, the MBPP \citep{austin2021program}  Python programming dataset, and four question-answering datasets: Natural Questions \citep{Kwiatkowski:2019}, TriviaQA \citep{Joshi:2017}, WebQuestions \citep{berant2013semantic} and the Stanford Question-Answering Dataset (SQuAD) 2.0 \citep{Rajpurkar:2016}.
In each case, we sample 1000 prompts for evaluation. We employ few-shot inference, and set the maximum output length to 80 for WMT, to 128 for CNN/DM, to 320 for GSM8K and MBPP, and to 5 for all the question-answering datasets.

\textbf{Models.} We construct cascades from T5 v1.1 family of encoder-decoder models \citep{Raffel:2020}, of different sizes T5-small (77M), T5-base (250M), T5-large (800M) and T5-XL (3B).\footnote{The pre-trained checkpoints we use are available  \href{https://console.cloud.google.com/
storage/browser/t5-data/pretrained_models}{here}.} We follow the protocol in \citep{zhou2024distillspec}: we initialize with the public checkpoints, pre-train them further for 100K steps, and supervise \textbf{finetune} pre-trained models on the three respective tasks. We finetune them for a maximum of 250K steps on WMT, a maximum of 100K steps on XSum and a maximum of 200K steps on CNNDM.

We construct the Gemma cascades from \textit{instruction-tuned} decoder-only v2 models. For MBPP alone we additionally experiment with \textit{pre-trained} models.  We use a 2B drafter, and either a 9B verifier or a 27B verifier \citep{team2024gemma}.

\textbf{Evaluation.} For each dataset, we evaluate the quality metrics on the entire validation set. For the run-time analysis, we adopt the protocol followed in \cite{Leviathan:2023, zhou2024distillspec}. We randomly sample 500 examples from the validation set, and calculate the wall-clock time taken for decoding with a batch size of 1. We repeat this for three trials and report the average running time. All methods are run on the same TPUv4 device. The drafter and verifier models are run without model parallelism. 

\textbf{Hyper-parameters.} We set the block-size $\gamma$ to 5 for all methods that use speculative execution. For the token-level cascades, we allow the small model to predict for a maximum of 10 tokens (similar to \citep{kim2023speculative}), before invoking the large model. This was needed, as otherwise, the small model would predict a long sequence, and when it eventually defers to the large model, the large model is bottle-necked by the pre-filling of the long prefix accumulated by the small model. We vary the lenience parameter $\alpha$ to vary the latency and plot quality as a function of latency. We vary this parameter in the range 0 to 1 for all methods where the thresholding is on a probability metric; the  exceptions to this are the BiLD variants, for which, we use a longer range, as detailed below.

\textbf{BiLD baseline.} For the BiLD method, we adopt the same discrepancy metric $D$ as \citep{kim2023speculative} for greedy decoding:
\vspace{-15pt}
$$
D(q, p) = -\log\left(p\left(\argmax_v q(v)\right)\right),
$$
and pick the value of the threshold $\alpha$ on this metric from the range $[0, 10]$. 
For temperature sampling with a non-zero temperature, we use the following natural analogue to the above $D$:
\[
D(q, p) = -\E_{v \sim q}\left[ \log(p(v)) \right] = -\sum_{v \in \cV} q(v) \cdot \log(p(v)).
\]
In \S\ref{app:expts-bild}, we present comparisons between different implementations of this method. 

\textbf{Lossy speculative decoding.} See \S\ref{app:expts-lossy} for details.

\subsection{Additional experimental plots}
\label{app:T5-greedy}
In Figures \ref{fig:tradeoffs-additional-temp-sample} and \ref{fig:tradeoffs-additional-greedy}, we provide additional plots of quality vs.\ latency for different inference strategies under temperature sampling ($T=1$) and greedy decoding respectively.

As noted in \S\ref{app:greedy}, with greedy decoding, the {\tt OPT} deferral rule  coincides with the {\tt Diff} deferral rule. When temperature $T \rightarrow 0$, $D_{\textup{\tv}}(\tilde{p}_t, \tilde{q}_t) = 1$ whenever $\argmax_v p_t(v) \ne \argmax_v q_t(v)$, and is zero otherwise. In this case, running Algorithm \ref{alg:spec-cascade} with  $\tilde{r}_{\rm\tt OPT}$ as the deferral rule (and $\tilde{q}_t$ as the drafter) is equivalent to running it with  $\hat{r}_{\rm\tt Diff}$ in \eqref{eq:seq-opt-def-01-plugin} as the deferral rule. In other words, for greedy decoding, the optimal deferral rules for a speculative cascade coincides with that for a sequential cascade.

Note that under greedy decoding, all methods yield better quality metrics  compared to their performance under temperature sampling.  

\subsection{Comparing speculative deferral rules under different temperatures}
\label{app:expts-temperature}
In Figure \ref{fig:tradeoffs-T5-vary-temp}, we present latency-quality trade-off plots for T5 cascades under temperature sampling with different temperatures. We compare lossy speculative decoding with two speculative cascade deferral rules: {\tt OPT} rule in \eqref{eq:seq-opt-def-01-plugin} and the {\tt Token}-specific rule in \eqref{eq:sample-dep-01-plugin-v3}. We find that the gap between {\tt OPT} and the {\tt Token}-specific rule diminishes as the temperature decreases. 

The reason the {\tt Token}-specific rule fares better than  {\tt OPT} is because the latter compute their deferral decisions based on which of $q_t(\cdot)$ and $p_t(\cdot)$ is more \textit{peaked}; this can be a disadvantage when the sampled token is not be close the distribution mode, which is likely to happen with higher temperatures. With lower temperatures, however, the sampled token is likely to be close the distribution mode, and as a result,  the advantage that the {\tt Token}-specific rule has over {\tt OPT} diminishes.

\begin{figure}[t]
    \centering
    \includegraphics[scale=0.27]{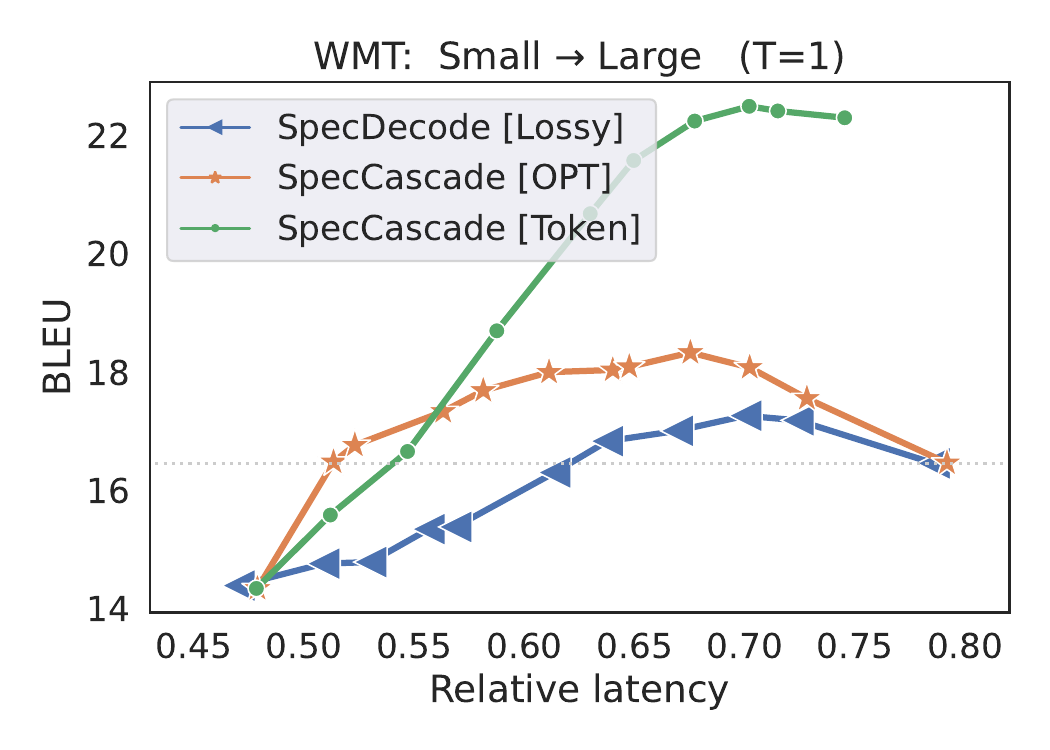}
    \includegraphics[scale=0.27]{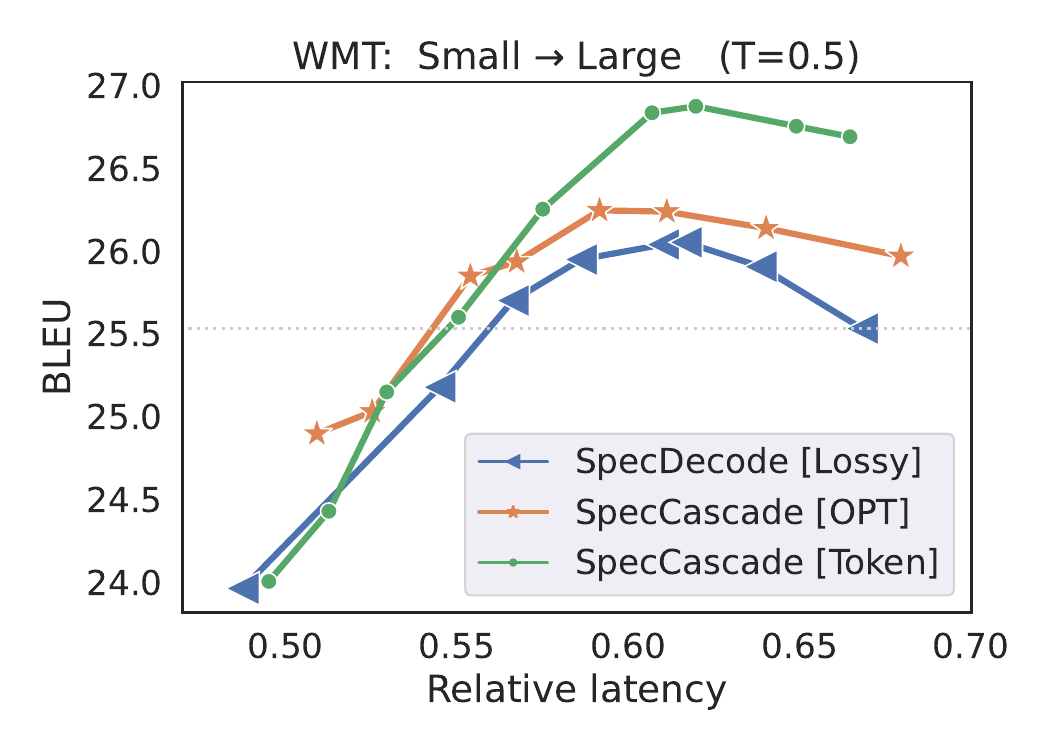}
    \includegraphics[scale=0.27]{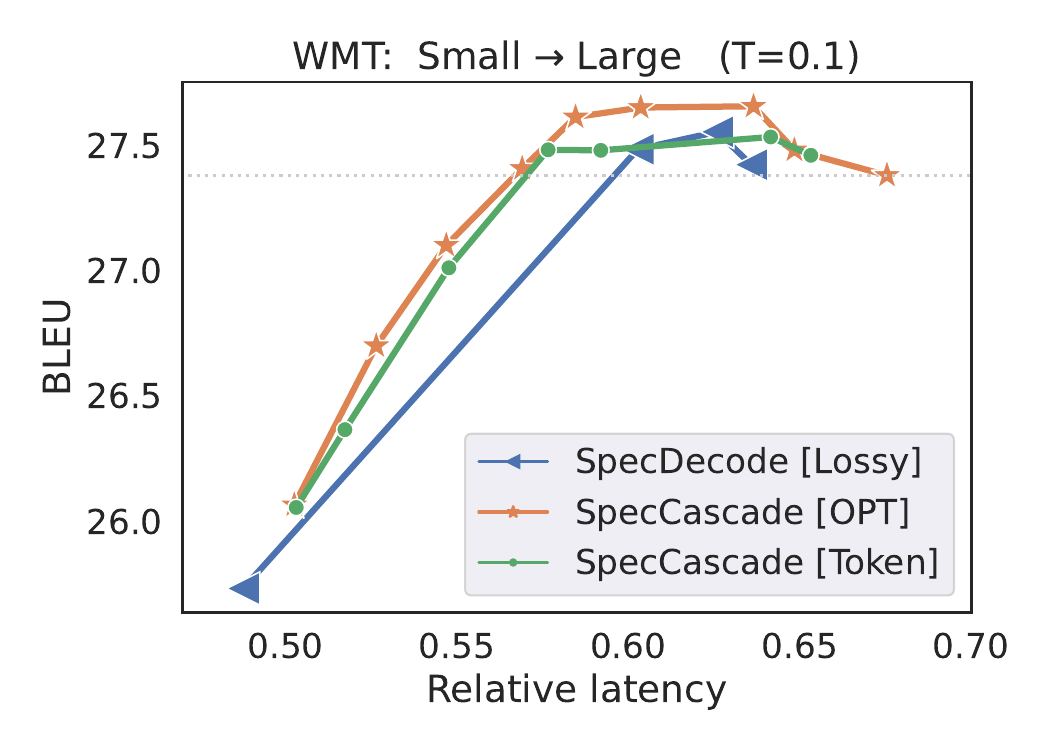}
    \\
    \includegraphics[scale=0.27]{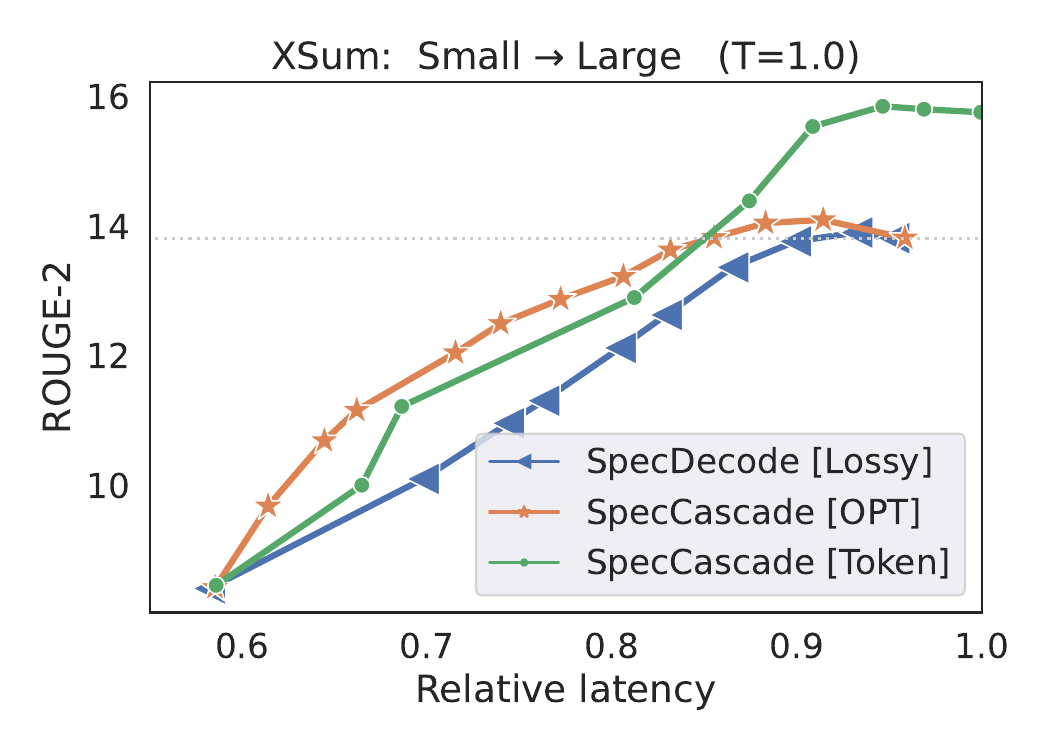}
    \includegraphics[scale=0.27]{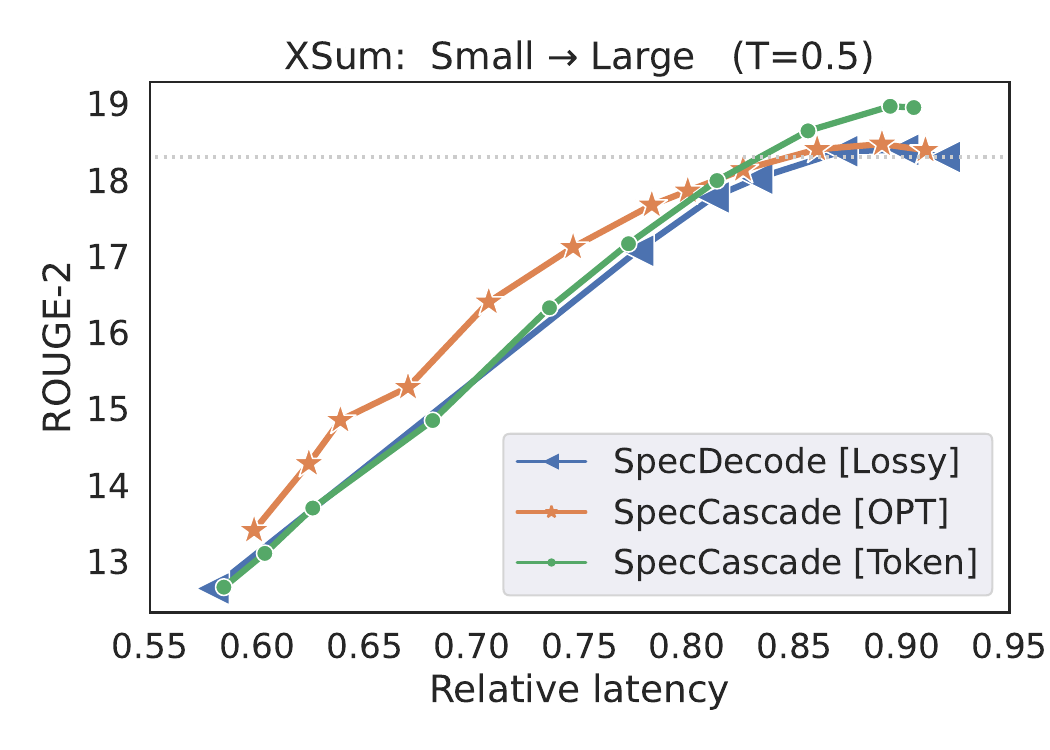}
    \includegraphics[scale=0.27]{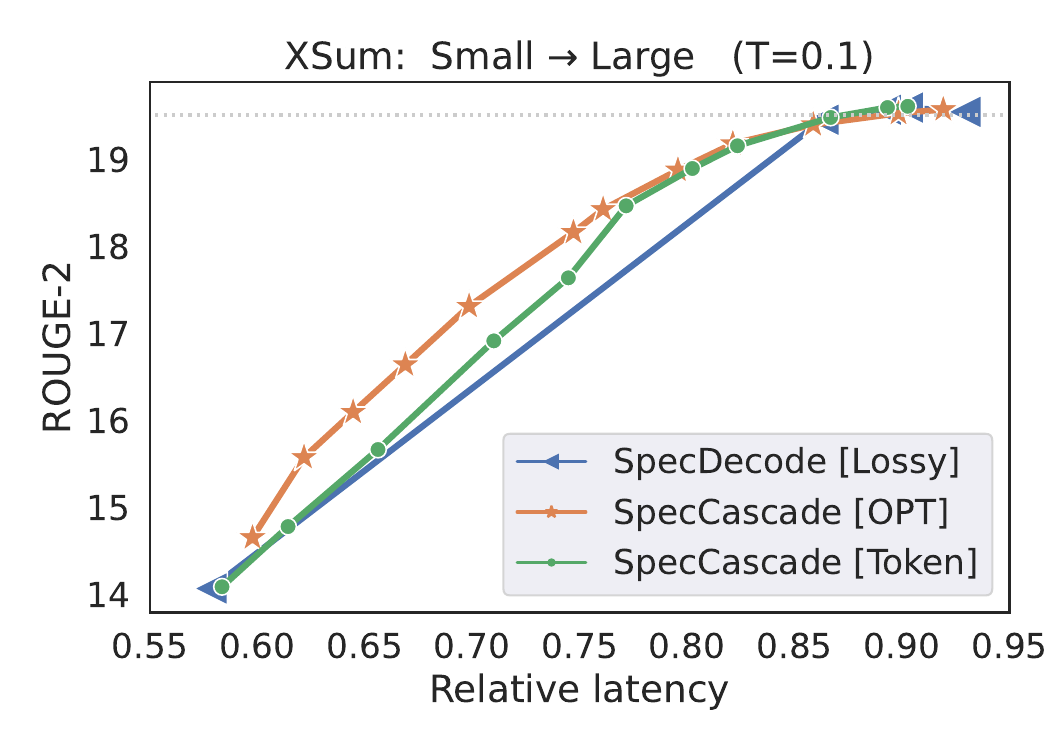}
    \\
    \includegraphics[scale=0.27]{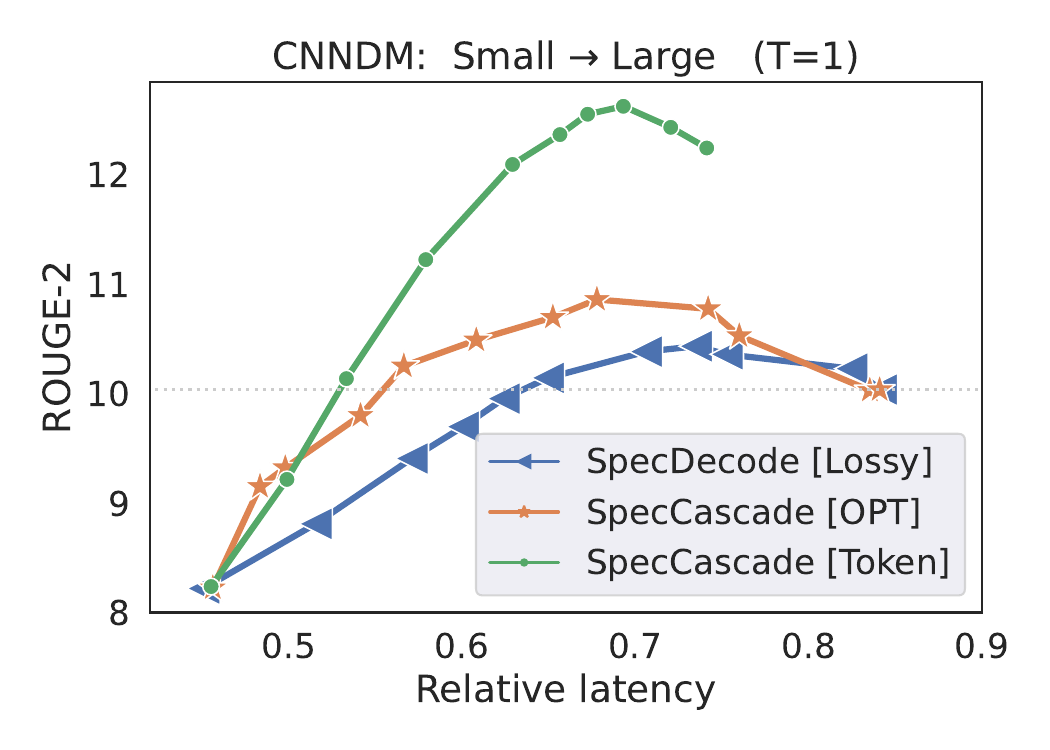}
    \includegraphics[scale=0.27]{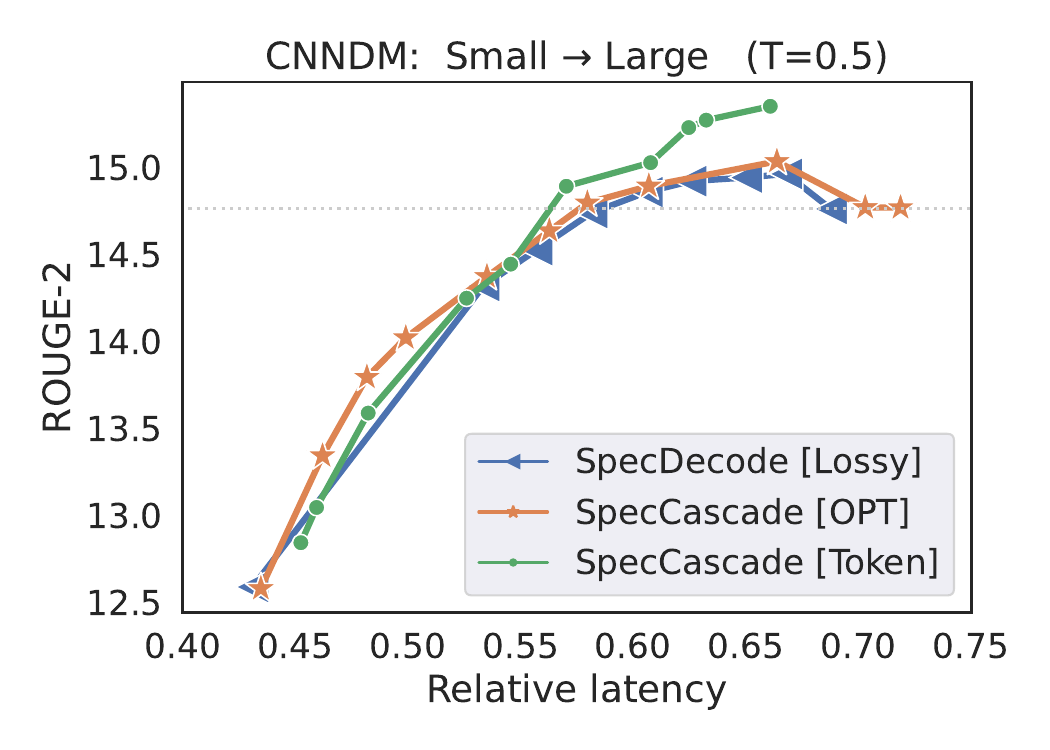}
    \includegraphics[scale=0.27]{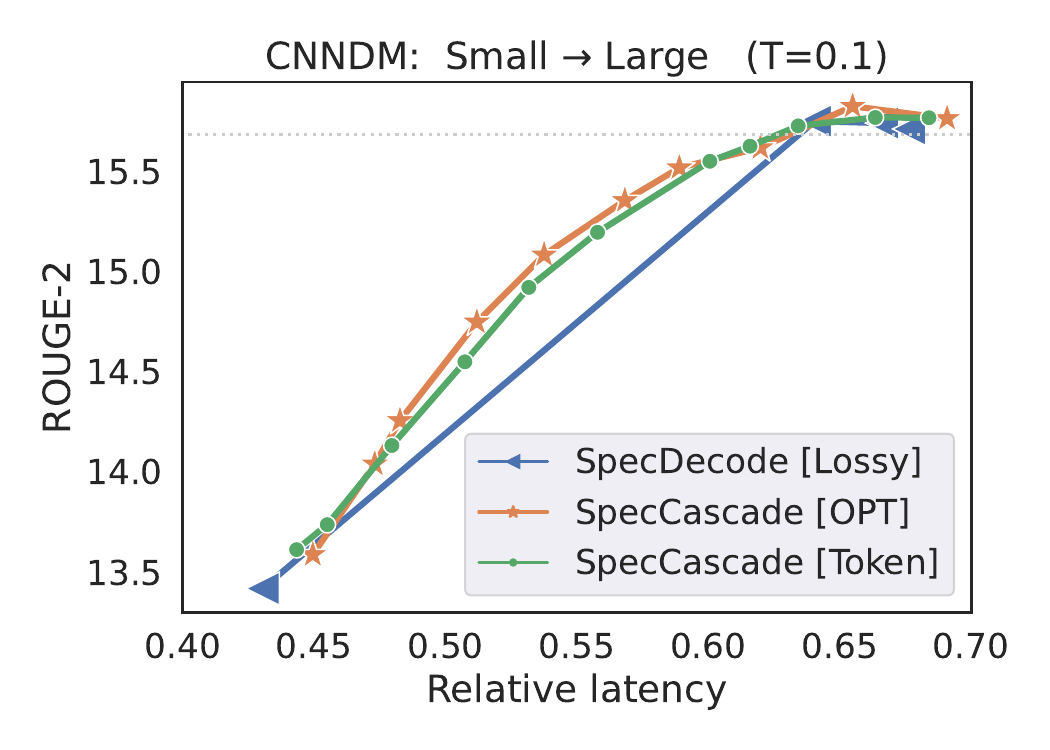}
    \vspace{-5pt}
    \caption{Plots of quality vs.\ latency for T5 models with \textbf{varying temperatures}. Each method interleaves T5-small with T-large. The $x$-axis tracks the latency \emph{relative} to that of  calling  the large model on all inputs.
    The horizontal dotted line denotes the quality of the large model. 
    }
    \label{fig:tradeoffs-T5-vary-temp}
\end{figure}

\todoakm{add plots. Hari: done!}

\subsection{Comparing speculative deferral rules under different block sizes $\gamma$}
\label{app:expts-gamma}
In Figure \ref{fig:tradeoffs-T5-vary-gamma}, we present latency-quality trade-off plots for T5 cascades under different block sizes $\gamma$. In each case, we find that the proposed speculative cascading techniques  outperform lossy speculative decoding across different latency values. Furthermore, higher values of $\gamma$ are seen to yield a wider range of trade-offs, with lower quality operating points shifting to the left, and better quality operating points shifting to the right.   For example, with XSum, {\tt SpecDecode [Lossy]} with $\gamma=3$ matches the small model's quality at ~0.64 relative latency, and matches the large model's quality at ~0.85 relative latency; with $\gamma=7$, it matches the small model's quality at an even lower latency, but practically provides no speed-up when matching the larger model's quality. The reason a larger block size can hurt speed-up at the higher quality regime is because it can result in frequent rollbacks, thus defeating the purpose of using speculative execution.

\begin{figure}[t]
    \centering
    \hspace{-5pt}
    \includegraphics[scale=0.27]{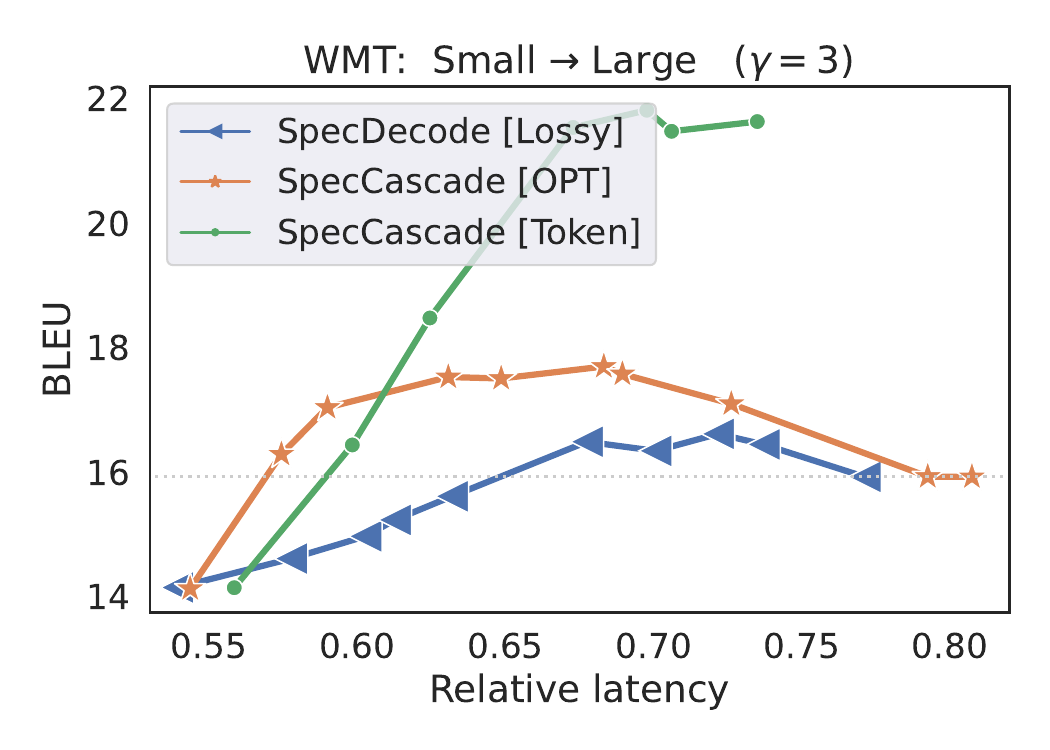}
    \hspace{-10pt}
    \includegraphics[scale=0.27]{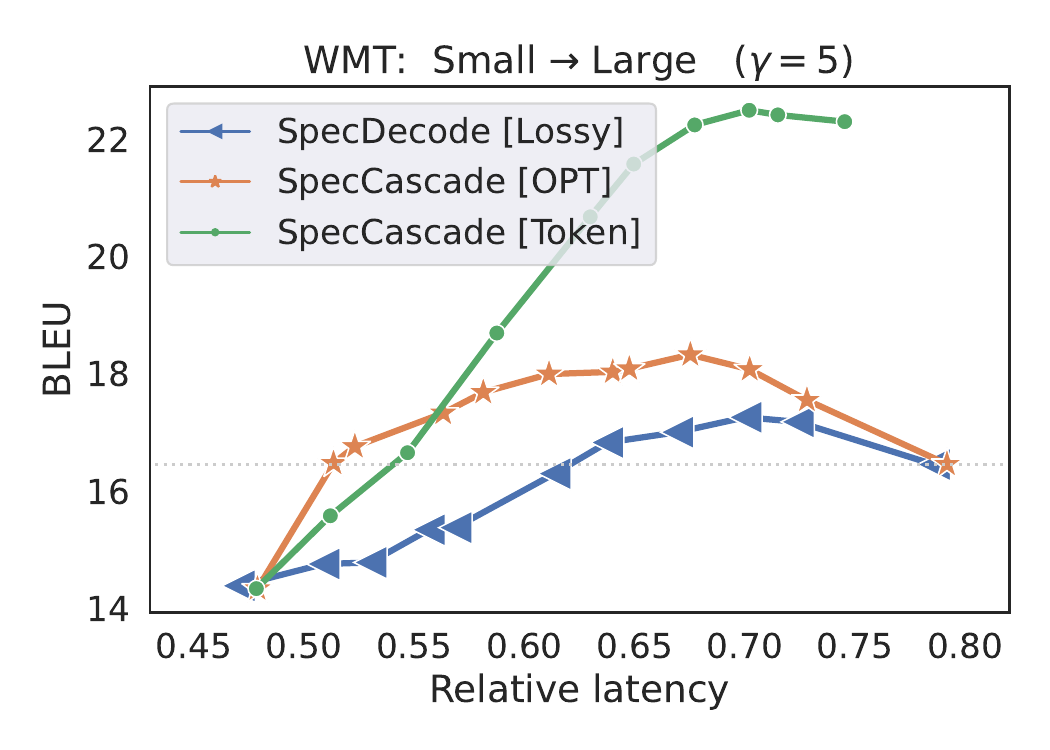}
    \hspace{-10pt}
    \includegraphics[scale=0.27]{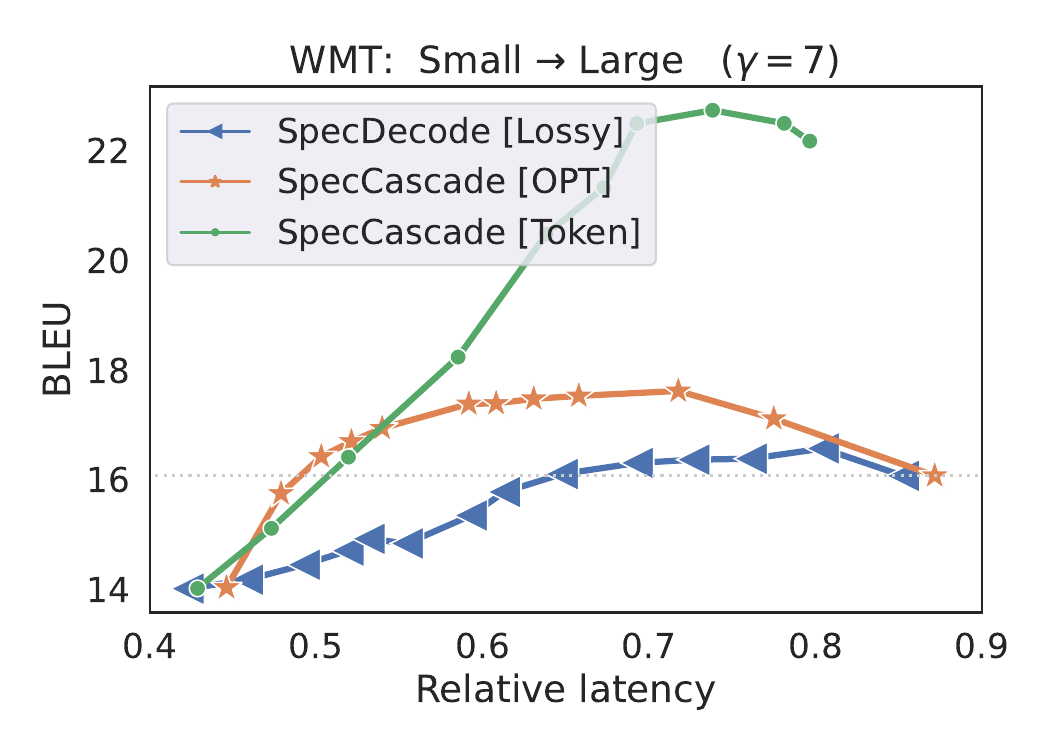}
    \\
    \hspace{-5pt}
    \includegraphics[scale=0.27]{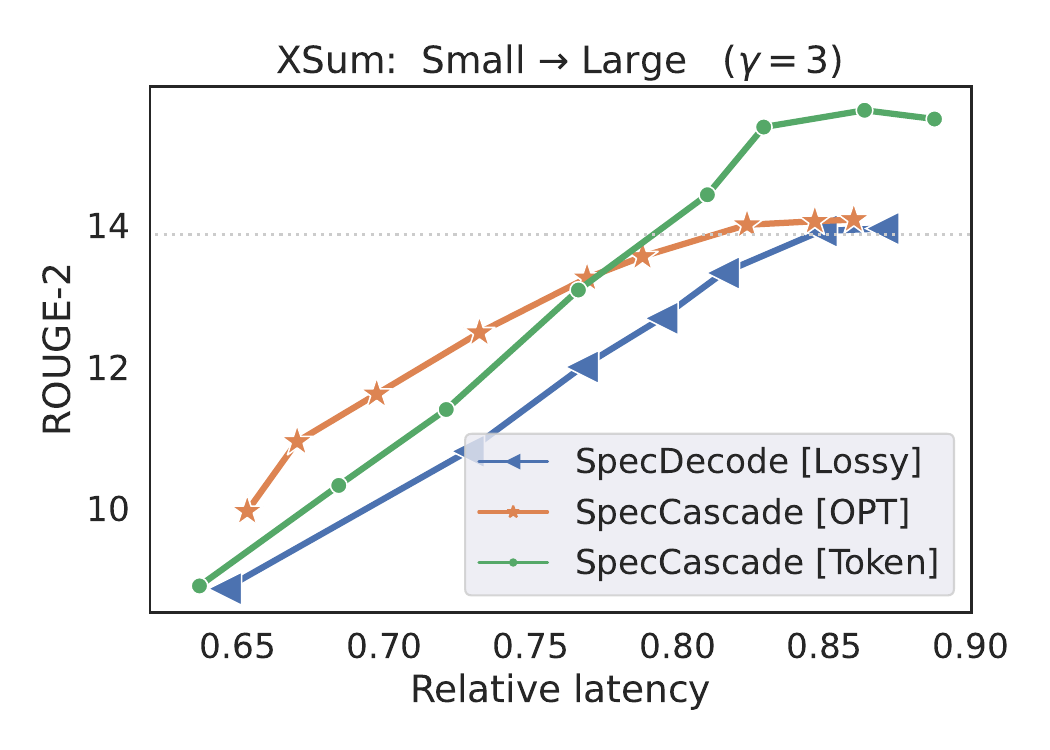}
    \hspace{-10pt}
    \includegraphics[scale=0.27]{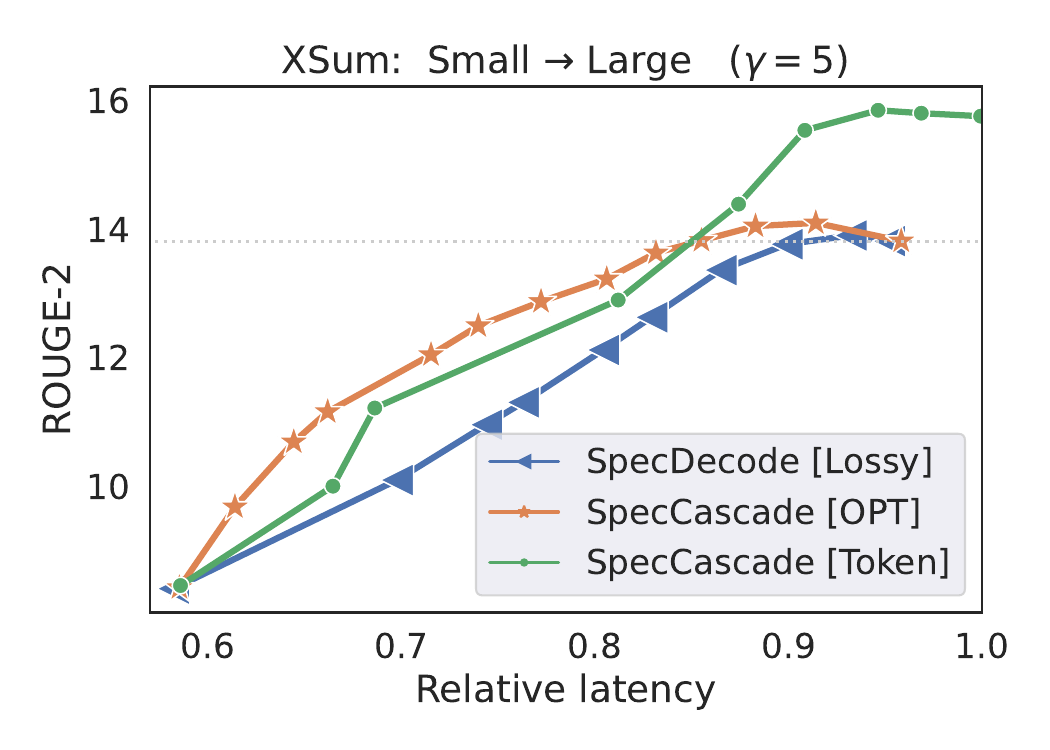}
    \hspace{-10pt}
    \includegraphics[scale=0.27]{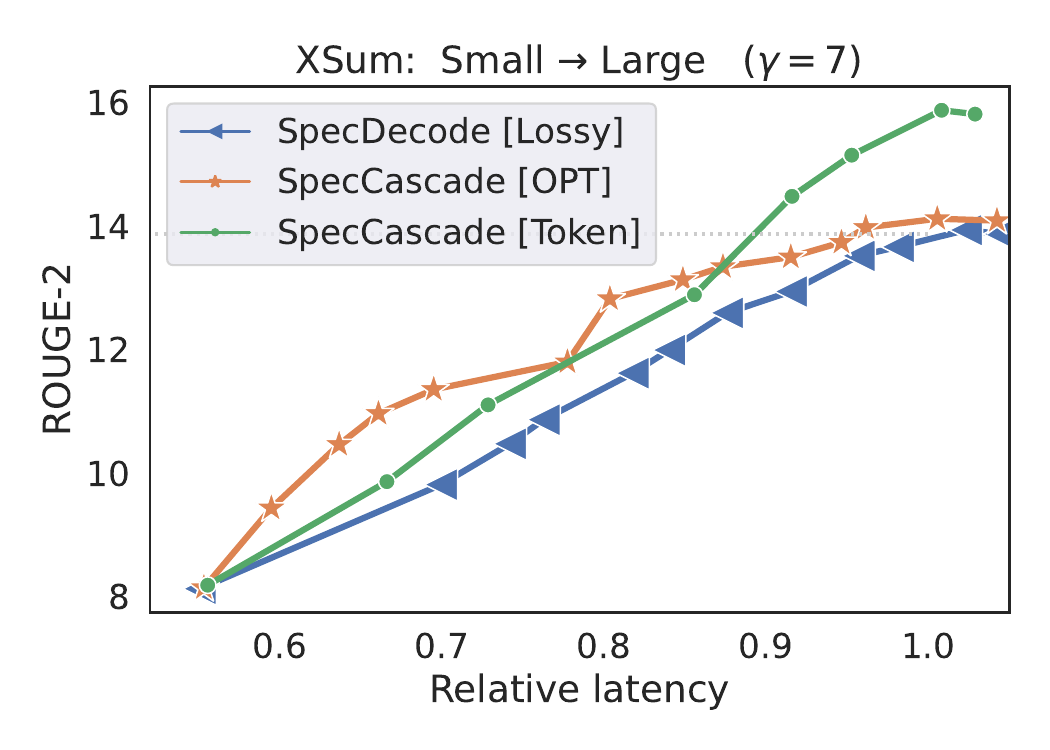}
    \\
    \hspace{-5pt}
    \includegraphics[scale=0.27]{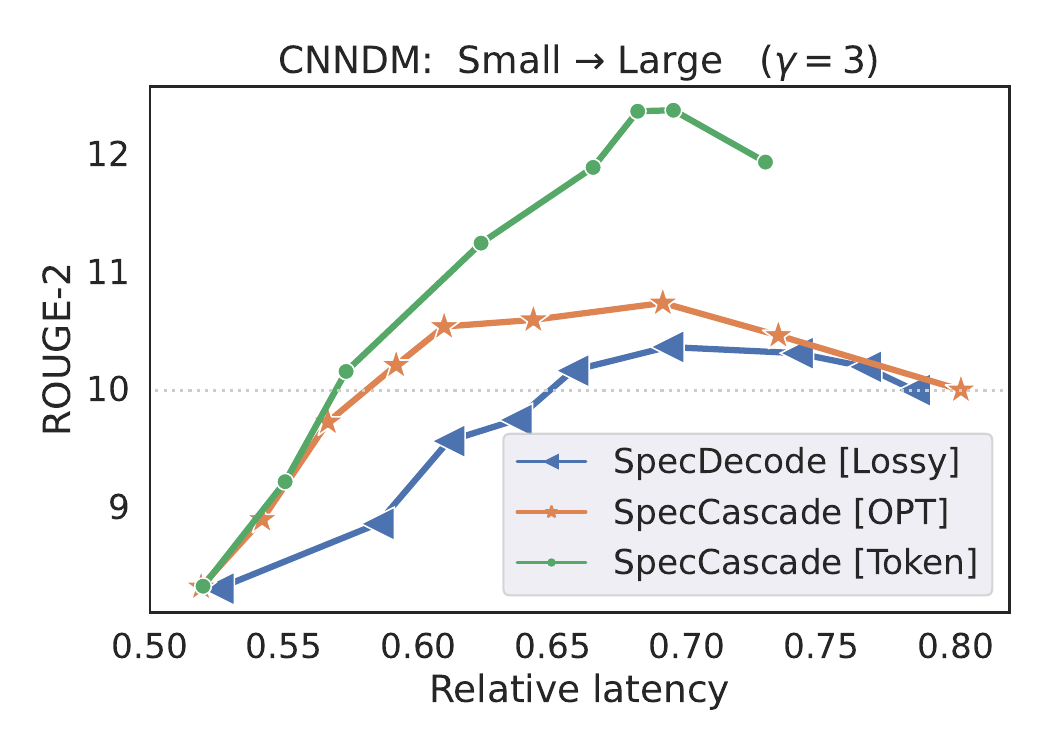}
    \hspace{-10pt}
    \includegraphics[scale=0.27]{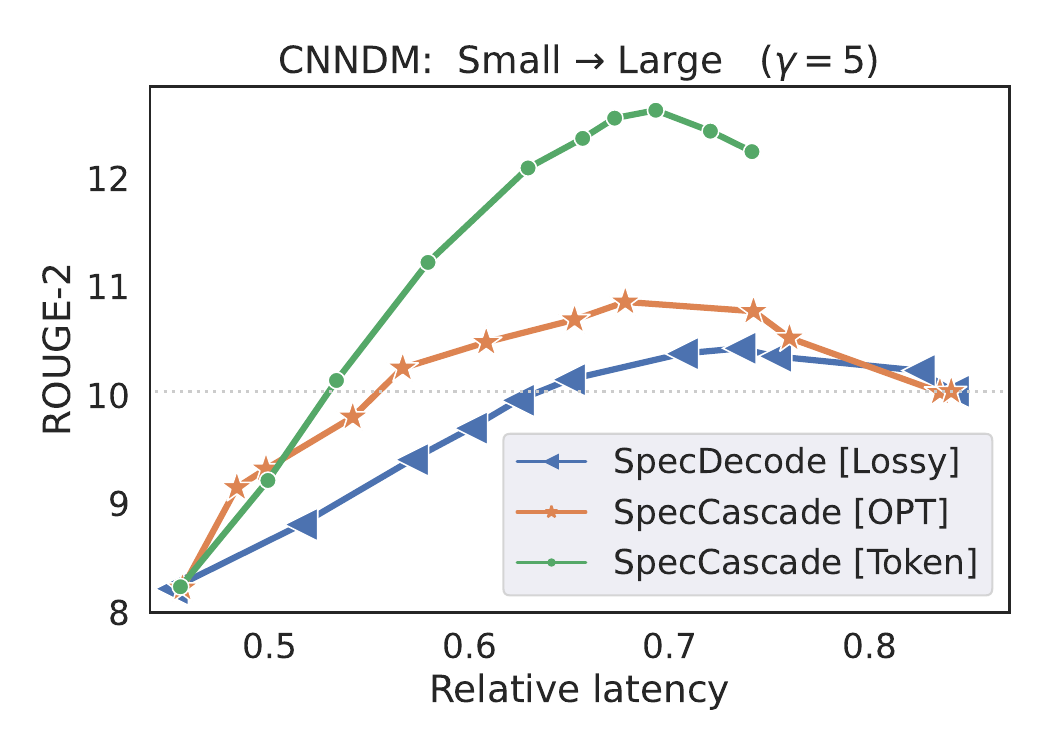}
    \hspace{-10pt}
    \includegraphics[scale=0.27]{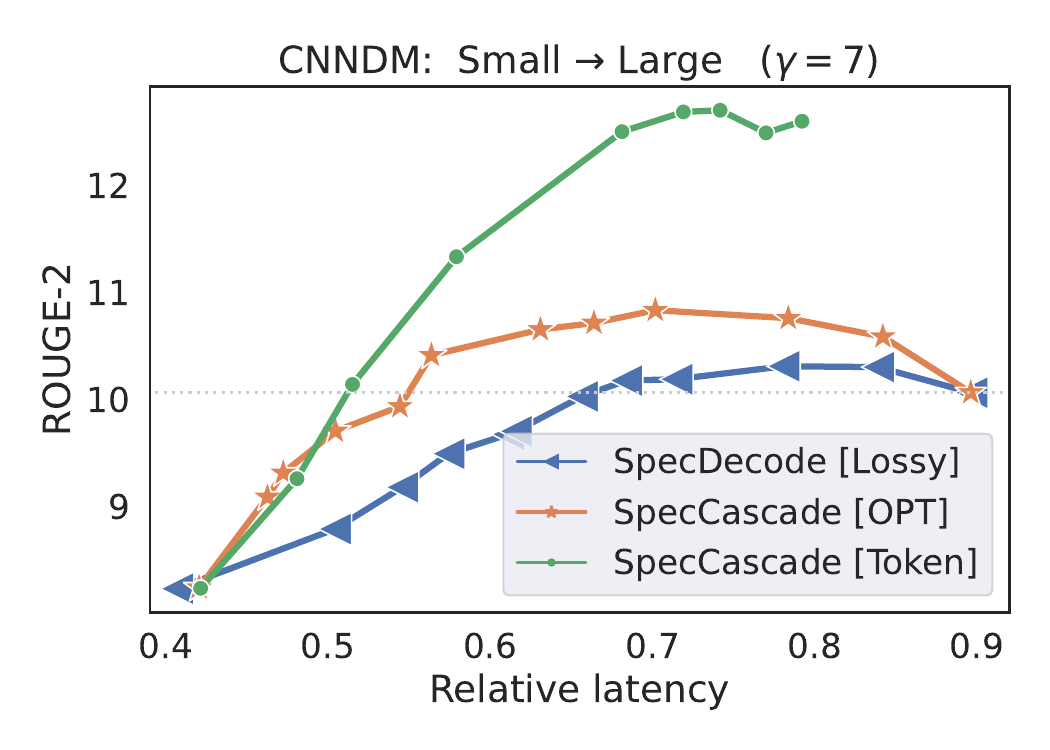}
    \vspace{-5pt}
    \caption{Plots of quality vs.\ latency for T5 models with \textbf{with varying block sizes $\gamma$}. Each method interleaves T5-small with T-large. The $x$-axis tracks the latency \emph{relative} to that of  calling  the large model on all inputs.
    The horizontal dotted line denotes the quality of the large model. 
    }
    \label{fig:tradeoffs-T5-vary-gamma}
\end{figure}

\begin{figure}[t]
    \centering
    \includegraphics[scale=0.5]{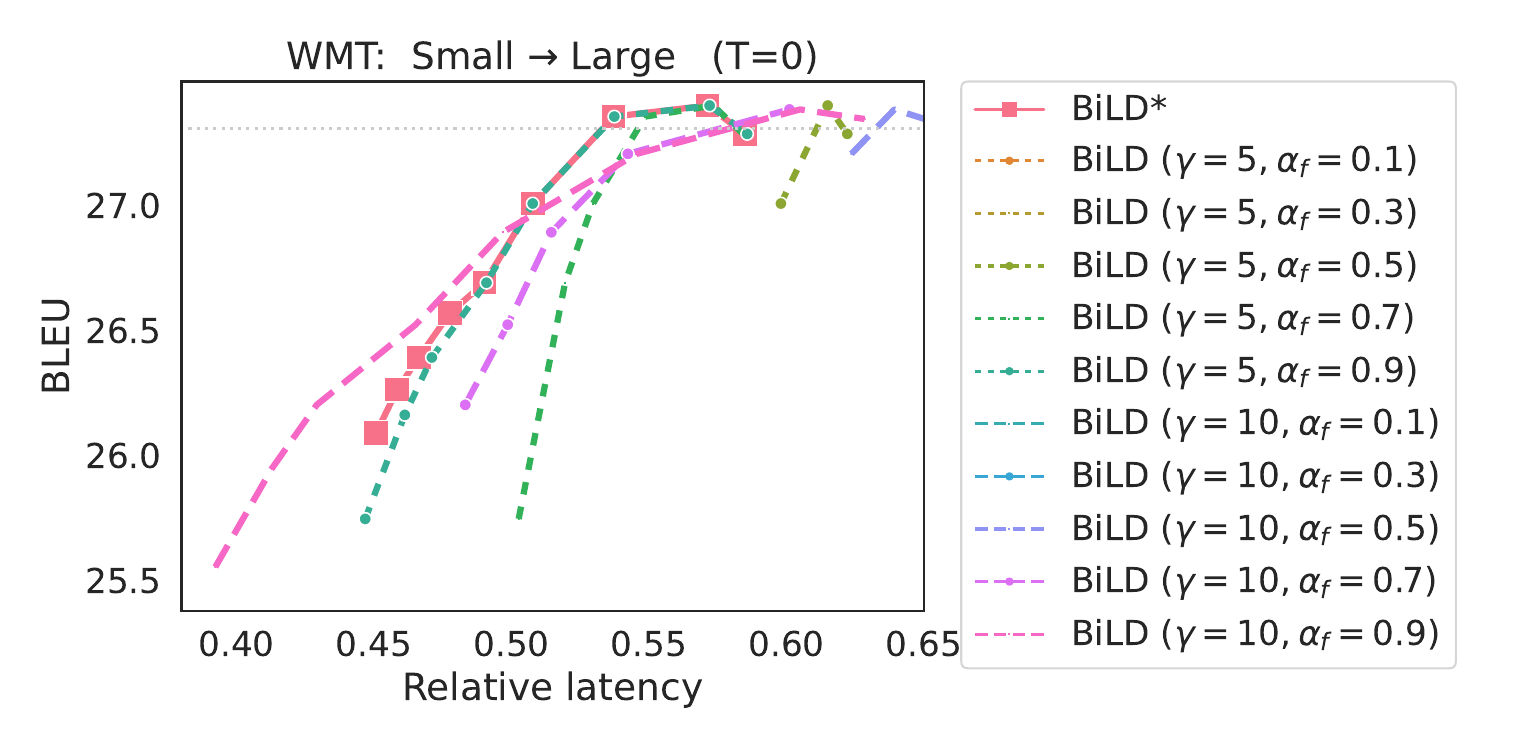}
    \includegraphics[scale=0.33]{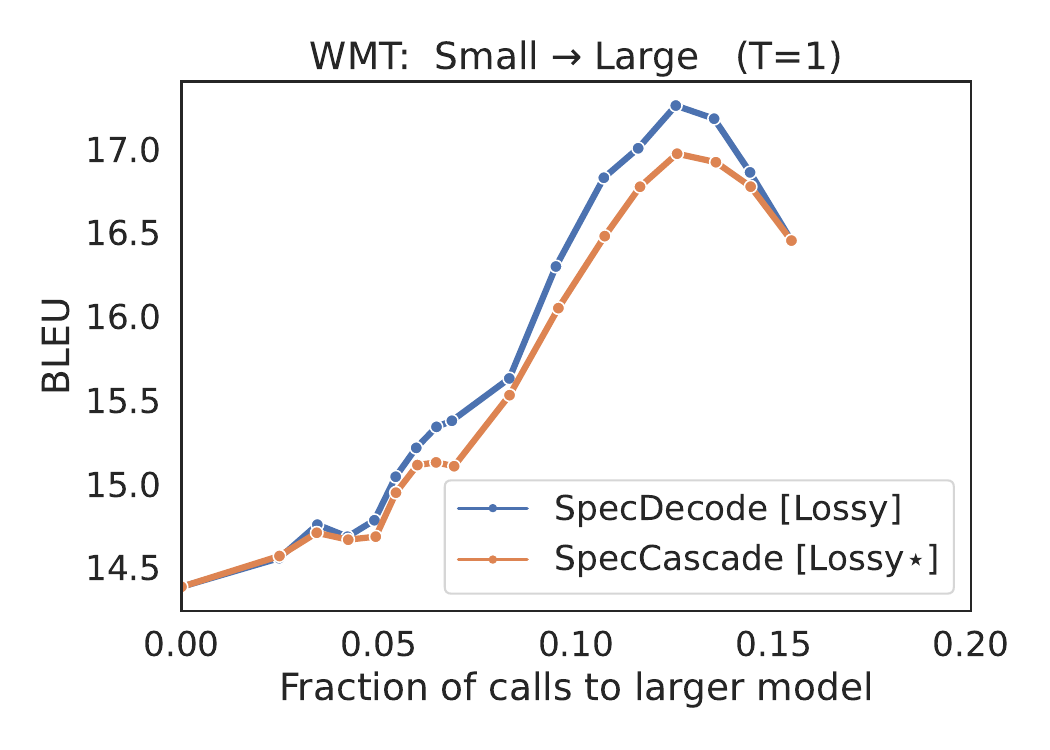}
    \includegraphics[scale=0.33]{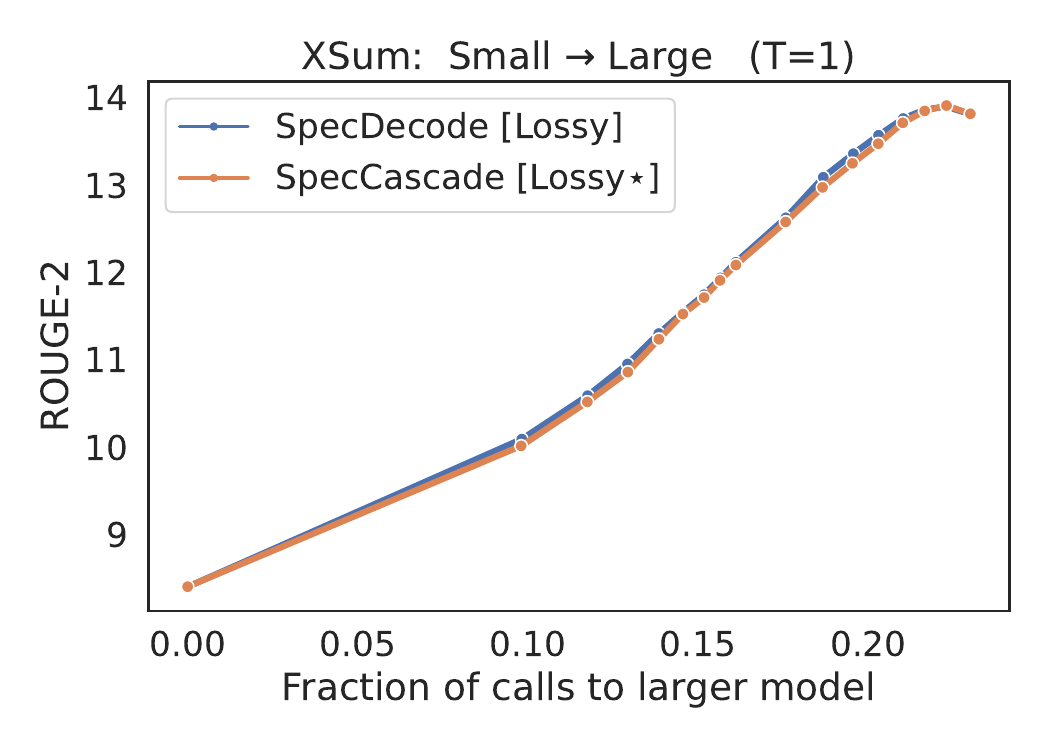}
    \vspace{-5pt}
    \caption{Top: Plots of quality vs.\ latency \textbf{comparing BiLD$^*$ with the original BiLD algorithm in \cite{kim2023speculative}} with varying maximum draft window size $\gamma$ and fallback confidence  threshold $\alpha_f$. Bottom: Comparison of lossy speculative decoding  with $\beta=1$ [{\tt Lossy}] and $\beta$ tuned using the procedure in \citep{Tran-Thien_2023} [{\tt Lossy}$^\star$].
    }
    \label{fig:tradeoffs-bild-lossy}
\end{figure}

\subsection{Big Little Decoder (BiLD) variants}
\label{app:expts-bild}
In \S\ref{sec:expts}, we compared against a version of the Big Little Decoder method \citep{kim2023speculative} that applied Algorithm \ref{alg:gen-speed} to the target distribution $\mathbb{T}_{\textup{BiLD}}$ the authors seek to mimic (\S\ref{sec:related}). We now show that this version performs similarly to the original BiLD algorithm in \citep{kim2023speculative}.

A key difference to the original algorithm in \citep{kim2023speculative} is the use of the fallback phase, where the drafter is run until its maximum predicted probability $\max_v q(v) < 1 - \alpha_f$, for a threshold $\alpha_f \in [0,1]$ (or until a maximum block size of 10 is reached), and the use of a deterministic rollback policy where the verifier rejects a draft token whenever $D(q, p) > \alpha$.
In our implementation, we adopt the speculative sampling algorithm from \citep{Leviathan:2023}: we do not have the fallback policy and replace the determinisic rollback policy with the rejection sampling in Algorithm \ref{alg:gen-speed}.  

Figure \ref{fig:tradeoffs-bild-lossy} (top) compares the original version of BiLD with the version we use in \S\ref{sec:expts}. We interleave between a T5-small and T5-large model on WMT, using greedy decoding ($T=0$) for inference. As prescribed by the authors \citep{kim2023speculative}, we use the following discrepancy metric for greedy decoding:
\[
D(q, p) = \log p\left(\argmax_v q(v)\right).
\]
We compare our implementation ({\tt BiLD}$^*$), where we set the block size 5 (same as our proposed speculative cascading approaches) with the original {\tt BiLD} for different choices of maximum block size $\gamma$ and different fallback thresholds $\alpha_f$. For both methods, we vary the threshold $\alpha$ on $D(q, p)$ to vary the latency and plot the resulting BLEU score. 

A  higher fallback threshold $\alpha_f$ results in larger draft generation windows; this gives an advantage in the low latency regime, where most of the draft tokens are accepted. As a result, {\tt BiLD} [$\gamma=10, \alpha=0.9$] yields the lowest latencies, but also yields lower quality. A low fallback threshold results in very small draft generation windows, and consequently, in higher latencies. This is why {\tt BiLD} [$\gamma=5, \alpha=0.1$] is the slowest but yields high quality metrics.

Our implementation {\tt BiLD}$^*$ is seen to perform comparable to the best parameter choices for the original {\tt BiLD} algorithm in Figure \ref{fig:tradeoffs-bild-lossy}.

\textbf{Note:} It is worth noting that while we view $\bT_{\rm BiLD}$ as the target distribution that algorithm in \citep{kim2023speculative} seeks to mimic, the presence of the fallback phase could mean that on some inputs a output response is generated without the verification (or rollback) phase being invoked. In such cases, the output will be a sample from the drafter even if it turns out that it contains tokens for which  $D(q_t,p_t) > \alpha$. 

\subsection{Lossy speculative decoding variants}
\label{app:expts-lossy}
In \S\ref{sec:expts}, we compared against the lossy speculative decoding \cite{Tran-Thien_2023, zhou2024distillspec} described in \S\ref{sec:prelims}, with the parameter $\beta$ set to 1. We now present results for this method with  $\beta$ tuned according to the procedure in \cite{Tran-Thien_2023}, and show that choosing $\beta=1$ fares at least as well as tuning $\beta$.

The goal in \cite{Tran-Thien_2023} is to choose $\alpha$ and $\beta$ so as to maximize the acceptance rate for the draft token, while ensuring that the KL divergence between the resulting target distribution and $p$ is within an allowable limit $R$. The authors prescribe specifying $R$, and for each prefix, tuning $\alpha$ and $\beta$ to solve the resulting constrained optimization problem. To be consistent with the rest of our experimental setup, we vary $\alpha$ to vary the draft acceptance rate (note that each choice of $\alpha$ corresponds to a particular KL divergence to $p$), and tune $\beta \geq 1 - \alpha$ to satisfy the following condition outlined in \cite{Tran-Thien_2023}:
\[
\sum_v \max\left\{0, q(v) - \frac{p(v)}{1 - \alpha}\right\} = \sum_v \max\left\{0, \frac{ p(v) }{\beta} - q(v)\right\}
\]

We pick $\beta$ using a grid-search over 1000 values between $\alpha$ and 10. Since this tuning procedure, in turn, can add to the method's latency,
for a fair comparison, we analyze quality as a function of the fraction of calls to the large model. In Figure \ref{fig:tradeoffs-bild-lossy} (bottom), we plot these trade-off curves for loss speculative decoding with $\beta = 1$ ({\tt Lossy}) and for speculative decoding with $\beta$ tuned using the above procedure ({\tt Lossy$^{\star}$}). We compare performances on WMT and XSum, and in each case, interleave a T5-small model with a T5-large model. 

In both cases, setting $\beta = 1$ provides trade-offs comparable to or better than using a tuned value of $\beta$. The reason using a tuned value of $\beta$ fares worse than setting $\beta = 1$ might be because we are measuring quality in terms of BLEU or ROUGE-2, which is different from the KL divergence to $p$ objective that the tuning procedure in \cite{Tran-Thien_2023} seeks to optimize.

\subsection{Token-specific speculative cascades}
\label{app:expts-token-specific}
In Figure \ref{fig:T5-small-large-token-specific}, we present latency-quality trade-off plots for cascades constructed from a T5 small and a T5 large model. We include in these comparisons, all three token-specific deferral rules in equations \ref{eq:sample-dep-01-plugin-v1}--\ref{eq:sample-dep-01-plugin-v3}. 
 In Figure \ref{fig:gemma-2B-27B-token-specific}, we present trade-off plots for cascades constructed from Gemma 2B and Gemma 27B models with all three token-specific rules, and 
in Figure \ref{fig:gemma-2B-9B-token-specific}, we include similar plots for cascades constructed from Gemma 2B and Gemma 9B models.
We note that the trends with the 2B $\rightarrow$ 9B are similar to those seen with the 2B $\rightarrow$ 27B cascades. 

With the T5 models, the results are mixed, with the V1 and V2 variants sometime surpassing the V3 variant (which is the variant we included in the main experiments results in \S\ref{sec:expts}) Interestingly, with the Gemma models, the V3 variant is seen to  outperform the others for most rejection rates, with the exception of the 2B$\rightarrow$27B cascade on SQuAD 2.0, where the V2 variant is better. 

The reason the V3 variant outperforms V1 and V2 on the Gemma models could be due to the fact that it uses the larger model's distribution $p_t(\cdot)$ to measure confidence for both the drafter and verifier (see LHS and RHS in \eqref{eq:sample-dep-01-plugin-v1}). We expect this to be particularly helpful when there is a larger gap in sizes between $q$ and $p$, and the larger model's distribution is better aligned with the ground-truth distribution compared to the smaller model. Furthermore, as  noted in \S\ref{sec:sample-dep}, the multiplicative form of the rule results in a target distribution that has an intuitive form: it seeks to mimic $q_t(\cdot)$ on the top-$\alpha$ ranked tokens by $p_t(\cdot)$ and uses a re-scaled version of $p_t(\cdot)$ for the other tokens. 

\begin{figure}[t]
    \centering
    \includegraphics[scale=0.39]{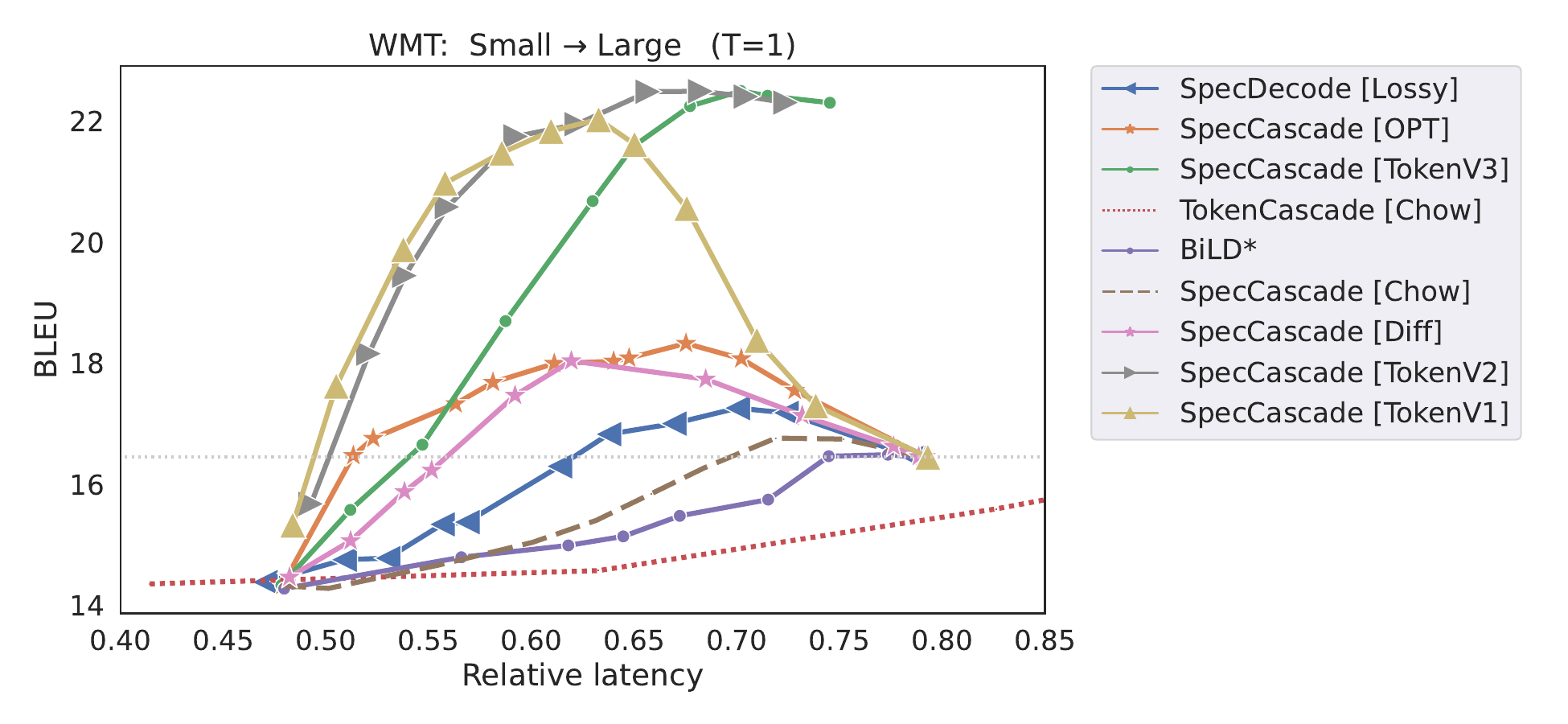}
    \includegraphics[scale=0.39]{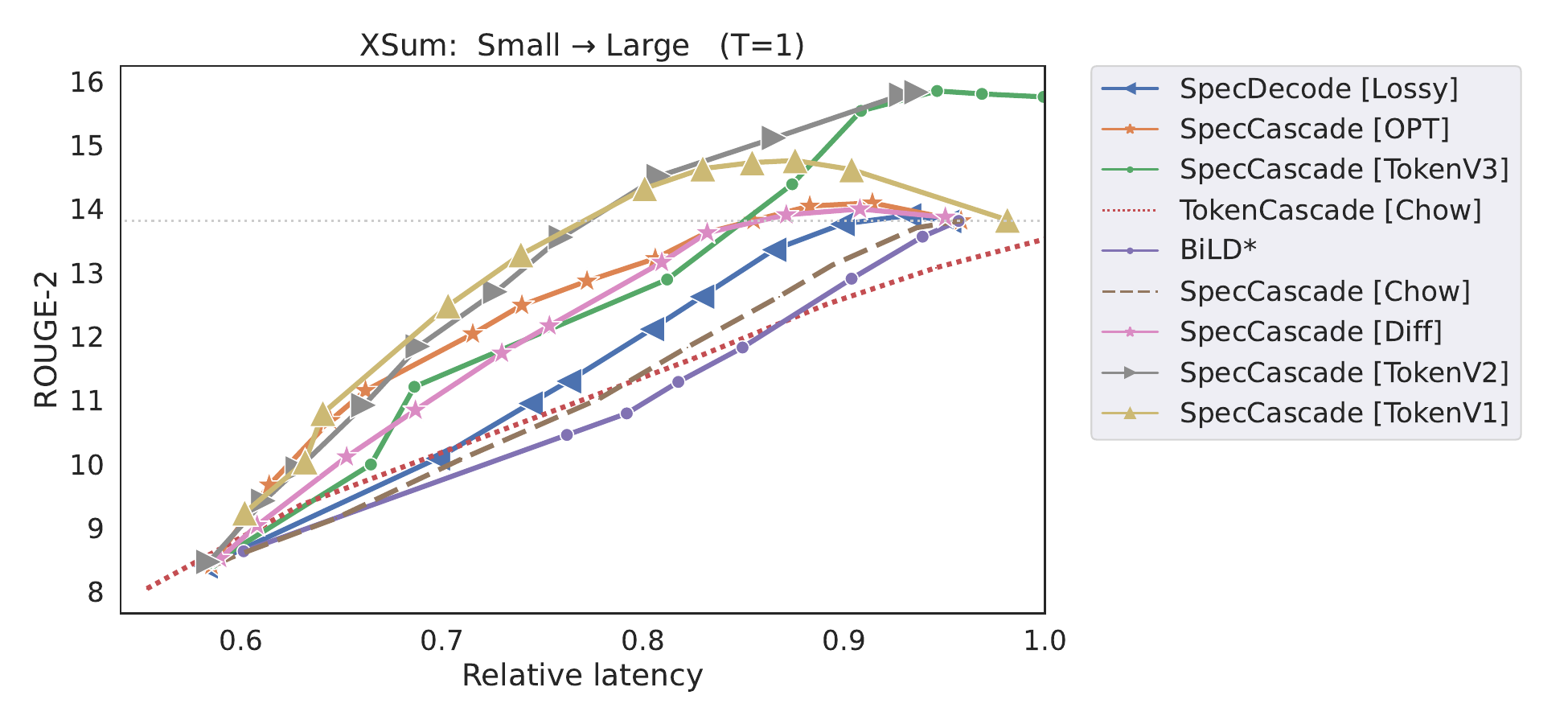}
    \includegraphics[scale=0.39]{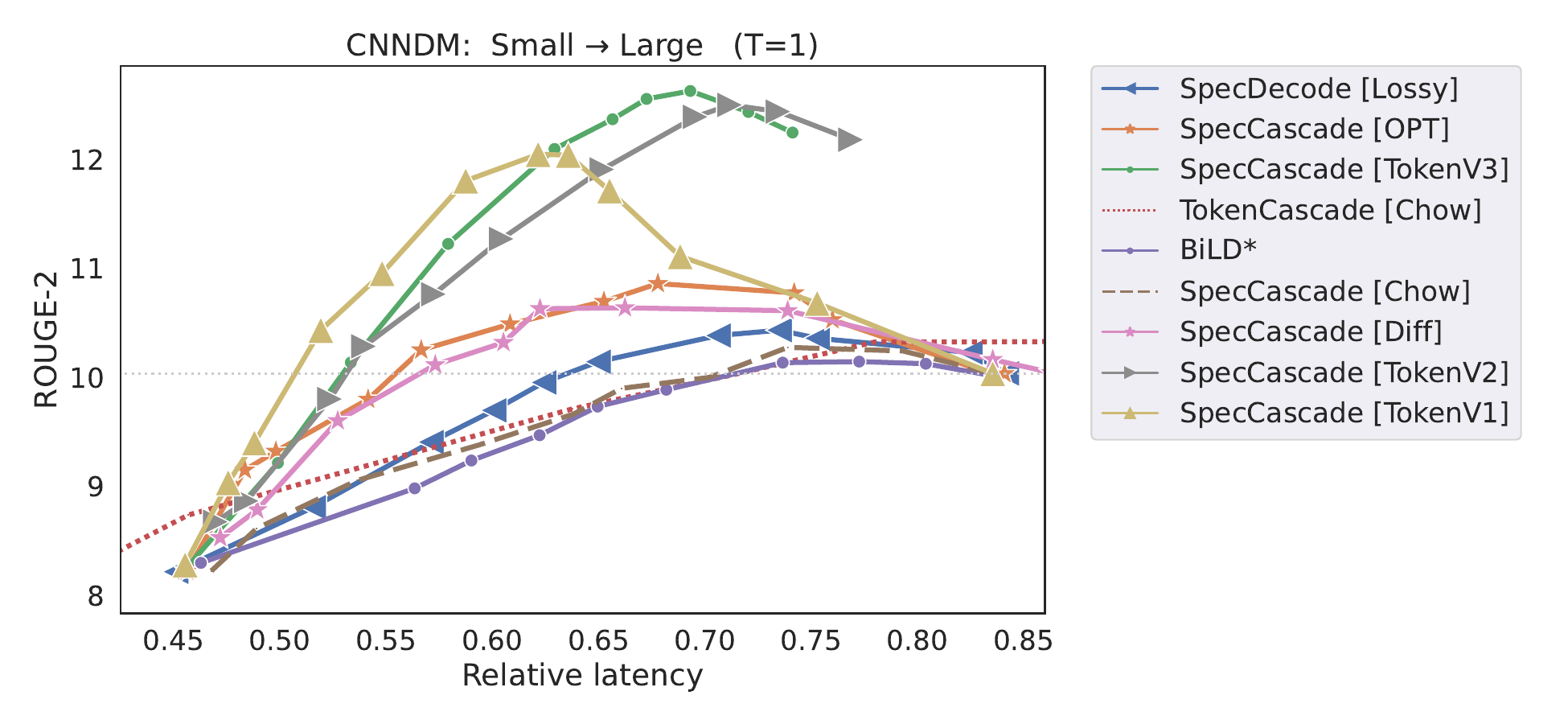}
\caption{Plots of quality vs.\ latency for  \textbf{T5 models with all three token-specific speculative cascade deferral rules} in equations \ref{eq:sample-dep-01-plugin-v1}--\ref{eq:sample-dep-01-plugin-v3}. Each method interleaves a T5 small and a T5 large model. The $x$-axis tracks the latency \emph{relative} to that of  calling  the large model on all inputs.
    The horizontal dotted line denotes the quality of the large model. }
    \label{fig:T5-small-large-token-specific}
\end{figure}

\begin{figure}[t]
    \centering
    \includegraphics[scale=0.25]{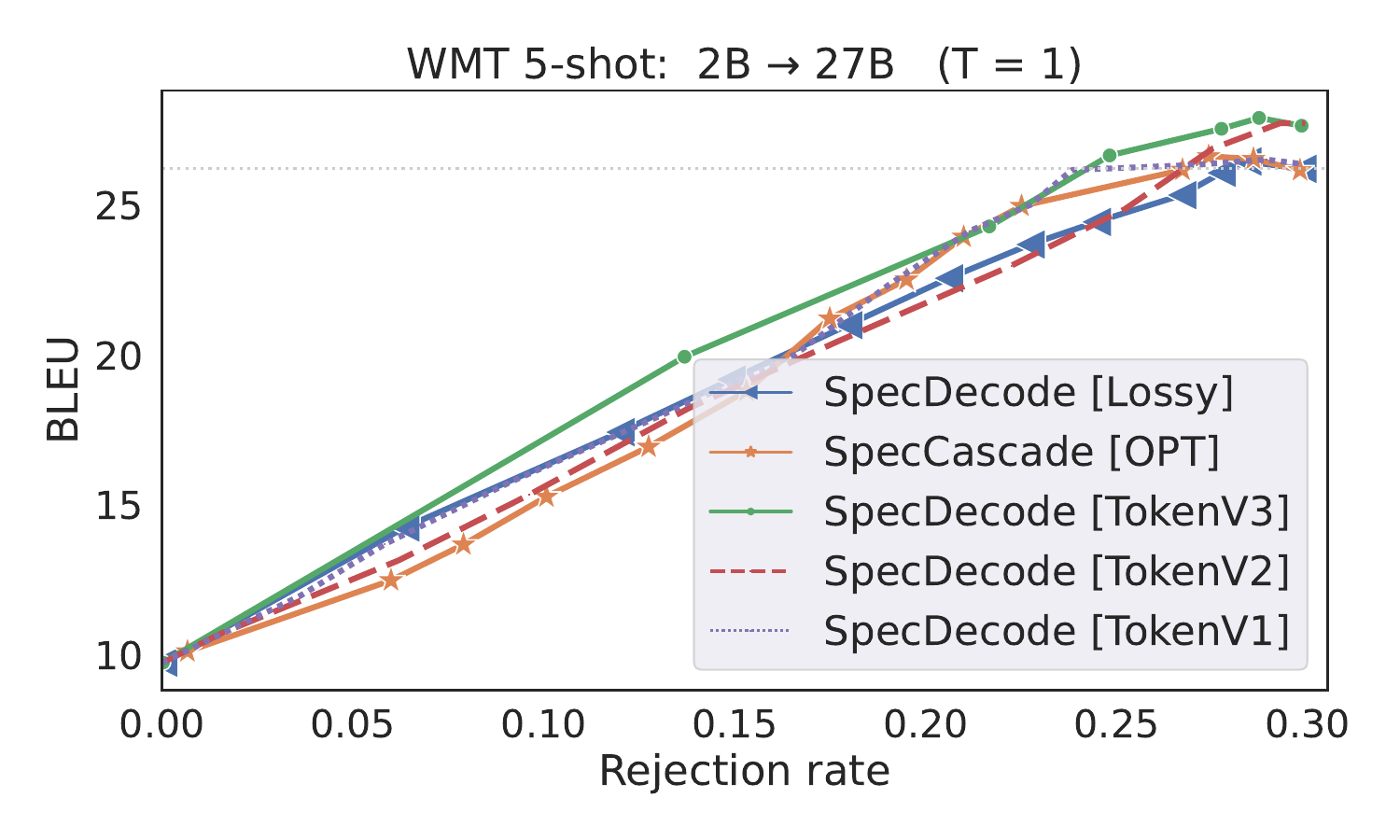}
    \includegraphics[scale=0.25]{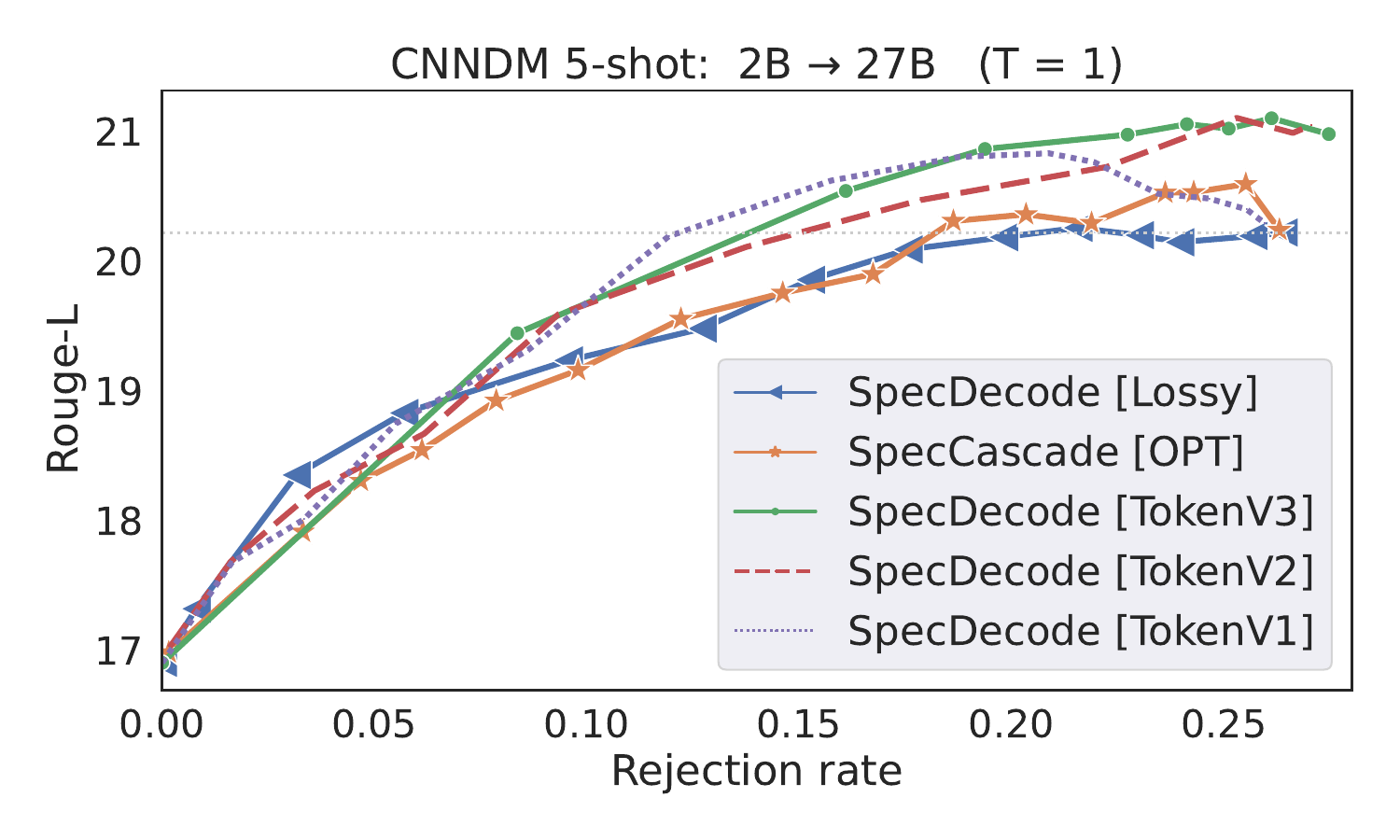}
    \includegraphics[scale=0.25]{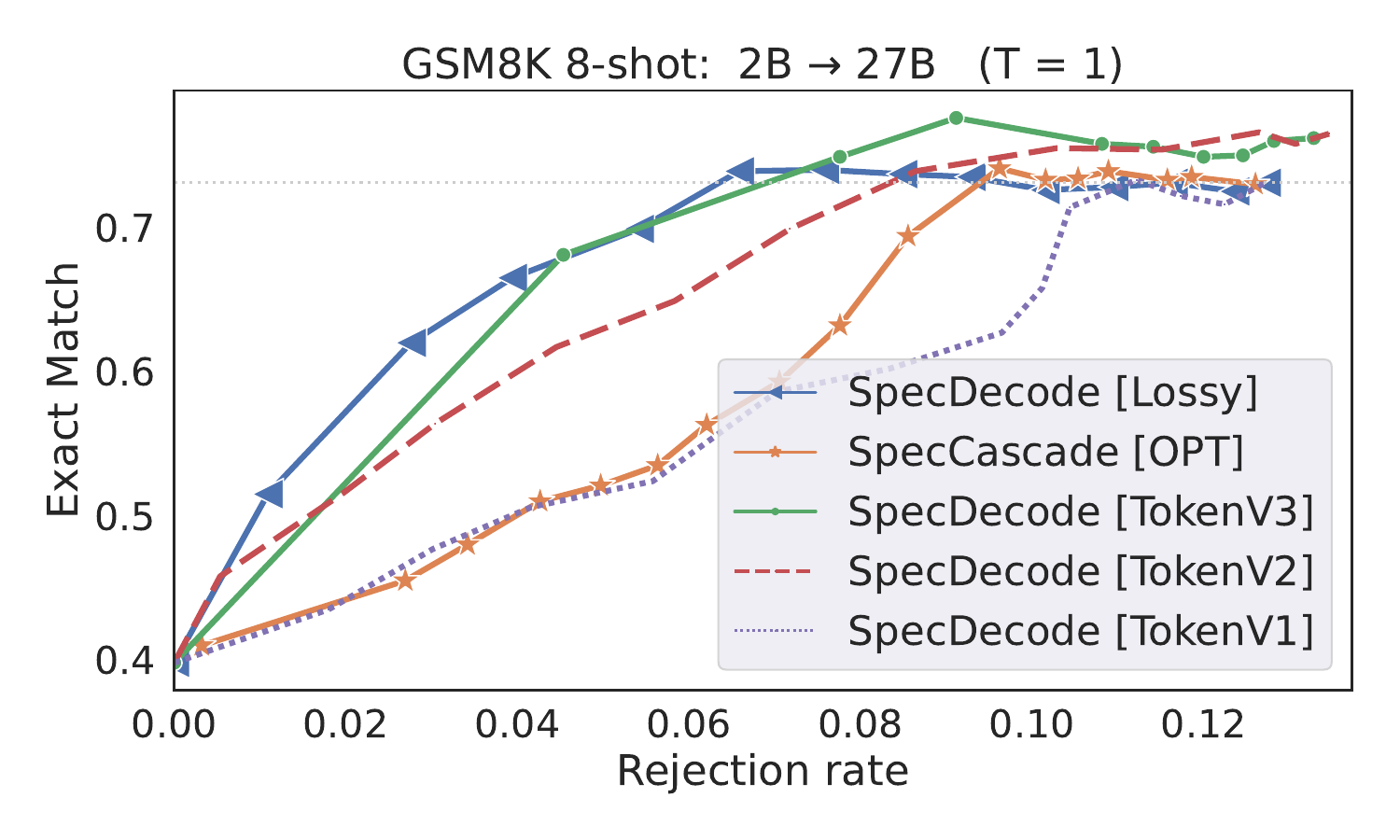}
    \includegraphics[scale=0.25]{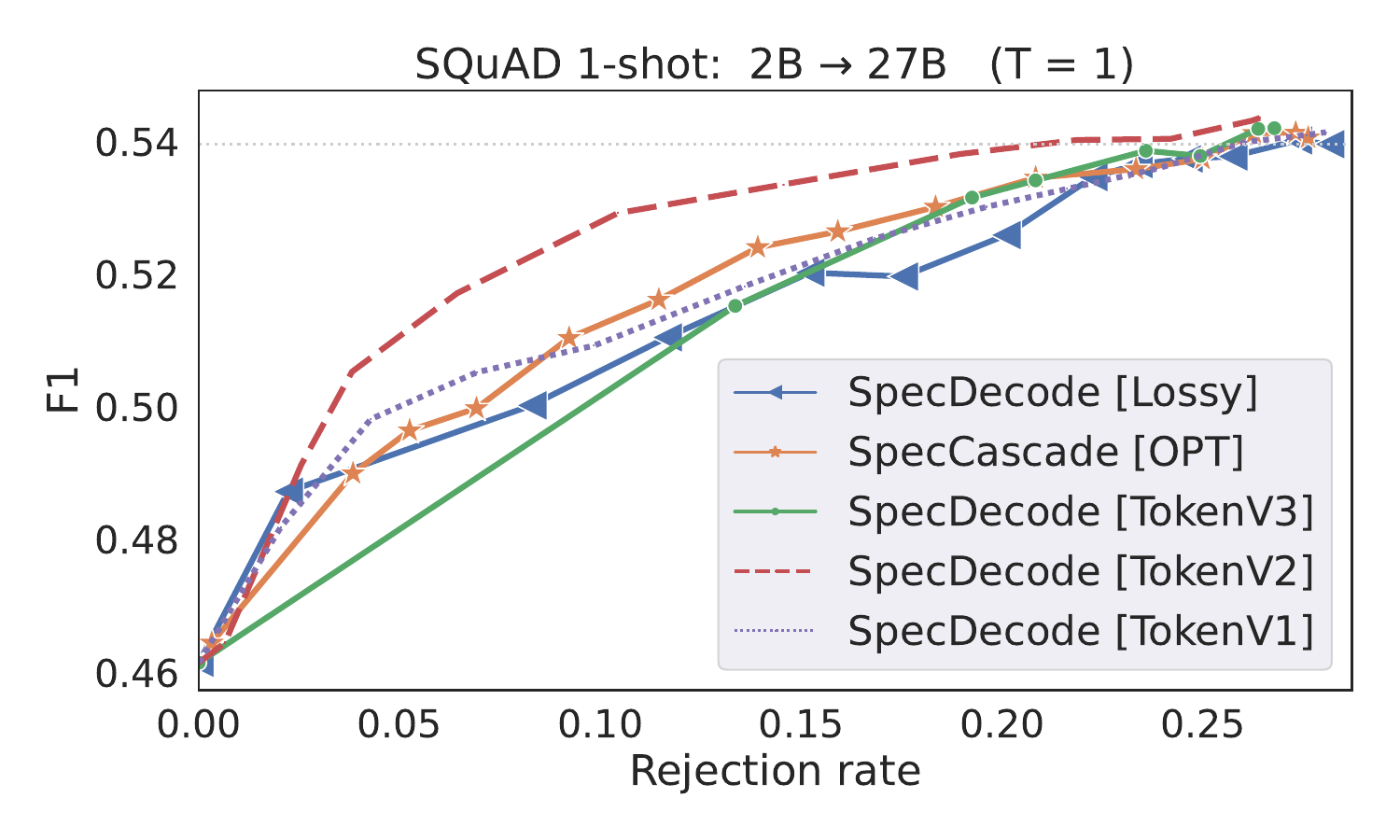}
    \includegraphics[scale=0.25]{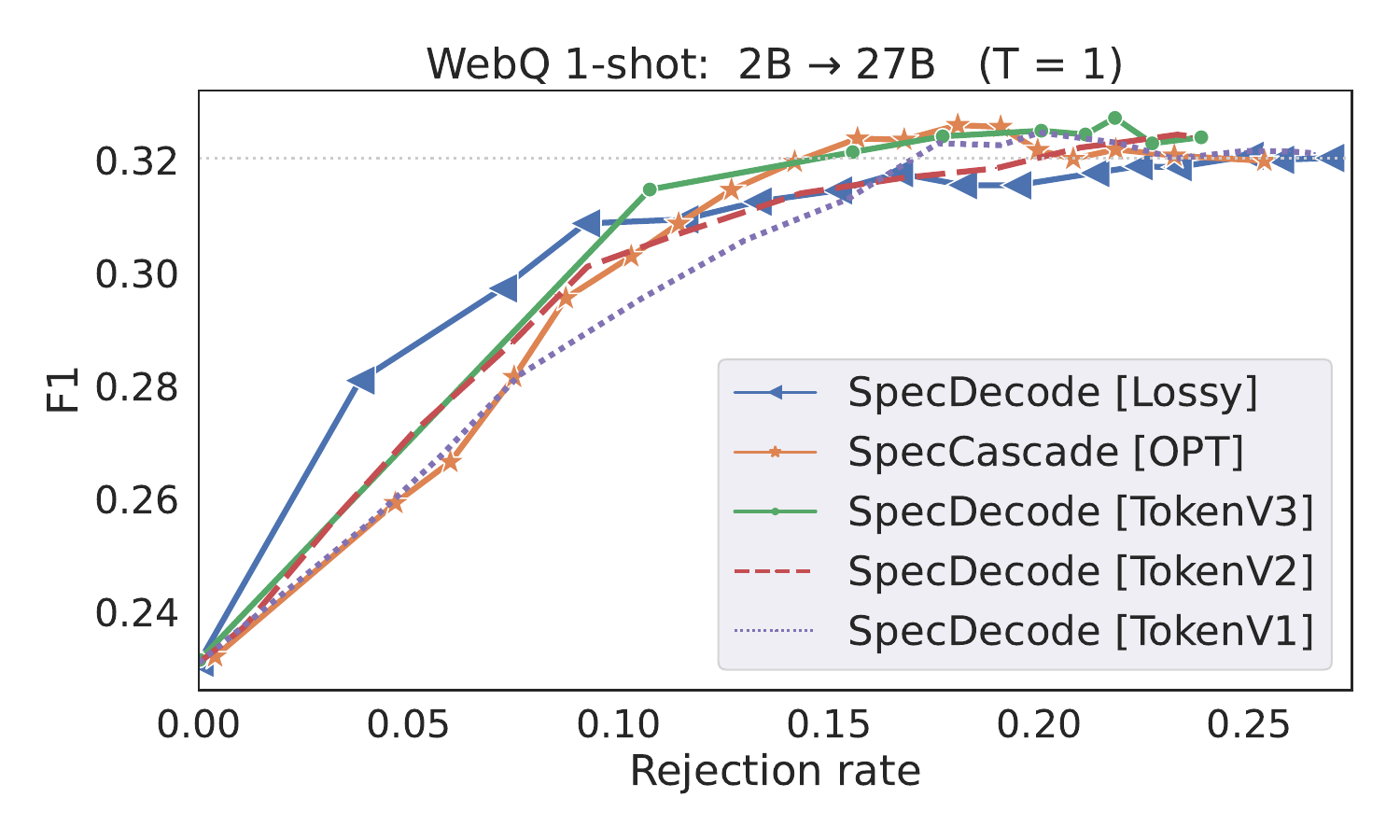}
    \includegraphics[scale=0.25]{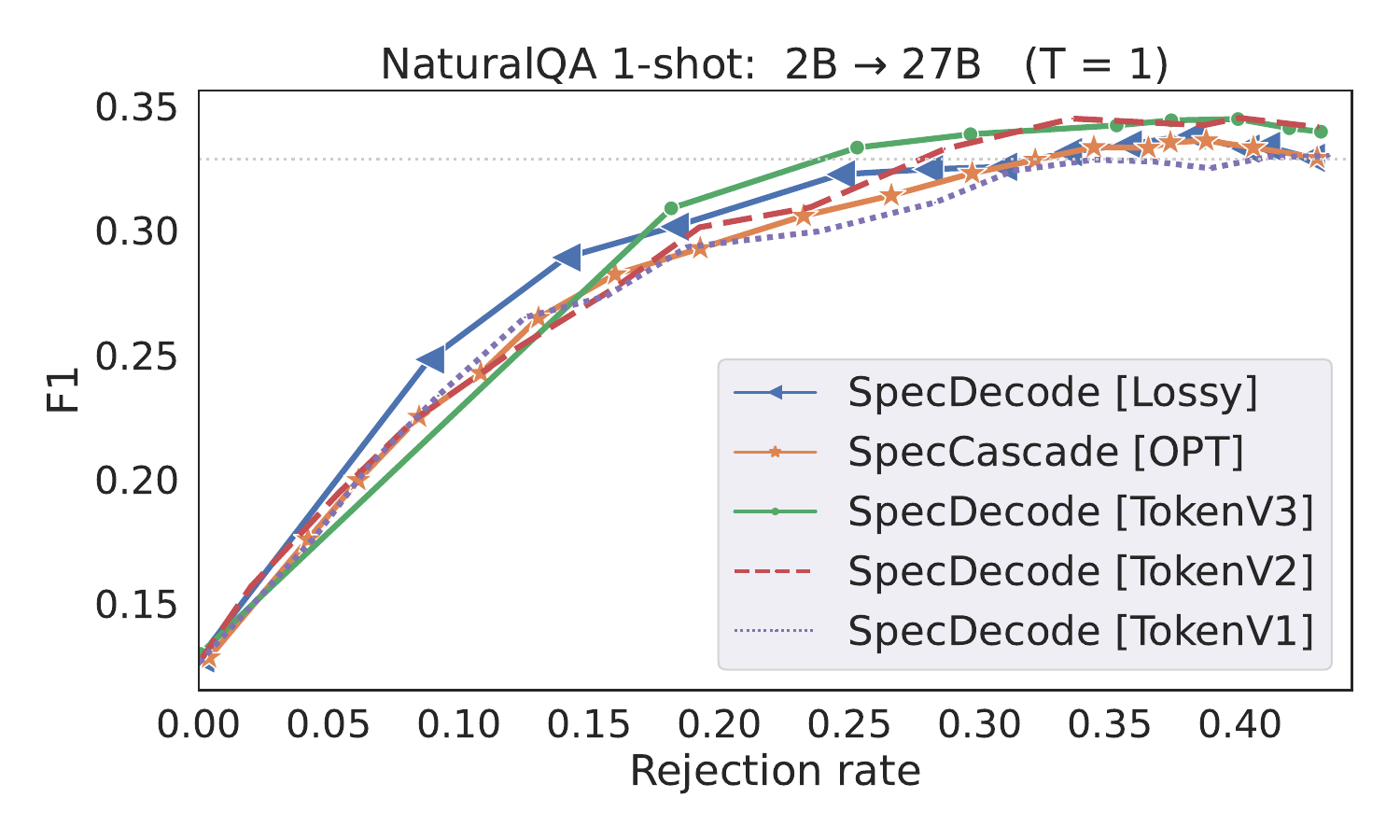}
    \includegraphics[scale=0.25]{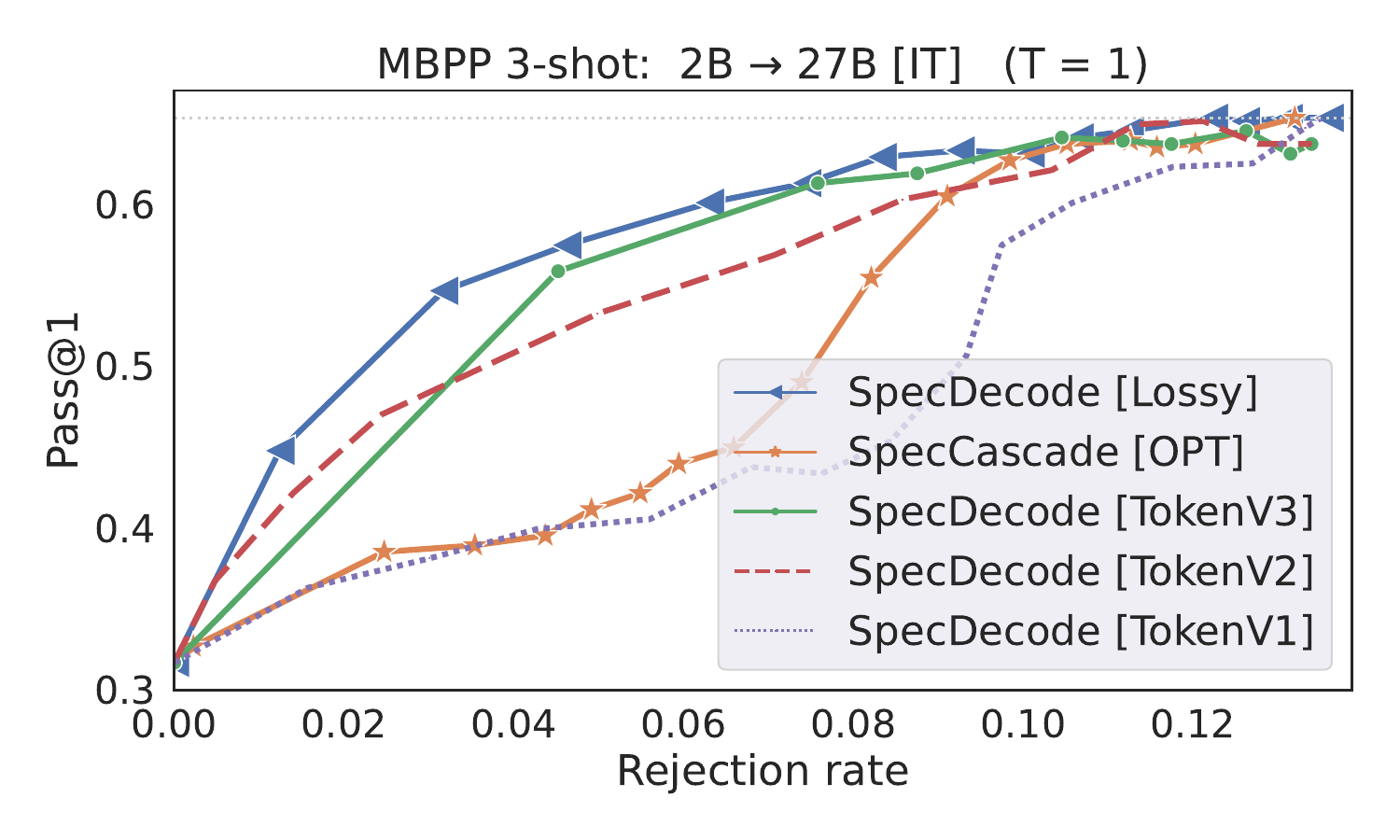}
    \includegraphics[scale=0.25]{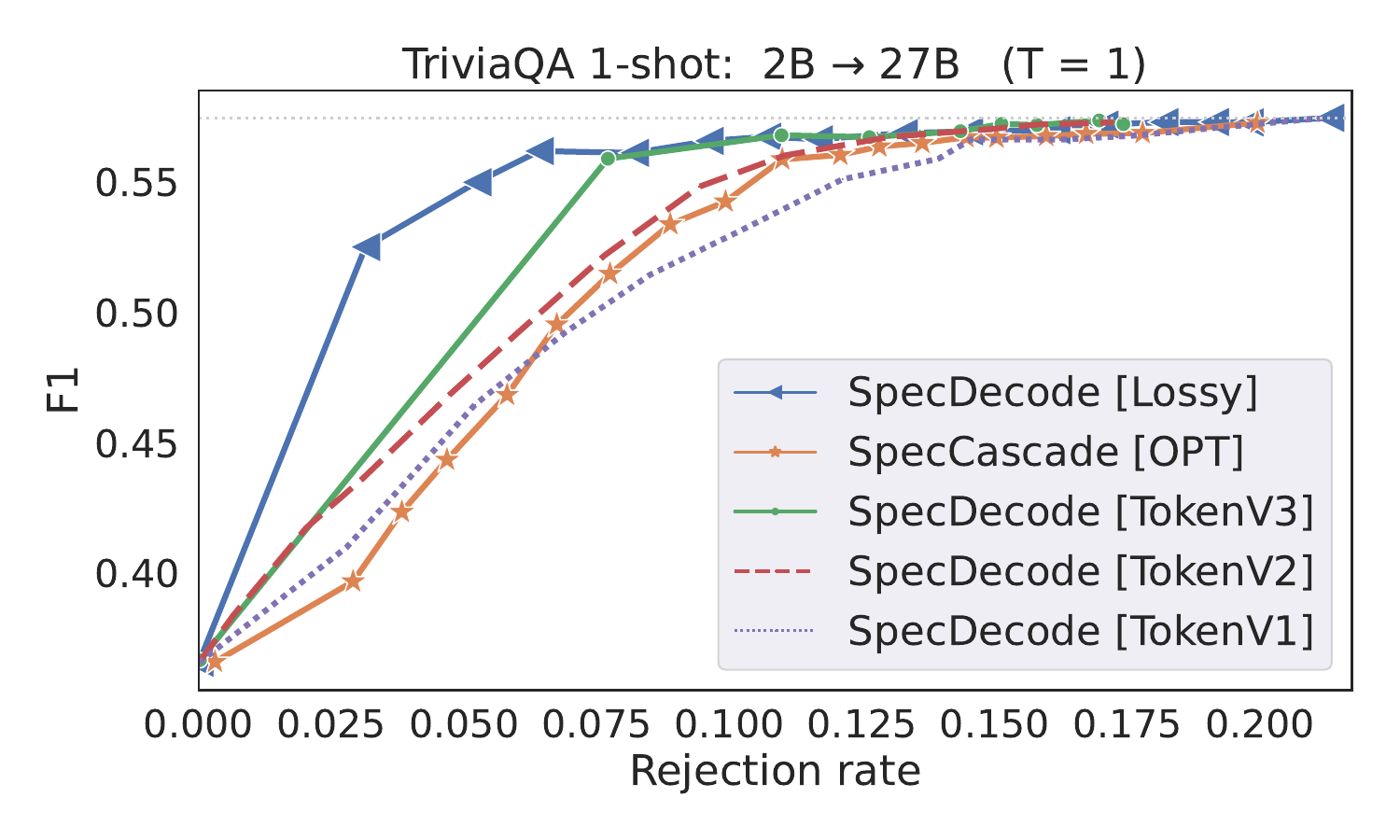}
\caption{Plots of quality vs.\ rejection rate for \textbf{Gemma models with all three token-specific speculative cascade deferral rules} in equations \ref{eq:sample-dep-01-plugin-v1}--\ref{eq:sample-dep-01-plugin-v3}. Each method interleaves a Gemma 2B drafter with a Gemma 27B verifier. 
    The horizontal dotted line denotes the quality of the large model. We include all three token-specific speculative cascade deferral rules in equations \ref{eq:sample-dep-01-plugin-v1}--\ref{eq:sample-dep-01-plugin-v3}.}
    \label{fig:gemma-2B-27B-token-specific}
\end{figure}

\todoakm{add plots. Hari:done!}

\begin{figure}[t]
    \centering
    \includegraphics[scale=0.25]{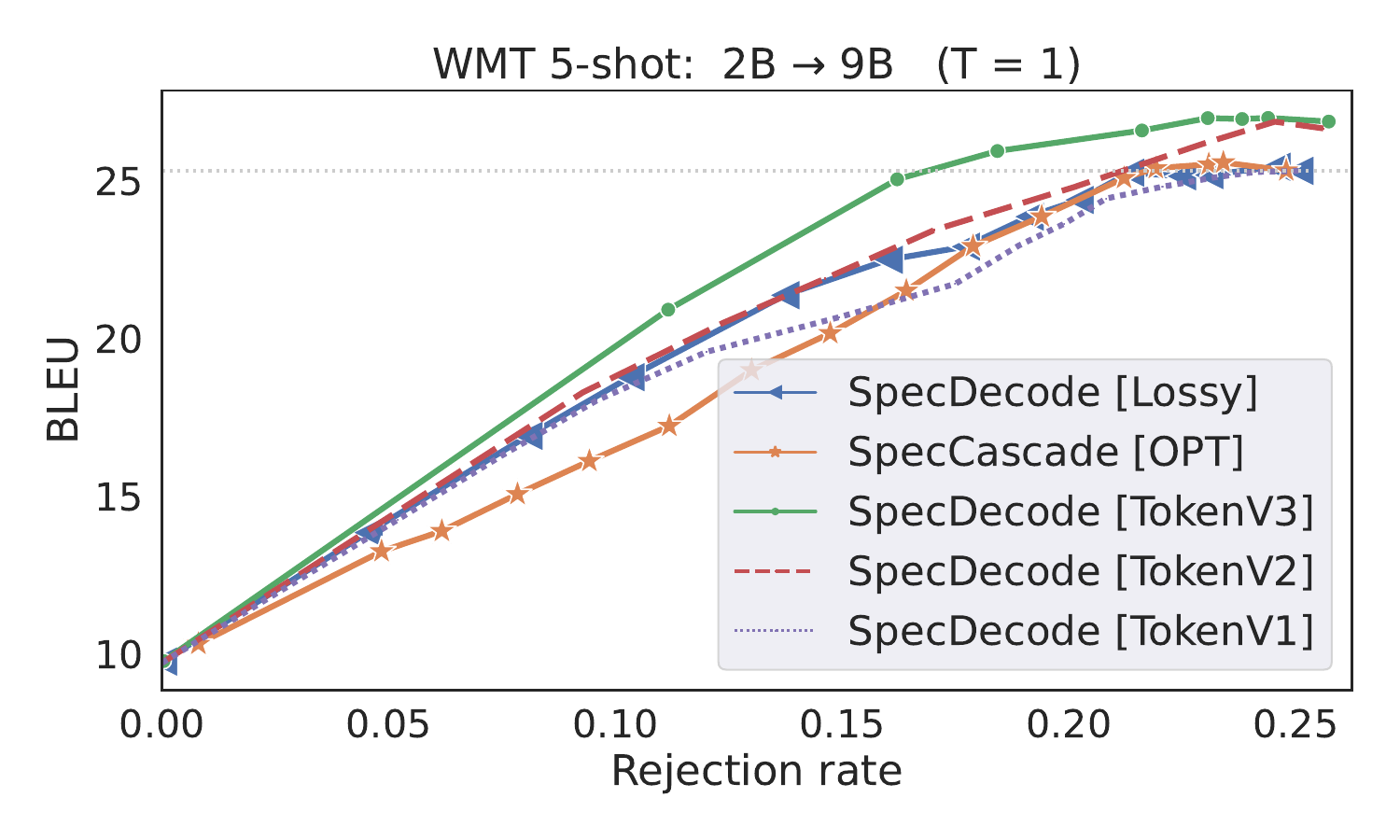}
    \includegraphics[scale=0.25]{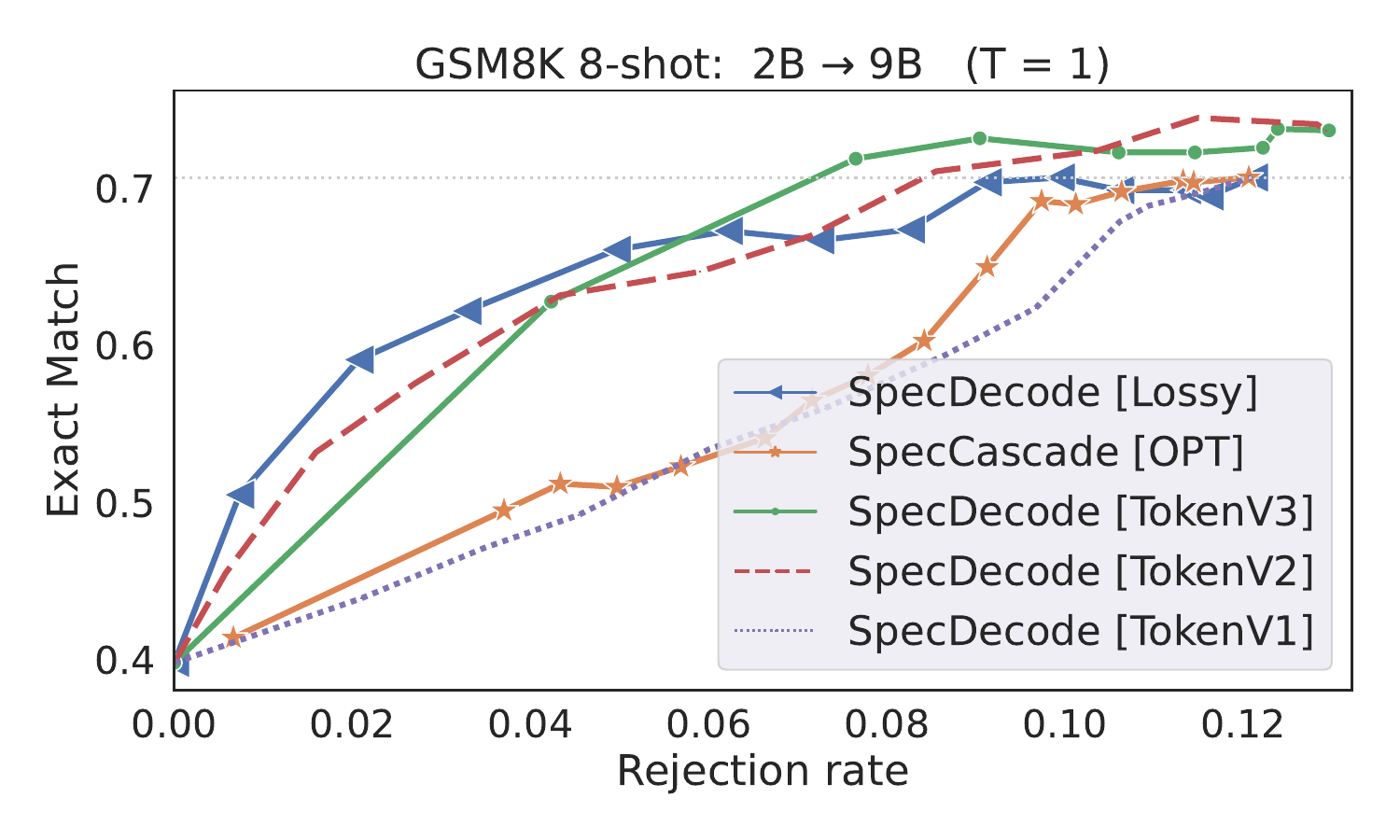}
    \includegraphics[scale=0.25]{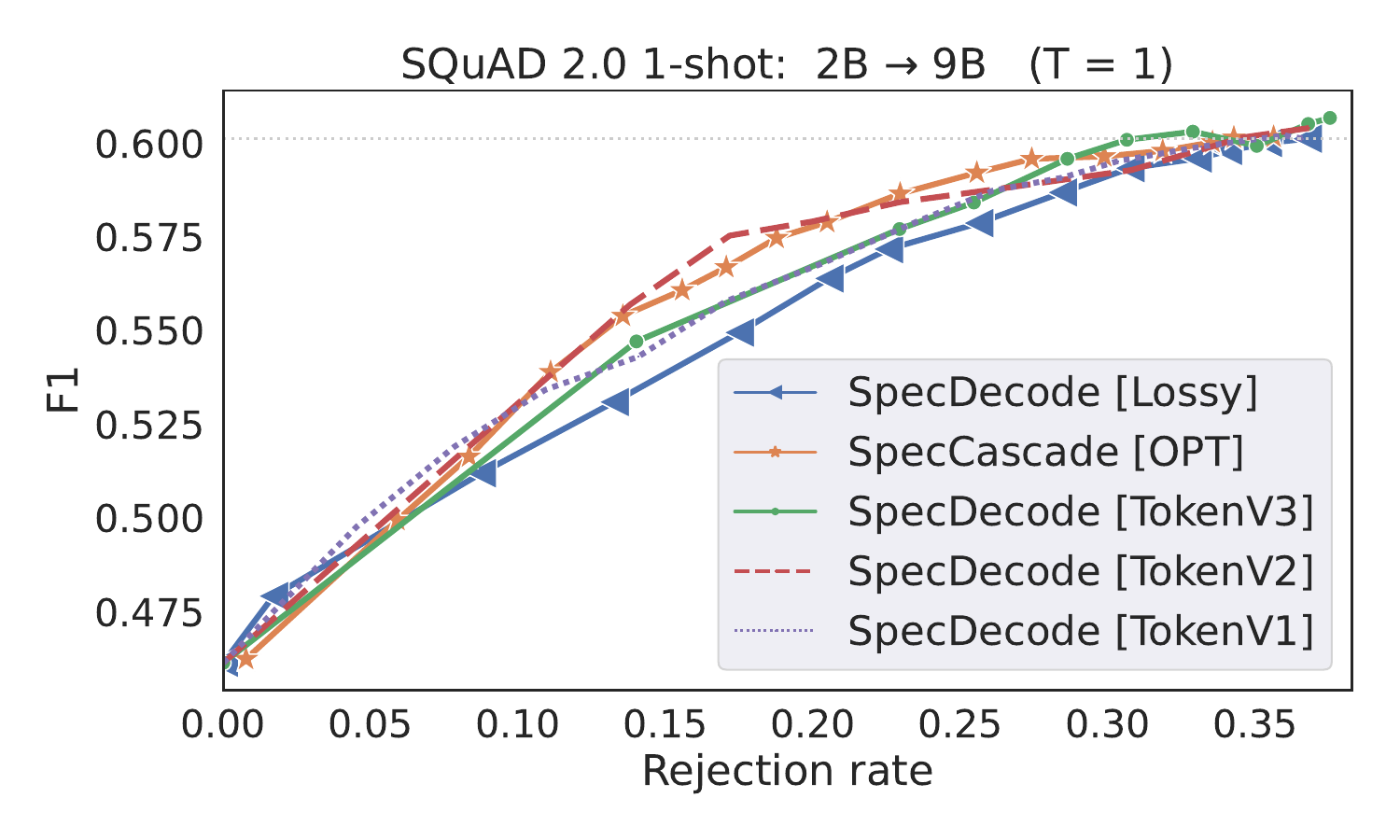}
    \includegraphics[scale=0.25]{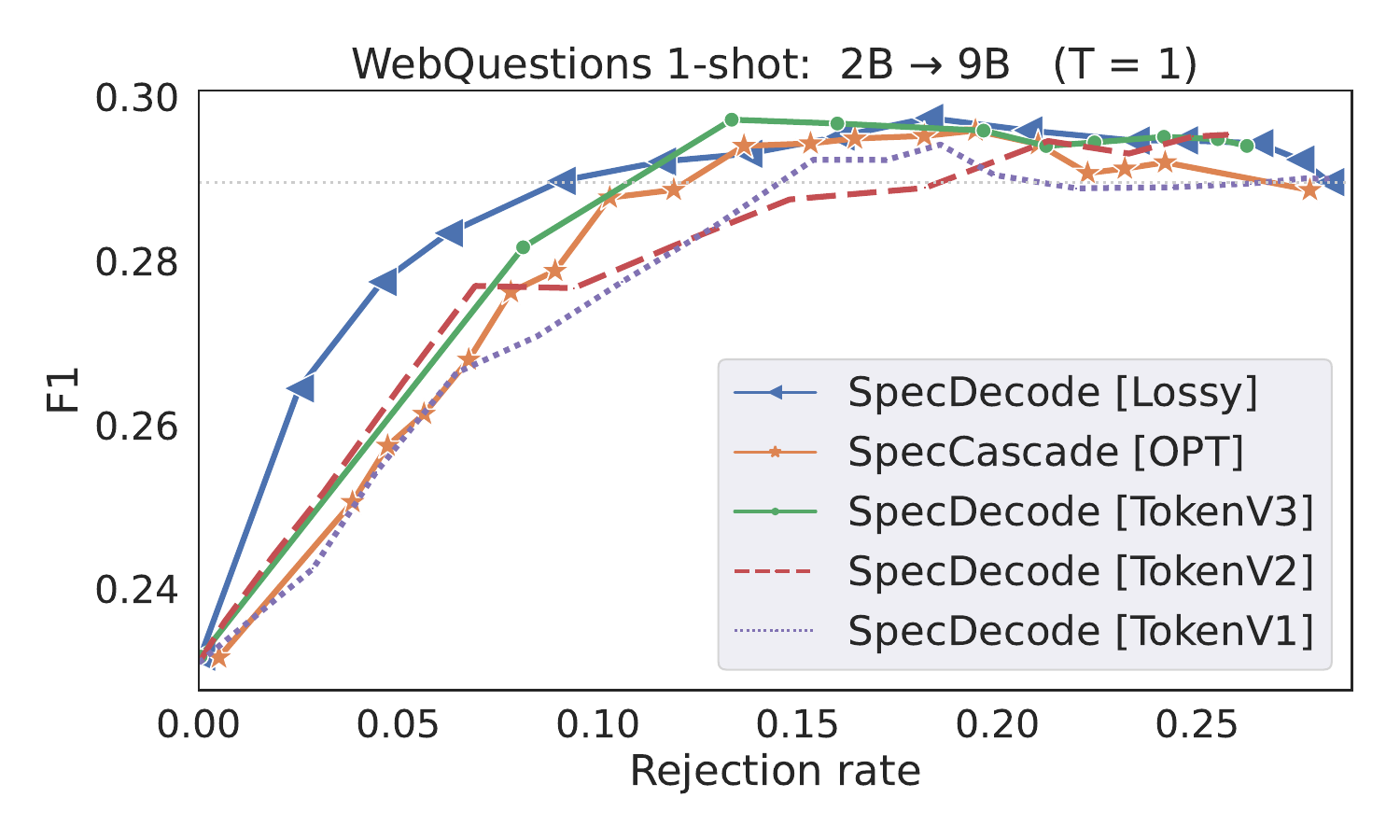}
    \includegraphics[scale=0.25]{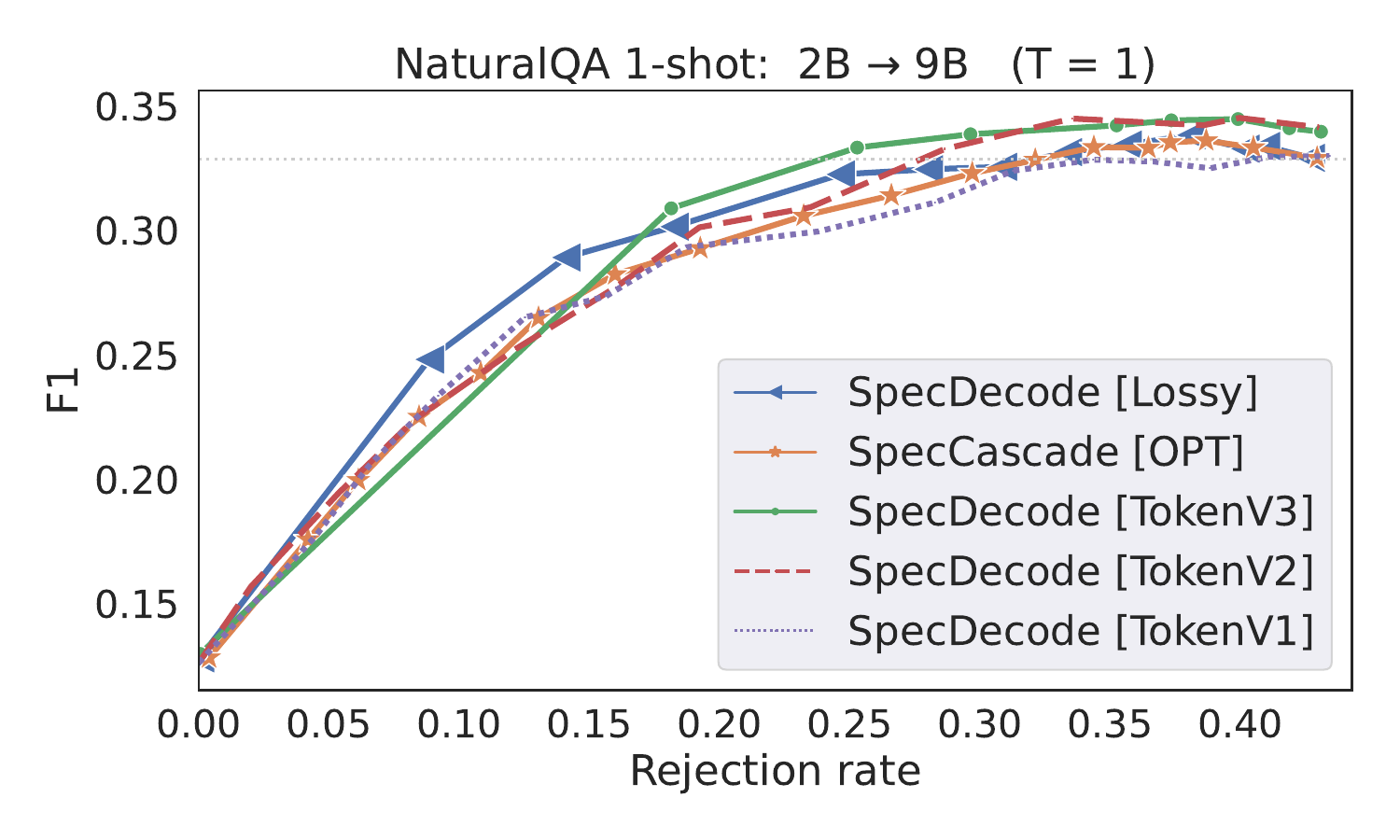}
    \includegraphics[scale=0.25]{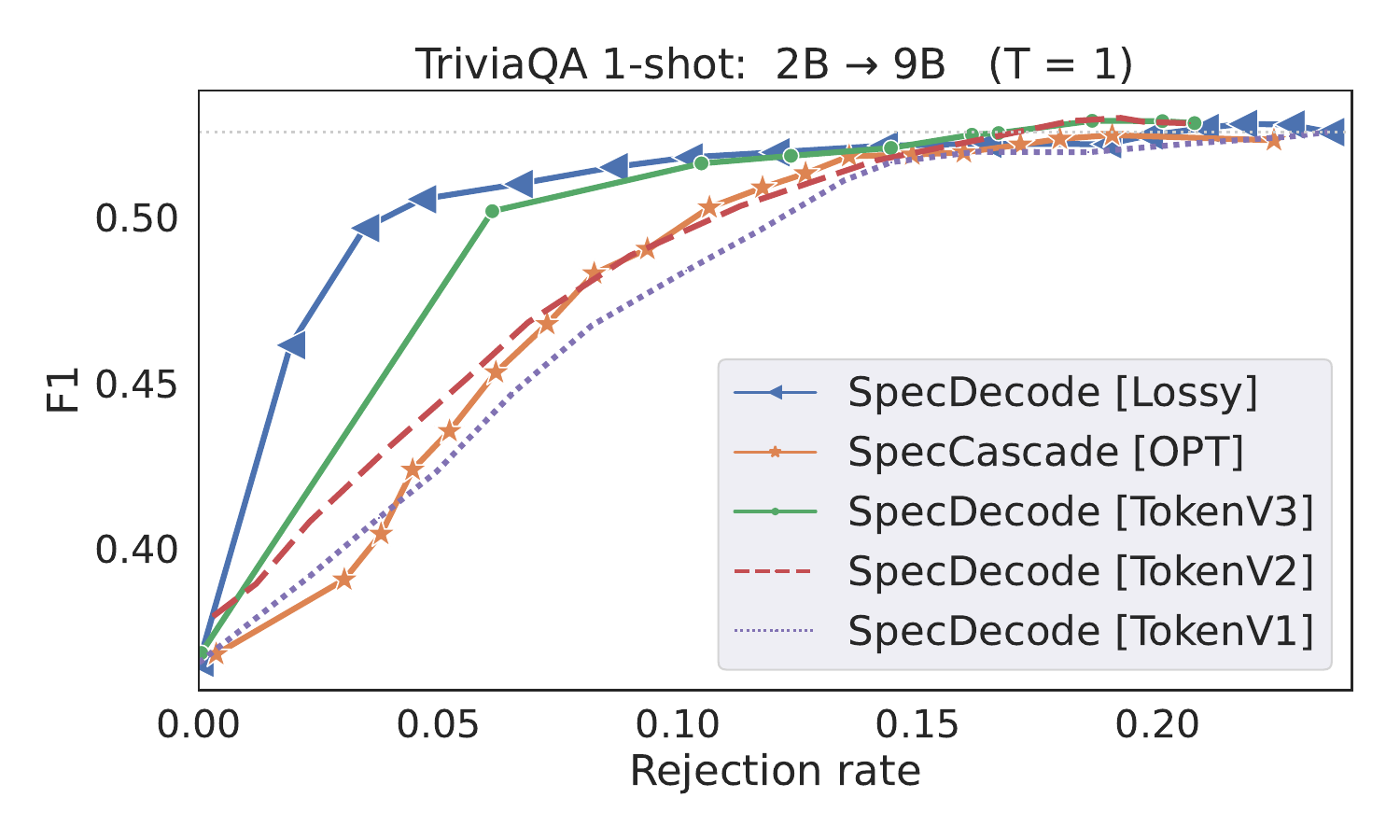}
\caption{Plots of quality vs.\ rejection rate with \textbf{Gemma 2B $\rightarrow$ 9B} speculative cascades. Each method interleaves a Gemma 2B drafter with a Gemma 9B verifier. 
    The horizontal dotted line denotes the quality of the large model. We include all three token-specific speculative cascade deferral rules in equations \ref{eq:sample-dep-01-plugin-v1}--\ref{eq:sample-dep-01-plugin-v3}.}
    \label{fig:gemma-2B-9B-token-specific}
\end{figure}

\todoakm{add plots. Hari:done!}

\section{Limitations}
\label{app:limitations}
One of the limitations of our proposal is the use of plug-in estimators to approximate the optimal rule \eqref{eq:spec-def-opt}. While these approximations are effective in practice, they rely on the individual models being calibrated. An alternative to the use of plug-in estimators is to use a router model explicitly trained to mimic the optimal rule using a validation sample drawn from $\Pr$ \citep{gupta2024language}.  Another limitation is that the optimization objectives we seek to minimize are local objectives that seek to make the best deferral decision at the current position $t$. In doing so, they ignore the downstream effects of choosing a particular model in the current step. Devising a global deferral objective that takes downstream errors into account would be an interesting direction for future work. 
 More broadly, our paper seeks to improve cost-quality trade-offs in LM inference. It is important that such improvements do not unfairly advantage one slice of the data or a subset of the population, at the cost of others. Ensuring that the trade-off gains that our approach offers is equitable across different slices of the data is another important direction for the future.

\end{document}